\documentclass[11pt,reqno]{amsart}

\usepackage[top=2.5cm,bottom=2.5cm,left=2.6cm,right=2.6cm]{geometry}
\usepackage{amsmath,amssymb,mathtools}
\usepackage{graphicx}
\graphicspath{{figs/}}
\usepackage{subcaption}
\usepackage{booktabs}
\usepackage{algorithm}
\usepackage{algpseudocode}
\usepackage{xcolor}
\usepackage{placeins} 
\usepackage{tikz}
\usetikzlibrary{arrows.meta,positioning,calc,shapes.geometric,fit,backgrounds}
\usepackage[colorlinks=true,linkcolor=blue!55!black,citecolor=blue!55!black,
            urlcolor=blue!55!black]{hyperref}

\definecolor{cblue}{HTML}{0072B2}   
\definecolor{cverm}{HTML}{D55E00}   
\definecolor{cgreen}{HTML}{009E73}  
\definecolor{camber}{HTML}{E69F00}  
\definecolor{cnavy}{HTML}{1A2B4C}   

\numberwithin{equation}{section}
\numberwithin{figure}{section}
\numberwithin{table}{section}
\numberwithin{algorithm}{section}

\theoremstyle{plain}
\newtheorem{theorem}{Theorem}[section]
\newtheorem{lemma}[theorem]{Lemma}
\newtheorem{proposition}[theorem]{Proposition}
\newtheorem{corollary}[theorem]{Corollary}
\theoremstyle{definition}
\newtheorem{definition}[theorem]{Definition}
\newtheorem{assumption}[theorem]{Assumption}
\newtheorem{problem}[theorem]{Problem}
\theoremstyle{remark}
\newtheorem{remark}[theorem]{Remark}

\newcommand{\R}{\mathbb{R}}
\newcommand{\norm}[1]{\left\|#1\right\|}
\newcommand{\abs}[1]{\left|#1\right|}
\newcommand{\Z}{\mathbb{Z}}
\newcommand{\dist}{\operatorname{dist}}
\newcommand{\Lip}{\operatorname{Lip}}
\newcommand{\bx}{\mathbf{x}}

\newcommand{\bsigma}{\boldsymbol{\sigma}}
\newcommand{\bTheta}{\boldsymbol{\Theta}}
\newcommand{\Fsc}{f}
\newcommand{\Ksc}{\mathcal{K}}
\newcommand{\Kinf}{\mathcal{K}_{\infty}}
\newcommand{\Mdata}{M_{H^s}}
\newcommand{\Tmax}{\tau_{\max}}
\newcommand{\Th}{\Theta}
\newcommand{\fTh}{f_{\Theta}}
\newcommand{\PhiTh}{\Phi_{\Theta}}
\newcommand{\PTh}{P_{\Theta}}
\newcommand{\Sth}{S_{\Theta}}
\newcommand{\cd}{c_\diamond}
\newcommand{\tauTh}{\tau_{\Theta}}
\newcommand{\Ltraj}{\mathcal{L}_{\mathrm{traj}}}
\newcommand{\LF}{\mathcal{L}_{F}}
\newcommand{\rhoT}{\rho_T}
\newcommand{\rhoTh}{\rho_{\Theta}}
\newcommand{\That}{\hat{T}}
\newcommand{\ghat}{\hat{\gamma}}

\begin{document}

\title[Long-time approximation via SA-NODEs]{Long-Time Trajectory Approximation via SA-NODEs:\\ Model Predictive and Floquet Strategies}

\author[Z. Li]{Ziqian Li$^{1}$}
\thanks{$^{1}$Chair for Dynamics, Control, Machine Learning and Numerics, Department of Mathematics, Friedrich-Alexander-Universit\"{a}t Erlangen-N\"{u}rnberg, 91058 Erlangen, Germany (\texttt{ziqian.li@fau.de}).}

\author[N. M. Matzakos]{Nikolaos M. Matzakos$^{2}$$^{1}$}
\thanks{$^{2}$School of Pedagogical \& Technological Education (ASPETE), Marousi Attikis, 151 22 Athens, Greece (\texttt{nikmatz@aspete.gr}).}

\makeatletter
\@ifundefined{subjclassname@2020}{%
  \expandafter\gdef\csname subjclassname@2020\endcsname{%
    \textup{2020} Mathematics Subject Classification}}{}
\makeatother

\begin{abstract}
We study the approximation of dynamical systems by semi-autonomous neural ordinary differential equations (SA-NODEs) over long time horizons.  For a single network trained on the whole horizon, the available error bound deteriorates double exponentially in the horizon length.  We develop two training strategies that avoid this barrier, each built on a reset of the state.  The model predictive strategy partitions the horizon adaptively and restarts every window from observed data: when training meets a prescribed tolerance on every window, the composite model meets it uniformly in time, with a parameter budget linear in the horizon for targets with a bounded, uniformly regular reachable tube.  The Floquet strategy addresses autonomous targets with a stable limit cycle and uses no data at deployment: a certified contraction of the learned return map confines the error to linear growth in the number of elapsed periods.  For the time-periodic architecture we deploy, the scalar certificate degenerates; we prove instead a uniform-in-time orbital guarantee whose hypotheses are measured on the trained model, and an obstruction showing that, for an exactly periodic learned field, small one-period error and a contracting stroboscopic map cannot hold at once.  Numerical experiments on four benchmarks confirm the predicted error laws and measure the hypotheses of every guarantee.
\end{abstract}

\maketitle

\section{Introduction}\label{sec:intro}

\subsection{Problem statement and motivation}
A data-driven model of a dynamical system is useful only over the time horizon on which its predictions stay accurate.  Neural ordinary differential equations (neural ODEs) parameterize the right-hand side of a differential equation by a neural network, viewing a residual architecture as the discretization of a continuous-time flow \cite{Chen18}.  When the training data are sampled from trajectories of an underlying dynamical system $\dot\bx = \Fsc(t,\bx)$, the trained network approximates the flow itself.  This paper quantifies how the error of such models grows with the horizon and develops two training strategies that control the growth.  The architecture we work with is the \emph{semi-autonomous} neural ODE (SA-NODE) \cite[eq.~(2.1)]{LLLZ24}.  For a horizon $T>0$ and an initial condition $\bx_0\in\R^d$, with $d$ the dimension of the state space, it evolves a state $\bx(t)\in\R^d$ by
\begin{equation}\label{eq:sanode}
  \left\{
  \begin{aligned}
    \dot\bx(t) &= \sum_{i=1}^{P} W_i \circ \bsigma\!\bigl(A_i^1\bx(t) + A_i^2 t + B_i\bigr),
      && t\in(0,T],\\[2pt]
    \bx(0) &= \bx_0 .
  \end{aligned}
  \right.
\end{equation}
We write $f_{\bTheta}(\bx,t)$ for the right-hand side of \eqref{eq:sanode}.  Its defining feature, and the reason for the name, is that the time $t$ enters only through the linear bias $A_i^2 t + B_i$.  The trainable parameters $\bTheta=(W_i, A_i^1, A_i^2, B_i)_{i=1}^{P}$ are time-independent, with $W_i, A_i^2, B_i\in\R^d$ and $A_i^1\in\R^{d\times d}$; $\circ$ is the componentwise product, $\bsigma$ the componentwise activation, and $P\in\mathbb{N}$ the hidden width.  On a compact set $\Ksc \subset \R^d$ of initial conditions, the universal approximation theorem \cite[Thm.~2.3]{LLLZ24} provides parameters $\bTheta_P$ with
\begin{equation}\label{eq:uap}
  \sup_{\bx_0\in\Ksc,\; t\in[0,T]}
  \norm{\Phi(t;\bx_0) - \widehat\Phi_{\bTheta_P}(t;\bx_0)}
  \;\leq\; C_{T,\Ksc,\Fsc}\, P^{-1/2},
\end{equation}
where $\Phi$ is the target flow, $\widehat\Phi_{\bTheta_P}$ is the SA-NODE flow, and the subscripts of the constant record its dependence on the horizon, the set of initial conditions, and the field.  The setting is approximation, not identification.  Estimate \eqref{eq:uap} provides short-horizon closeness of the learned flow.  We take that closeness as given and study how the error propagates over long horizons.  The experiments train the models from sampled trajectories, but we do not address system identification from finite or noisy data.  The question this paper addresses is how the constant in \eqref{eq:uap} grows with $T$, and how to avoid that growth.

The available upper bound on that constant carries the Gr\"onwall factor $e^{LT}$, with $L$ the Lipschitz constant of $\Fsc$ in $\bx$, and the explicit constant of \cite[Remark~2.6]{LLLZ24} deteriorates double exponentially in $T$; Section~\ref{sec:prelim} states it and shows that both mechanisms behind it are individually sharp.  We call this the \emph{time horizon barrier}.  {It is a property of the available certificate: no matching lower bound over Lipschitz targets is known.  According to this guarantee, a single network trained once on the whole horizon, which we call the \emph{monolithic} model, is useful only for horizons of a few multiples of $1/L$.  Whether the barrier is intrinsic is open; Lemma~\ref{lem:flow-to-field} reduces the question to a width lower bound for the target field, stated as Problem~\ref{prob:lower}.}  The two strategies below avoid the barrier by changing the learning problem rather than enlarging the network.  The barrier was identified in \cite[Remarks~2.6 and~2.9]{LLLZ24}, together with a proposal, not developed there, to control it by model predictive ideas.

The factor $e^{LT}$ results from integrating error over the whole of $[0,T]$ without any correction.  Structured targets can escape it, but through learning strategies that exploit that structure.  In computational terms, we replace one network fitted to the whole horizon by several short-horizon networks restarted from data, or by a single periodic network with a certified stable orbit.

If the true state is available at intermediate times, the horizon can be partitioned into windows of length at most $\Tmax \ll T$, one network per window, with the initial condition reset to the true state at each \emph{switch time}; the Gr\"onwall factor then only reaches $e^{L\Tmax}$.  This is model predictive control (MPC) applied to approximation rather than control, and we call the injection of the true state a \emph{data reset}.  The name records the mechanics: adaptive multiple shooting with the node values pinned to data, not an online feedback law.

If instead the target is autonomous and its trajectories are attracted to a stable limit cycle, the orbit itself corrects errors: every perturbation transverse to the cycle is contracted at each return to a Poincar\'e section, at the rate given by the Floquet multiplier.  A learned model that inherits this contraction needs no data at deployment.  We call this a \emph{dynamical reset}.  The second half of this paper asks whether a trained SA-NODE can be made to inherit it, with a certificate.

The two resets are complementary.  The data reset restarts the error at every switch time; it applies to any Lipschitz target with a bounded, uniformly regular reachable tube, at the cost of the true state at each switch.  The dynamical reset contracts transverse errors at every return and needs no data at deployment, but it requires an autonomous target settling onto a limit cycle, and the contraction of the learned return map must be trained in and then certified.

Targets of the latter kind arise in chemical oscillators, vortex shedding, structural vibrations, and circadian or cardiac rhythms, where the horizon of interest spans many periods.  The main purpose of this paper is to turn each reset into a quantitative long-horizon approximation guarantee for trained SA-NODEs, and to measure the hypotheses of these guarantees on the trained models themselves.

\subsection{Main results}\label{ssec:main}
The main results are as follows.  In the Floquet statements below, and throughout Section~\ref{sec:floquet}, states are written unbold, $x_0$, as points of $\R^d$, and the parameter vector $\bTheta$ is abbreviated $\Th$.

\begin{itemize}
\item \emph{Composite data-assisted approximation.}
We design an adaptive receding-horizon algorithm (Algorithm~\ref{alg:mpc}) that grows a composite SA-NODE one window at a time, opening a new window when the supervised error first exceeds a tolerance $\varepsilon$.  {The guarantee is conditional on training: the tolerance must be realized on every window and the run must cover the horizon (Assumption~\ref{asm:partition}).}  In data-IC mode, in which the true state is supplied as the initial condition at each of the $N$ switch times, the composite flow $\widehat\Phi$ then satisfies
  \begin{equation}\label{eq:main-mpc}
    \sup_{t\in[0,T],\,\bx_0\in\Ksc}
    \norm{\Phi(t;\bx_0)-\widehat\Phi(t;\bx_0)} \;\leq\; \varepsilon,
    \qquad
    P_{\mathrm{total}} \;\lesssim\; N\, C_{\Tmax,\Kinf,\Fsc}^{2}\,
    \varepsilon^{-2},
  \end{equation}
where the window constant is controlled by the data of $\Fsc$ on the closure $\Kinf$ of the forward reachable tube; it is independent of $T$ whenever that tube is bounded and its time-shifted Sobolev data are uniformly bounded (Assumption~\ref{asm:reach}, Theorem~\ref{thm:linearT}).  Every admissible partition has $N \geq \lceil T/\Tmax\rceil$; the uniform partition attains it, so the budget is then linear in $T$.  For the adaptive partition, the minimum window length of Algorithm~\ref{alg:mpc} already forces $N=O(T)$, with a constant fixed by the resolution of the training data rather than by the target.  Whether linear growth survives the removal of that safeguard is open (Problem~\ref{prob:NT}), and the measurements of Section~\ref{sec:numerics} are consistent with it.  The reset confines the Gr\"onwall factor of \eqref{eq:uap-const} to a single window; the growth of the reachable tube persists through the domain of the window constant.

\item \emph{Deployment on novel initial conditions.}
Without state resets the per-window errors chain multiplicatively.  We call \emph{predicted-IC} the deployment in which each window starts from the previous window's prediction instead of the true state.  For the predicted-IC composite flow and every $\bx_0\in\Ksc$ we prove the partition-adapted bound
  \begin{equation}\label{eq:main-novel}
    \sup_{t\in[0,T]}
    \norm{\Phi(t;\bx_0)-\widehat\Phi^{\mathrm{pred}}(t;\bx_0)}
    \;\leq\; \varepsilon \sum_{k=1}^{N} e^{\bar L (T-\tau_k)}
    \;\leq\; \varepsilon\,
    \frac{e^{N\bar L\Tmax}-1}{e^{\bar L\Tmax}-1},
  \end{equation}
where $\tau_1<\dots<\tau_N$ are the switch times and $\bar L$ bounds the Lipschitz constants of the learned fields (Proposition~\ref{prop:novel-ic}).  The bound remains exponential; the gain over the monolithic model lies in the prefactor.

\item \emph{Certified Floquet approximation.}
For an autonomous target with a hyperbolic stable limit cycle $\Gamma$ of period $\That$, we train a single SA-NODE with a differentiable \emph{Floquet loss} that penalizes a surrogate $\tilde\rho_T(\Th)$ of the learned Poincar\'e multiplier; the subscript of the surrogate marks the transverse direction, not the horizon.  Let $P_\Th$ denote the first-return map of the learned flow to a section through a base point $p$ of the cycle, and let $\rho_*<1$ be a prescribed contraction threshold.  If the certificate $\rho(DP_{\Th}(p)) \leq \rho_*$ holds for the autonomous return map, then for every initial condition near $\Gamma$ and every $t \geq 0$,
  \begin{equation}\label{eq:main-floquet}
    \norm{\PhiTh(t;x_0) - \Phi_f(t;x_0)}
    \;\leq\; C\,\varepsilon\Bigl(1 + \frac{t}{\That}\Bigr),
  \end{equation}
where $\varepsilon$ now denotes the one-period accuracy of the learned flow, the analogue for this strategy of the tolerance above (Theorem~\ref{thm:floquet}).  The bound is linear in the number of elapsed periods, with $C$ independent of $t$.  Transverse errors are contracted at every return, and only the phase drift between the two clocks persists.  The periodically encoded architecture reduces to the autonomous case by Lemma~\ref{lem:nearaut}, at the cost of adding to $\varepsilon$ the measurable oscillation $\eta$ of the learned field about its time average.  That reduction is effective only while $\eta$ is small, and the trained models of Section~\ref{ssec:num-cert} are not in that regime; this caveat governs how the experiments are read.

\item \emph{Certification routes.}
The contraction hypothesis is verifiable.  In dimension two the identity of Liouville--Abel ties the surrogate to the autonomous return-map multiplier: exactly when the trajectory loss vanishes, and to $O(\varepsilon)$ in general (Corollary~\ref{cor:lf-cert}).  Writing $\LF$ for the Floquet loss built on the surrogate and $\rhoTh$ for the spectral radius of the learned return map at its fixed point, for $\varepsilon$ small enough and a relaxation returning a closed orbit, as quantified in Corollary~\ref{cor:lf-cert}, the chain
  \begin{equation}\label{eq:main-chain}
    \LF(\Th) = 0
    \;\Longrightarrow\;
    \rhoTh < 1
    \;\Longrightarrow\;
    \norm{e(k\That)} \leq C\varepsilon(1+k),
  \end{equation}
where $e(t):=\PhiTh(t;x_0)-\Phi_f(t;x_0)$ is the trajectory error and $k$ counts the elapsed periods, links the training objective directly to the long-horizon guarantee (Corollary~\ref{cor:lf-cert}; see also Proposition~\ref{prop:c1-cert}).  {This chain is proved for an autonomous planar field.  For the deployed time-periodic architecture we prove a separate orbital guarantee (Theorem~\ref{thm:orbital}) and show that the trajectory-wise one is quantitatively incompatible with the certified contraction.}  One-period closeness and a contracting stroboscopic map are quantitatively incompatible, under a contraction hypothesis on a neighborhood of the cycle (Proposition~\ref{prop:obstruction}) or under the local spectral hypothesis the experiments measure (Proposition~\ref{prop:obstruction-local}).  For the encoded architecture the divergence integral controls only the determinant of the monodromy, so we compute the full spectrum instead (Remark~\ref{rem:det}).  In higher dimension certification proceeds through $C^1$-closeness and a Bauer--Fike margin.

\item \emph{Numerical experiments.}
Section~\ref{sec:numerics} illustrates each result with a dedicated experiment.  On a forced Duffing benchmark the composite meets its tolerance up to a quantified safeguard excess, while a monolithic baseline under matched training cost exceeds it about $4.0$-fold.  On the pendulum the monolithic error grows exponentially along the horizon, at a measured rate $e^{0.39\,t}$, and the window count grows with $T$ at fixed $\varepsilon$, consistent with a linear law.  A two-by-two ablation on two limit-cycle benchmarks measures the spectral radius of the learned one-period monodromy, together with the two hypotheses of Theorem~\ref{thm:orbital} and the lower bound that Proposition~\ref{prop:obstruction} predicts for the one-period error.  An autonomous arm shows the same protocol failing without the encoding: no cold-started run produces a certifiable closed orbit.  A warm-started autonomous arm repairs the failure: fine-tuning from the converged encoded models yields closed, transversally stable learned orbits whose surrogate--multiplier gap is $65$ to $320$ times below the measured $\varepsilon$, consistent with Corollary~\ref{cor:lf-cert}(ii).  A direct comparison on the van der Pol oscillator places the two reset mechanisms side by side under one protocol.
\end{itemize}

\subsection{Novelties}\label{ssec:novelty}
The novelties of this paper concern two audiences: numerical analysis, and the practice of learning dynamical models from data.

For numerical analysis, the paper turns two classical mechanisms into approximation guarantees for learned flows.  To our knowledge, the MPC strategy gives the first uniform-in-time, data-restarted approximation bound for SA-NODEs with an explicit summed-width budget.  In multiple-shooting practice the partition serves identification, control, or discretization-error control, not an approximation guarantee for a learned flow.  The Floquet strategy establishes a propagation law for learned periodic dynamics: a certified contraction of the trained return map converts one-period closeness into the global-in-time linear bound of Theorem~\ref{thm:floquet}.  The certificate is exact in the plane for an autonomous learned field up to an $O(\varepsilon)$ period correction (Corollary~\ref{cor:lf-cert}) and perturbative in higher dimension (Proposition~\ref{prop:c1-cert}).  For the time-periodic architecture we deploy, neither form is available; there the full monodromy spectrum supplies the measured hypothesis of the orbital bound of Theorem~\ref{thm:orbital}.

For machine learning and data science, the results change what a long-horizon model of dynamics should be asked to demonstrate.  A small loss on short trajectories does not by itself guarantee long-horizon accuracy: long-time accuracy is governed by the reset mechanism, not by the fit.  Data play two roles: they train the model, and they reset it at deployment.  The uniform guarantee of the MPC strategy rests on the second.  Without deployment data, stability must be trained in and then verified.  The Floquet loss makes a differentiable surrogate of the certificate trainable.  For the time-periodic architecture that surrogate is the determinant of the monodromy, which does not control its spectral radius, so what the spectrum then supports is not \eqref{eq:cert-cond} but the orbital guarantee of Theorem~\ref{thm:orbital} (Section~\ref{ssec:orbital}).

\subsection{Related works}\label{ssec:related}
Our work draws on four lines of research: neural ODEs, multiple shooting and receding-horizon methods, stability-regularized learning of dynamics, and the error-growth theory of numerical integration near periodic orbits.

Neural ODEs were introduced in \cite{Chen18}, who interpret a residual network as the discretization of a continuous-time flow and train the vector field by differentiating through the solver.  Beyond supervised learning, they have become a model class for dynamics, and extended to operator learning \cite{li2026deep}.  The architecture of the present paper is the SA-NODE of \cite{LLLZ24}. Its universal approximation theorem, recalled as Theorem~\ref{thm:uap}, supplies the certificate whose horizon dependence is the object of this paper.  Section~\ref{sec:mpc} turns the model predictive remark of \cite[Rem.~2.9]{LLLZ24} into an algorithm and a guarantee (Algorithm~\ref{alg:mpc}, Theorem~\ref{thm:linearT}).

Partitioning a horizon and training one model per window, with the true state injected at the switch times, has the same mechanics as multiple shooting for trajectory fitting, with the node values pinned to the data.  This is the limiting case of the classical multiple-shooting method of Bock for parameter identification \cite{Bock1981} and optimal control \cite{BockPlitt1984}, in which the node values are decision variables coupled by matching constraints.  Its neural-ODE incarnations \cite{TuranJaeschke2022} and segment-wise training are by now common in scientific machine learning.  Classical MPC theory is surveyed in \cite{CamachoBordons}; boundary-value and shooting methods in \cite{Keller}; receding-horizon numerics for parameter identification appear in \cite{VeldmanBorkowskiZuazua}, and turnpike phenomena for neural ODEs in \cite{GeshkovZuazua22,EstevYagGesh23,RuizBaletZuazua23,LiuZuazua25}.

On the Floquet side, stability-regularized neural ODEs have been studied through Lyapunov functions \cite{jimenez2022lyanet,kang2021sodef,kolter2019learning}, which target equilibria rather than periodic orbits.  Limit-cycle learning appears in \cite{NawazLMF24,Wilson24}, orbital stability has been trained through transverse contraction criteria \cite{ZhangCheng26}, and limit cycles with prescribed trajectories and phase response have been designed through Floquet theory directly \cite{NamuraIshiiNakao24}.  These embed stability into training or design in various forms; what is new here is not a stability-constrained objective but the propagation law that a certified return map yields, together with the obstruction that separates orbital from trajectory accuracy for the periodic architecture.  Closest to our infinite-horizon question, \cite{SagodiPark26} proves universal approximation of Morse--Smale systems on $[0,\infty)$ in an $\varepsilon$-$\delta$ sense that excludes an initial-condition set of small measure; that result is existential, whereas our bounds are conditional on quantities measured on the trained model.  On the multiple-shooting side, differentiable shooting layers \cite{MassaroliPoli21} use the same partition mechanics for parallelism rather than for an approximation guarantee.  Matzakos \cite{Matzakos26} proves that activation saturation obstructs the transverse contraction required by dissipative limit cycles in standard neural ODEs; our architecture counters it with a periodic time encoding, and Section~\ref{ssec:num-cert} measures how much each ingredient contributes.

The linear-in-$k$ structure of Theorem~\ref{thm:floquet} parallels the error-growth theory of Cano and Sanz-Serna \cite{CanoSanzSerna1997} for one-step integration of attracting periodic orbits; see also \cite{Calvo2011,HairerLubichWanner2006}.  The parallel is structural only.  In that theory the per-period error is a discretization error that vanishes with the step size, and the contraction of the numerical return map is inherited from a convergent scheme applied to the true field.  Here the field is learned: the per-period error is an approximation error that no step refinement removes, and the return map of the learned flow need not contract at all.  Contraction is therefore not automatic.  It must be certified on the trained model, or trained in directly by the Floquet loss (Assumption~\ref{asm:cert}, Section~\ref{ssec:cert}).  As in that theory, the learned system has its own periodic orbit, its own return time, and a neutral phase direction that no transverse contraction controls.  Separating the contracted transverse error from the accumulating phase mismatch is the technical core of Theorem~\ref{thm:floquet}.  Floquet theory itself is classical \cite{Floquet1883}; a modern account is \cite[\S 2.4]{Chicone}.

\subsection{Organization}\label{ssec:org}
Section~\ref{sec:prelim} fixes the SA-NODE class and quantifies the time horizon barrier.  Section~\ref{sec:mpc} develops the model predictive strategy: the composite flow, the adaptive algorithm, and the error analysis. Section~\ref{sec:floquet} develops the Floquet strategy: the geometric setting, the main theorem, and the two certification routes. Section~\ref{sec:numerics} reports the numerical experiments, one per theoretical claim.  Section~\ref{sec:conclusions} concludes and states three open problems.  The proofs are deferred to Appendix~\ref{app:proofs}, one part per result: Appendix~\ref{app:pf-flowfield} proves Lemma~\ref{lem:flow-to-field}, Appendix~\ref{app:pf-mpc} the results of Section~\ref{sec:mpc}, and Appendices~\ref{app:pf-lemmas}--\ref{app:pf-orbital} the results of Section~\ref{sec:floquet}.

\section{Preliminaries: the SA-NODE class and the time horizon barrier}
\label{sec:prelim}

This section recalls the SA-NODE class and its approximation theorem, and isolates how the approximation constant grows with the horizon.  Section~\ref{ssec:sanode} states the theorem together with the explicit constant \eqref{eq:uap-const}.  Section~\ref{ssec:barrier} shows that the two mechanisms behind that constant are sharp.  They are the obstruction that Sections~\ref{sec:mpc} and~\ref{sec:floquet} remove.

\subsection{The SA-NODE class and its approximation theorem}
\label{ssec:sanode}
We recall the approximation theorem whose constant is the object of the paper, in the notation used throughout.  Fix a compact $\Ksc\subset\R^d$, a horizon $T>0$, a hidden width $P\in\mathbb{N}$, and a non-polynomial activation $\sigma\in C^0(\R)$ applied componentwise as $\bsigma$.  The SA-NODE \eqref{eq:sanode} then has $Pd(d+3)$ trainable parameters; we write $\widehat\Phi_{\bTheta}(t;\bx_0)$ for its flow, the solution $\bx(t)$ of \eqref{eq:sanode} from $\bx(0)=\bx_0$, and
\[
  \Ksc_T \;:=\; \{\Phi(t;\bx_0) : \bx_0\in\Ksc,\ t\in[0,T]\}
\]
for the reachable tube of the target flow.  The following is Theorem~2.3 of \cite{LLLZ24}, reproduced for completeness; the constant is the object of interest.

\begin{theorem}[universal approximation, {\cite[Theorem~2.3]{LLLZ24}}]\label{thm:uap}
Let $\Fsc\in C^0(\R^d\times[0,T])$ be Lipschitz in $\bx$ uniformly in $t$, with constant $L$, and belong to $H^s_{\mathrm{loc}}(\R^d\times[0,T])$ with $s>(d+1)/2+2$; let $\sigma$ be the ReLU activation.  For every $P\geq 3$ there exist parameters $\bTheta_P$ such that \eqref{eq:uap} holds with a constant $C_{T,\Ksc,\Fsc}$ independent of $P$ and of $\bx_0$.
\end{theorem}

The smoothness threshold $s>(d+1)/2+2$ is the $C^2$-embedding threshold in $d+1$ variables, matching the fact that \eqref{eq:sanode} is a shallow network in the joint input $(\bx,t)$.  The activation is fixed to ReLU because the uniform rate rests on the $L^\infty$ approximation estimate for ReLU networks.  For activations that are twice weakly differentiable and satisfy
\[
  \int_\R\abs{\sigma''(x)}\,(1+\abs{x})\,\mathrm{d}x<\infty,
\]
the sigmoid among them, the same rate holds with the supremum over initial conditions weakened to an $L^2$ average in $\bx_0$ \cite[Remark~3.7]{LLLZ24}.

The horizon enters through the constant, which \cite[Remark~2.6]{LLLZ24} makes explicit in the normalization $\Ksc=[-1,1]^d$ and $T\geq 1$, for $\Fsc$ uniformly $L$-Lipschitz in $\bx$ and in $H^{d/2+3}_{\mathrm{loc}}(\R^{d+1};\R^d)$.  The spacetime reachable set is then contained in the box $\Omega_{L,T}:=[-Te^{LT},Te^{LT}]^{d+1}$, under the normalization on $\norm{\Fsc(\mathbf{0},\cdot)}$ implicit in the growth lemma of \cite[\S 3.3]{LLLZ24}, and
\begin{equation}\label{eq:uap-const}
  C_{T,\Ksc,\Fsc}
  \;\leq\;
  C_d\,T\,\norm{\Fsc}_{*}\,
  \exp\Bigl(\tfrac{5}{2}LT+\sqrt{d}\,L
            +C_d\,e^{\frac{3}{2}LT}\norm{\Fsc}_{*}\Bigr),
  \qquad
  \norm{\Fsc}_{*}:=\|\Fsc\|_{H^{d/2+3}(\Omega_{L,T})},
\end{equation}
with $C_d$ depending only on $d$.  This bound is double exponential in $T$: the Sobolev norm is taken on a domain that already grows like $e^{LT}$, and it re-enters the Gr\"onwall exponent through the Lipschitz constant of the learned field.

\subsection{Sharpness of the factor
\texorpdfstring{$e^{LT}$}{exp(LT)}}\label{ssec:barrier}
Two mechanisms drive the growth in \eqref{eq:uap-const}, and both are already visible at the level of the flow.  We isolate them here because the strategies of Sections~\ref{sec:mpc} and~\ref{sec:floquet} act on the first, which is Gr\"onwall amplification: if $\norm{f_{\bTheta}-\Fsc}\leq\delta$ on a neighborhood of $\Ksc_T$ that contains both flows on $[0,T]$, with $L$ the Lipschitz constant of $\Fsc$ there, then for every $\bx_0\in\Ksc$
\begin{equation}\label{eq:gronwall}
  \sup_{t\in[0,T]}\norm{\widehat\Phi_{\bTheta}(t;\bx_0)-\Phi(t;\bx_0)}
  \;\leq\; \delta\,\frac{e^{LT}-1}{L}
  \;\leq\; \delta\, T e^{LT},
\end{equation}
and the first inequality is an equality for $\dot\bx=L\bx$ with a constant field perturbation, so the exponential bound is attained.  For $L=0$ the quotient is read as its limit $T$.  The second mechanism is the growth of the reachable tube: two solutions started $\mathrm{diam}(\Ksc)$ apart may separate by the factor $e^{LT}$, and bounded forcing contributes a term of the same exponential order, so
\begin{equation}\label{eq:diam}
  \mathrm{diam}(\Ksc_T)\;\leq\;\bigl(\mathrm{diam}(\Ksc)+C\bigr)\,e^{LT}
\end{equation}
for fields with $\sup_{t}\norm{\Fsc(t,\mathbf{0})}<\infty$, with $C$ depending on this bound, on $L$, and on $\sup_{\bx\in\Ksc}\norm{\bx}$; this order is attained for $\dot\bx = L\bx$.  Neither mechanism is an artifact of a loose estimate, since each is attained on that example.  Attainment concerns the two steps of the bound, not the approximation problem: on $\dot\bx=L\bx$ the class \eqref{eq:sanode} is exact, since a ReLU network represents a linear field with $Lx=L\sigma(x)-L\sigma(-x)$, so the best SA-NODE error there is zero while \eqref{eq:uap-const} is large.  The example shows the estimate is tight step by step, not that the approximation error grows.  The two do not merely multiply in \eqref{eq:uap-const}.  The Gr\"onwall factor sits in the exponent, and the tube growth enters the same exponent through the Sobolev norm on $\Omega_{L,T}$, which is what makes that bound double exponential.  The width that \eqref{eq:uap-const} certifies to be sufficient for accuracy $\varepsilon$ is of order $C_{T,\Ksc,\Fsc}^{2}\,\varepsilon^{-2}$.  That prescribed width therefore deteriorates at the same rate, the Gr\"onwall factor alone contributing $e^{2LT}$.  Both statements concern the bound and the width it prescribes, not the smallest width that suffices, which is unknown.  Whether every SA-NODE must incur the factor is a separate question, and no matching lower bound over Lipschitz targets is known.  The next lemma says what such a bound would have to control.  It converts closeness of flows into closeness of fields, at the cost of a square root, and so reduces a lower bound for the flow problem to one for approximation of the target field on the spacetime reachable set by the class \eqref{eq:sanode}.

\begin{lemma}[flow closeness forces field closeness]\label{lem:flow-to-field}
Write $\Psi_{\bTheta}(t;s,\bx)$ and $\Psi(t;s,\bx)$ for the solutions at time $t$ of \eqref{eq:sanode} and of the target, launched from $\bx$ at time $s$, so that $\widehat\Phi_{\bTheta}(t;\bx)=\Psi_{\bTheta}(t;0,\bx)$ and $\Phi(t;\bx)=\Psi(t;0,\bx)$.  Let $T\geq1$ and $r>0$, and let $\Fsc$ and $f_{\bTheta}$ be $C^2$ with $\norm{\Fsc}_{C^2},\norm{f_{\bTheta}}_{C^2}\leq B$ on $N_r\times[0,T]$, where $N_r$ is the closed $r$-neighborhood of $\Ksc_T$ and the norms are taken jointly in $(\bx,t)$; set $C_2:=2B(B+1)$.  If $\varepsilon\leq C_2e^{-B}/4$, if $(1+e^{B})\varepsilon\leq r/2$, and if
\[
  \sup_{\bx_0\in\Ksc,\ t\in[0,T]}\norm{\widehat\Phi_{\bTheta}(t;\bx_0)-\Phi(t;\bx_0)}\;\leq\;\varepsilon ,
\]
then for every $t_0\in[0,T-1]$ and every $\bx\in\Phi(t_0;\Ksc)$,
\begin{equation}\label{eq:flow-to-field}
  \norm{f_{\bTheta}(\bx,t_0)-\Fsc(t_0,\bx)}\;\leq\;2\sqrt{\varepsilon\,e^{B}C_2}\, .
\end{equation}
\end{lemma}

The proof is in Appendix~\ref{app:pf-flowfield}.

By \eqref{eq:flow-to-field}, a width lower bound for the flow problem follows from one for uniform approximation of $\Fsc$ on the spacetime reachable set by the fields \eqref{eq:sanode}, which are shallow networks in the joint variable $(\bx,t)$.  Problem~\ref{prob:lower} states the question in that form.  The linear field $\dot\bx=L\bx$ cannot witness it: with $\sigma$ the ReLU, $L\bx=L\sigma(\bx)-L\sigma(-\bx)$ componentwise, so it is represented exactly at width $2$, one unit for each sign and its approximation error is zero at every horizon.  A witness must therefore be a target whose field is hard to approximate on a domain that its own flow makes large.  Any algorithm relying on \eqref{eq:uap-const} for a single approximant trained once on $[0,T]$ must therefore allow its width to grow exponentially.  The remedy is to change the learning problem: partition the horizon, or exploit the orbit structure of the target. These are the two routes of Sections~\ref{sec:mpc} and~\ref{sec:floquet}.

\section{The model predictive strategy}\label{sec:mpc}

The Gr\"onwall mechanism behind \eqref{eq:uap-const} is a consequence of global error accumulation, not of the SA-NODE class.  This section replaces the single model on $[0,T]$ by a composite of short-horizon models, each trained against the true trajectory at its left endpoint.  The data reset severs the error chain at every switch time, and the Gr\"onwall factor is confined to a single window.  The main results are the uniform data-assisted bound of Theorem~\ref{thm:linearT} and the predicted-IC bound of Proposition~\ref{prop:novel-ic}, which shows what is lost when the resets are withheld.

\subsection{The composite flow and the data reset}\label{ssec:composite}
This subsection defines the composite model and its two evaluation modes, restarted from data or from its own predictions.  The distinction between the two carries the whole analysis.  Let $0=\tau_0<\tau_1<\dots<\tau_N=T$ be a partition of $[0,T]$, and write $\Tmax$ for a bound on its mesh, $\max_k(\tau_{k+1}-\tau_k)\leq\Tmax$.  The analysis treats $\Tmax$ as a fixed quantity, independent of the horizon $T$; that the partitions produced in Section~\ref{sec:numerics} respect a horizon-independent mesh is reported there.  On each window $I_k:=[\tau_k,\tau_{k+1})$, closed at the right for $k=N-1$, we place one SA-NODE $f_{\bTheta_k}$ driven by the \emph{local time} $s:=t-\tau_k$,
\begin{equation}\label{eq:window-dyn}
  \dot\bx(t) \;=\; f_{\bTheta_k}\!\bigl(\bx(t),\, t-\tau_k\bigr),
  \qquad t\in I_k.
\end{equation}
Local time keeps the time input $s$ of the bias $A_i^2 s + B_i$ in $[0,\Tmax]$ rather than $[0,T]$, and it makes the warm start $\bTheta_k \leftarrow \bTheta_{k-1}$ shape-preserving.

The composite flow admits two evaluation modes, and the analysis rests on this distinction.  In \emph{data-IC mode} the true state is injected at each switch time,
\begin{equation}\label{eq:data-ic}
  \widehat\Phi(t;\bx_0)
  \;=\; \widehat\Phi_{\bTheta_k}\!\bigl(t-\tau_k;\ \Phi(\tau_k;\bx_0)\bigr),
  \qquad t\in I_k .
\end{equation}
Window $k$ then sees no upstream error: its accuracy depends on its own training only.  In \emph{predicted-IC mode} the composite prediction is propagated,
\begin{equation}\label{eq:pred-ic}
  \widehat\Phi^{\mathrm{pred}}(t;\bx_0)
  \;=\; \widehat\Phi_{\bTheta_k}\!\bigl(t-\tau_k;\
  \widehat\Phi^{\mathrm{pred}}(\tau_k;\bx_0)\bigr),
  \qquad t\in I_k ,
\end{equation}
which is the classical single-shooting chain.  Mode \eqref{eq:data-ic} is the training-time and validation-time regime; mode \eqref{eq:pred-ic} is deployment on a novel initial condition.  Theorem~\ref{thm:linearT} concerns the first, Proposition~\ref{prop:novel-ic} the second. Figure~\ref{fig:mpc-arch} shows the three ingredients: the colored bands carry the per-window fields $f_{\bTheta_k}$, the curved arrows copy parameters between adjacent blocks (the warm start), and the vertical arrows inject the true state at each switch time (the data reset).

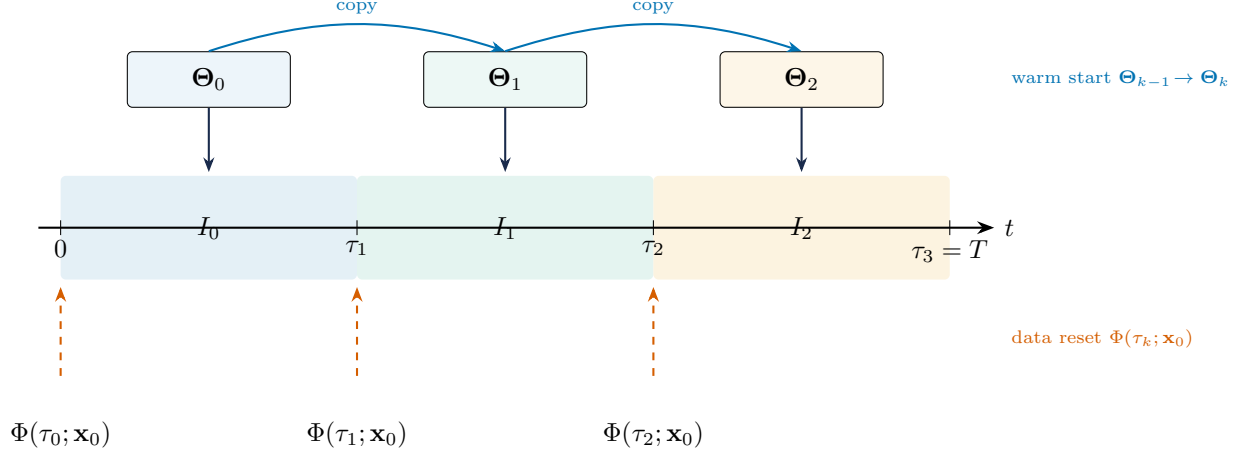
\begin{figure}[t]
\centering
\resizebox{\linewidth}{!}{%
\begin{tikzpicture}[>={Stealth[length=2mm]},font=\small,
    band/.style={rounded corners=2pt,opacity=0.75},
    param/.style={draw,rounded corners=2pt,inner sep=3pt,
                  minimum width=2.2cm,minimum height=0.75cm,fill=white},
    icarrow/.style={dashed,->,cverm,thick},
    warm/.style={->,cblue,thick,bend left=18},
    axisline/.style={-Stealth,thick}
  ]
  \def\ytop{0.7}
  \def\ybot{-0.7}
  \draw[axisline] (-0.3,0) -- (12.6,0) node[right]{$t$};
  \foreach \px/\lbl in {0/0,4/{\tau_1},8/{\tau_2},12/{\tau_3=T}} {
    \draw (\px,-0.1) -- (\px,0.1);
    \node[below=1pt] at (\px,0) {$\lbl$};
  }
  \begin{scope}[on background layer]
    \fill[cblue!14,band]  (0,\ybot) rectangle (4,\ytop);
    \fill[cgreen!14,band] (4,\ybot) rectangle (8,\ytop);
    \fill[camber!16,band] (8,\ybot) rectangle (12,\ytop);
  \end{scope}
  \node at (2,0)  {$I_0$};
  \node at (6,0)  {$I_1$};
  \node at (10,0) {$I_2$};
  \node[param,fill=cblue!7]  (T0) at (2,2.0)  {$\bTheta_0$};
  \node[param,fill=cgreen!7] (T1) at (6,2.0)  {$\bTheta_1$};
  \node[param,fill=camber!9] (T2) at (10,2.0) {$\bTheta_2$};
  \foreach \nm/\px in {T0/2,T1/6,T2/10}
    \draw[->,thick,cnavy] (\nm.south) -- (\px,\ytop+0.05);
  \draw[warm] (T0.north) to node[above]{\tiny copy} (T1.north);
  \draw[warm] (T1.north) to node[above]{\tiny copy} (T2.north);
  \node[below=8pt] at (0,\ybot-1.5)  {$\Phi(\tau_0;\bx_0)$};
  \node[below=8pt] at (4,\ybot-1.5)  {$\Phi(\tau_1;\bx_0)$};
  \node[below=8pt] at (8,\ybot-1.5)  {$\Phi(\tau_2;\bx_0)$};
  \draw[icarrow] (0,\ybot-1.3) -- (0,\ybot-0.1);
  \draw[icarrow] (4,\ybot-1.3) -- (4,\ybot-0.1);
  \draw[icarrow] (8,\ybot-1.3) -- (8,\ybot-0.1);
  \node[anchor=west] at (12.7,2.0)
    {\tiny\textcolor{cblue}{warm start $\bTheta_{k-1}\!\to\bTheta_k$}};
  \node[anchor=west] at (12.7,-1.5)
    {\tiny\textcolor{cverm}{data reset $\Phi(\tau_k;\bx_0)$}};
\end{tikzpicture}}
\caption{The MPC--SA-NODE composite: one SA-NODE $\bTheta_k$ per band $I_k$, warm-start copies (curved arrows), and data resets at the switch times (dashed arrows, mode \eqref{eq:data-ic}); at deployment the dashed arrows are replaced by \eqref{eq:pred-ic}.}
\label{fig:mpc-arch}
\end{figure}

\subsection{The adaptive algorithm}\label{ssec:alg}
This subsection specifies how the partition is chosen at run time, through the stopping rule \eqref{eq:eps-crossing}, and states the training procedure as Algorithm~\ref{alg:mpc}.  The partition is not fixed in advance.  After window $k$ is trained, the algorithm simulates it on the remaining horizon and declares $\tau_{k+1}$ to be the first time the supervised error crosses the tolerance,
\begin{equation}\label{eq:eps-crossing}
  \tau_{k+1} \;=\; \inf\Bigl\{t\in(\tau_k,T] :
  \sup_{\bx_0\in\Ksc}
  \norm{\Phi(t;\bx_0)-\widehat\Phi_{\bTheta_k}(t-\tau_k;\Phi(\tau_k;\bx_0))}
  > \varepsilon\Bigr\},
\end{equation}
with $\inf\varnothing = T$.  A minimum window length $s_{\min}\Delta t$, with $\Delta t$ the sampling step of the training data and $s_{\min}$ an integer, excludes degenerate windows, and a window cap $S$ bounds the total count.  Training window $k$ on the full remaining interval would waste effort on times covered by a later window, so the loss is evaluated only on $[\tau_k,\min(\tau_k+H,T)]$, where $H>0$ is a \emph{training horizon} playing the role of the MPC prediction horizon.  The stopping rule \eqref{eq:eps-crossing} is still evaluated on the full remaining interval.  Truncating the training interval at $H$ changes the constant of a window but not the logic of the reset, so the analysis below is unaffected by the choice of $H$.

The stopping rule and the reset together produce a sawtooth error profile, measured in Figure~\ref{fig:pend-sweeps}(a): a monolithic model accumulates error over the whole horizon, while the composite restarts from the data at each $\tau_k$ and its error re-enters the tolerance band.  The experiments of Sections~\ref{ssec:num-mpc} and~\ref{ssec:num-budget} measure this contrast.

\begin{algorithm}[t]
\caption{Adaptive MPC--SA-NODE training}\label{alg:mpc}
\begin{algorithmic}[1]
\Require sampled trajectory data $(t,\bx)$ on $[0,T]$, tolerance
$\varepsilon$, width $P$, epochs $N_{\mathrm{ep}}$, min window $s_{\min}$, window cap
$S$, training horizon $H$, warm-start flag
\State $k\gets 0$;\quad $t_{\mathrm{left}}\gets 0$;\quad
$\bTheta_{\mathrm{prev}}\gets\varnothing$
\While{$t_{\mathrm{left}}<T$ \textbf{and} $k<S$}
  \State open window $W_k$ on
  $[t_{\mathrm{left}},\,\min(t_{\mathrm{left}}+H,\,T)]$; copy weights
  from $\bTheta_{\mathrm{prev}}$ if warm start
  \State train $W_k$ for $N_{\mathrm{ep}}$ epochs on the supervised loss with a
  regularizer as in \cite[\S 4]{LLLZ24} (we use the extended Barron norm), with the true state at
  $t_{\mathrm{left}}$ as initial condition
  \State evaluate $W_k$ on $[t_{\mathrm{left}},T]$; find $\tau_{k+1}$
  by \eqref{eq:eps-crossing}, enforce
  $\tau_{k+1}-t_{\mathrm{left}}\geq s_{\min}\Delta t$
  \State shrink $W_k$ to $[t_{\mathrm{left}},\tau_{k+1}]$;\quad
  $t_{\mathrm{left}}\gets\tau_{k+1}$;\quad
  $\bTheta_{\mathrm{prev}}\gets\bTheta_k$;\quad $k\gets k+1$
\EndWhile
\Ensure composite $\widehat\Phi$ of $k$ windows covering
$[0,t_{\mathrm{left}}]$; $t_{\mathrm{left}}=T$ unless the window cap
binds (Section~\ref{ssec:mpc-theory})
\end{algorithmic}
\end{algorithm}

\subsection{Error analysis}\label{ssec:mpc-theory}
The analysis rests on three structural assumptions on the target and one admissibility condition on the partition the algorithm produces.  They lead to Theorem~\ref{thm:linearT}, the uniform data-assisted bound.

\begin{assumption}[Lipschitz target]\label{asm:lip}
$\Fsc$ is $L$-Lipschitz in $\bx$, uniformly in $t\geq 0$, on a neighborhood of the forward reachable tube $\bigcup_{t\geq 0}\Phi(t;\Ksc)$.  Since Theorem~\ref{thm:uap} asks for a globally Lipschitz field, we tacitly replace $\Fsc$ outside that neighborhood by a globally Lipschitz extension in the same Sobolev class; the flow from $\Ksc$ is unchanged, and the constants below depend on the extension only through its bounds.
\end{assumption}

\begin{assumption}[approximation class]\label{asm:sobolev}
$\Fsc\in H^{s}_{\mathrm{loc}}(\R^d\times[0,\infty))$ with $s>(d+1)/2+2$, so that Theorem~\ref{thm:uap} applies on every window; throughout this section the activation is ReLU, as that theorem requires.
\end{assumption}

\begin{assumption}[bounded reachable tube]\label{asm:reach}
The forward reachable tube $\Kinf:=\overline{\bigcup_{t\geq 0}\Phi(t;\Ksc)}$ is compact, and the target has bounded data on it in the sense that
\[
  \Mdata
  \;:=\;
  \sup_{\tau\geq 0}\,
  \norm{\Fsc(\cdot,\cdot+\tau)}_{H^{s}(\Kinf^{\varrho}\times[0,\Tmax])}
  \;<\;\infty ,
\]
where $\Kinf^{\varrho}$ is the closed $\varrho$-neighborhood of $\Kinf$, for a fixed $\varrho>0$ whose choice the constants below absorb.  Here $s$ is additionally taken at least $d/2+3$, the regularity at which the explicit constant \eqref{eq:uap-const} is available; the $\tau$-uniformity of the window constant below is derived from that explicit form.
\end{assumption}

Assumption~\ref{asm:reach} holds for dissipative or positively invariant targets whose data are uniformly bounded in time; for autonomous and time-periodic fields the time-uniformity of the data is automatic, so there the assumption reduces to boundedness of the reachable tube.  Boundedness of the trajectories, uniform over $\bx_0\in\Ksc$ and $t\geq0$, alone makes $\Kinf$ compact, but does not by itself bound $\Mdata$ for a genuinely time-dependent target.  The assumption fails for the barrier example $\dot\bx=L\bx$, whose reachable tube grows without bound.  Write
\begin{equation}\label{eq:window-const}
  C_{\Tmax,\Kinf,\Fsc}
  \;:=\;
  \sup_{\tau\geq 0}\, C_{\Tmax,\,\Phi(\tau;\Ksc),\,\Fsc(\cdot,\cdot+\tau)}
\end{equation}
for the uniform window constant, the supremum of the constants of Theorem~\ref{thm:uap} over all windows of length $\Tmax$ started along the flow.  The assumption makes this constant finite and independent of $T$, for two reasons.  Every window trajectory is a true trajectory of the target, so each window's tube lies in the fixed compact $\Kinf$ and its Sobolev data are measured on $\Kinf^{\varrho}\times[0,\Tmax]$; and $\Mdata$ bounds those data for every time shift at once.  The form inherited from \eqref{eq:uap-const} is exponential in $e^{cL\Tmax}\Mdata$, hence double exponential in $\Tmax$, because the Sobolev norm of the field enters the Gr\"onwall exponent through the Lipschitz constant of the learned field.  What matters here is that the bound involves only $\Tmax$, $\Mdata$ and $L$, never the horizon $T$.

\begin{assumption}[$\varepsilon$-admissible partition]\label{asm:partition}
Algorithm~\ref{alg:mpc} terminates with $\tau_N=T$, its partition has mesh $\max_k(\tau_{k+1}-\tau_k)\leq\Tmax$, and every window satisfies
\begin{equation}\label{eq:per-window}
  \varepsilon_k^{\mathrm{win}}
  \;:=\;
  \sup_{t\in I_k,\ \bx_0\in\Ksc}
  \norm{\Phi(t;\bx_0)
  -\widehat\Phi_{\bTheta_k}(t-\tau_k;\Phi(\tau_k;\bx_0))}
  \;\leq\;\varepsilon .
\end{equation}
\end{assumption}

Both conditions of Assumption~\ref{asm:partition} are properties of the partition that the algorithm produces, and both are verified on the output rather than imposed during the run; Section~\ref{sec:numerics} reports the realized mesh and the realized per-window error.  The stopping rule \eqref{eq:eps-crossing} enforces the error condition when it alone determines the partition, up to the boundary case of a first crossing exactly at $t=T$, which perturbs only the endpoint value of the last window. Two effects perturb it.  The minimum window length $s_{\min}\Delta t$ may push a window past its crossing; when it binds, Assumption~\ref{asm:partition} holds with $\varepsilon$ replaced by $\max_k\varepsilon_k^{\mathrm{win}}$, and every statement below holds verbatim with that value.  The supremum in \eqref{eq:eps-crossing} is evaluated on a finite grid of initial conditions and times.  We do not bound the gap to the continuous supremum: that would require a fill distance for the grid together with a Lipschitz bound on the error, which we do not establish.  Assumption~\ref{asm:partition} is therefore verified on the grid, and every experimental statement in Section~\ref{sec:numerics} is to be read grid-wise.  The window cap $S$ behaves differently.  If the algorithm reaches $k=S$ with $\tau_S<T$, it terminates before covering the horizon.  The final stretch $[\tau_S,T]$ then carries no trained window, so Assumption~\ref{asm:partition} fails outright rather than by a bounded factor, and Theorem~\ref{thm:linearT} does not apply.  A cap that binds indicates that the per-window width is too small for the tolerance; it is not a bounded excess to be absorbed.  Section~\ref{ssec:num-budget} exhibits this regime; the safeguard excess, when the cap does not bind, is quantified in Sections~\ref{ssec:num-mpc}--\ref{ssec:num-budget}.

\begin{theorem}[data-assisted composite approximation]\label{thm:linearT}
Under Assumptions~\ref{asm:lip}--\ref{asm:partition}, the data-IC composite flow \eqref{eq:data-ic} satisfies
\begin{equation}\label{eq:linearT}
  \sup_{t\in[0,T],\ \bx_0\in\Ksc}
  \norm{\Phi(t;\bx_0)-\widehat\Phi(t;\bx_0)}\;\leq\;\varepsilon,
\end{equation}
and a total width sufficient to achieve $\varepsilon_k^{\mathrm{win}}\leq\varepsilon$ on all $N$ windows obeys
\begin{equation}\label{eq:budget}
  P_{\mathrm{total}}
  \;=\;\sum_{k=0}^{N-1}P_k
  \;\lesssim\;
  N\, C^2_{\Tmax,\Kinf,\Fsc}\,\varepsilon^{-2},
\end{equation}
with $C_{\Tmax,\Kinf,\Fsc}$ independent of $T$.
\end{theorem}

The proof is in Appendix~\ref{app:pf-mpc}.

\begin{remark}[how the budget depends on $T$]\label{rem:budget-T}
The budget \eqref{eq:budget} is proportional to the window count, and the window count is bounded below: window lengths are at most $\Tmax$ and sum to $T$, so every admissible partition has
\begin{equation}\label{eq:N-lower}
  N \;\geq\; \lceil T/\Tmax\rceil .
\end{equation}
The uniform partition of mesh $\Tmax$ attains \eqref{eq:N-lower}, and for it the budget is exactly linear in $T$, replacing the Gr\"onwall factor $e^{2LT}$ of the monolithic width by $T/\Tmax$.  The data reset thus confines the Gr\"onwall factor to a single window; the growth of the reachable tube remains.  When Assumption~\ref{asm:reach} fails, the uniform bound on the per-window constant is lost, because the constant may grow with $T$ through its domain $\Phi(\tau_k;\Ksc)$.  For the spreading barrier example $\dot\bx=L\bx$ the certified constant inherited from \eqref{eq:uap-const} does grow, so the certified budget is super-linear in $T$, though the Gr\"onwall factor stays capped at $e^{L\Tmax}$; the smallest sufficient width does not grow there, since the linear field is exactly representable (Section~\ref{ssec:barrier}), so the statement concerns the certificate and not the class.  (Unboundedness of the tube alone does not force even this: a pure drift $\dot\bx=c$ has an unbounded tube but translation-invariant window data, and its certified budget stays linear.)  The adaptive partition of Algorithm~\ref{alg:mpc} is not uniform, and Theorem~\ref{thm:linearT} imposes no upper bound on its $N$.  The algorithm itself does impose one, $N\leq T/(s_{\min}\Delta t)+1$ from the minimum window length, but its constant is the data resolution, not the target.  Whether a target-controlled linear law survives without the safeguard is Problem~\ref{prob:NT}; Section~\ref{ssec:num-budget} measures the realized counts.
\end{remark}

\subsection{Deployment on novel initial conditions}\label{ssec:novel}
Without resets the composite reduces to a single-shooting chain, and upstream errors are amplified by each learned flow.  The resulting bound is adapted to the partition.

\begin{proposition}[predicted-IC bound]\label{prop:novel-ic}
Let
\[
  \bar L\;:=\;\max\Bigl\{L,\ \max_k\Lip_{\bx}f_{\bTheta_k}\Bigr\},
\]
the Lipschitz constants being taken on a compact neighborhood of the reachable tube large enough to contain the predicted trajectories and the data-restarted window trajectories.  This constant is finite, since each $f_{\bTheta_k}$ is globally Lipschitz, with an explicit constant in terms of its weights.  It is an a posteriori constant of the trained model and the evaluation set, not one controlled a priori by the assumptions.  Under Assumptions~\ref{asm:lip}--\ref{asm:partition}, for every $\bx_0\in\Ksc$ the predicted-IC flow \eqref{eq:pred-ic} satisfies
\begin{equation}\label{eq:novel-ic}
  \sup_{t\in[0,T]}
  \norm{\Phi(t;\bx_0)-\widehat\Phi^{\mathrm{pred}}(t;\bx_0)}
  \;\leq\;
  \varepsilon\sum_{k=1}^{N} e^{\bar L\,(T-\tau_k)}
  \;\leq\;
  \varepsilon\,\frac{e^{N\bar L\Tmax}-1}{e^{\bar L\Tmax}-1},
\end{equation}
the quotient being read as its limit $N$ when $\bar L=0$.
\end{proposition}

The proof is in Appendix~\ref{app:pf-mpc}.

\begin{remark}[reading the bound]\label{rem:novel-ic}
The bound \eqref{eq:novel-ic} has three regimes.  For the uniform partition, $N\Tmax=T$ and the quotient equals
\[
  \varepsilon\,\frac{e^{\bar LT}-1}{e^{\bar L\Tmax}-1};
\]
for non-uniform partitions $N\Tmax\geq T$, that substitution is not available, and only the $N$-dependent quotient of \eqref{eq:novel-ic} is a valid upper bound.  As $\bar L\to 0$ both bounds tend to $N\varepsilon$, so the per-window errors simply add.  For $\bar L\Tmax\gg 1$ the bound behaves as $\varepsilon\,e^{(N-1)\bar L\Tmax}$.  Deployment therefore remains exponential in the horizon; the gain over the monolithic model is the prefactor $(e^{\bar L\Tmax}-1)^{-1}$, not the exponent.  A continuity penalty at the switch times should shrink the constant of the chain; we leave its analysis open.
\end{remark}

\subsection{Scope}\label{ssec:mpc-scope}
Theorem~\ref{thm:linearT} is a statement about the data-assisted regime: it requires the true state at the $N$ switch times, which is available during training and validation but not in autonomous deployment.  The scale of $\Tmax$ is expected to be set by the local Lipschitz behavior of $\Fsc$ along the trajectory, and the experiments are consistent with this; a rigorous link to Lyapunov exponents is open (Problem~\ref{prob:NT}).  Beyond a Lipschitz target with a bounded reachable tube (Assumptions~\ref{asm:lip}--\ref{asm:reach}) the strategy makes no structural assumption: within that class it applies to non-autonomous and non-periodic systems alike, and the data dependence provides this generality.  The tube hypothesis confines the parameter budget to grow only with the window count; it is the point where the target's global behavior re-enters the estimate.

\section{The Floquet strategy}\label{sec:floquet}

The MPC strategy needs the true state at every switch time.  When the target is autonomous and its trajectories are attracted to a stable limit cycle, the dynamics itself can play the role of the reset: transverse perturbations are contracted at every return to a Poincar\'e section. This section shows how a trained SA-NODE can be made to inherit that contraction, with a certificate.  Standard neural ODEs are obstructed here: activation saturation caps the contraction their flows can express \cite{Matzakos26}.  The periodic encoding below is designed against this obstruction.  The Floquet loss supplies a differentiable surrogate that drives the learned return map towards contraction, and a partial repair when the encoding is absent (Section~\ref{ssec:num-cert}).  What that surrogate does and does not certify is the subject of Section~\ref{ssec:cert}.  The main results are Theorem~\ref{thm:floquet}, for an autonomous learned field, and Theorem~\ref{thm:orbital}, for the periodic encoding we deploy.  In the first, certified contraction of the learned return map confines the deployment error to linear growth in the number of periods; in the second, the measured monodromy spectrum yields a uniform-in-time orbital bound.

\subsection{Setting and assumptions}\label{ssec:floquet-setting}
This subsection fixes the geometry of the problem and states the standing assumptions.  The essential one is the certified contraction of Assumption~\ref{asm:cert}; the others make it meaningful.  The target is an autonomous system
\begin{equation}\label{eq:target-aut}
  \dot z = f(z), \qquad f\in C^2(\R^d;\R^d),\quad d\geq 2,
\end{equation}
with flow $\Phi_f$ and a $\That$-periodic orbit $\ghat$ with image $\Gamma:=\ghat([0,\That])$.  The following notation is used throughout this section and in Appendices~\ref{app:pf-lemmas}--\ref{app:pf-orbital}:
\begin{align*}
  p&:=\ghat(0), & v&:=\norm{f(p)}, & v_{\min}&:=\min_s\norm{f(\ghat(s))}>0,\\
  N_{\delta_0}&:=\{x:\dist(x,\Gamma)\leq\delta_0\}, & L&:=\sup_{N_{\delta_0}}\norm{Df}, &
  \Sigma_r&:=\{p+w : w\perp f(p),\ \norm{w}\leq r\}.
\end{align*}
As announced in Section~\ref{ssec:main}, the parameter vector $\bTheta$ of \eqref{eq:sanode} is abbreviated $\Th$.  Here $\delta_0>0$ is fixed small enough that the compact tube $N_{\delta_0}$ lies in the basin of attraction of $\Gamma$ and that the normal tubular chart $(s,u)\mapsto\ghat(s)+U(s)u$ is injective on it.  Such a $\delta_0$ exists because $\Gamma$ is a compact embedded $C^2$ curve, and it depends only on $f$ and $\ghat$.  The speed $v_{\min}$ is the lower bound used for the winding time in Lemma~\ref{lem:winding}, and $v$ the speed at the base point, used for the return-time estimates.  The disc $\Sigma_r$ is the transverse disc of radius $r$ at the base point $p$, inside the hyperplane $\Sigma$ through $p$ orthogonal to $f(p)$, and $\dist_H$ is the Hausdorff distance between compact sets.  Five spectral quantities appear in this section and are pairwise distinct: the transverse Floquet radius $\rhoT(f)$ of the target orbit; the prescribed contraction threshold $\rho_*$; the spectral radius $\rhoTh$ of the learned return map at its fixed point; the divergence surrogate $\tilde\rho_T$ of \eqref{eq:rho-surrogate}; and the spectral radius $\rho(M)$ of the one-period stroboscopic monodromy of the encoded flow (Remark~\ref{rem:det}).  They coincide only where explicitly stated.  We take the constant $B$ of Assumption~\ref{asm:learned} to bound both fields,
\[
  \max\bigl\{\norm{f}_{C^2(N_{\delta_0})},\ \norm{\fTh}_{C^2(N_{\delta_0}\times\R)}\bigr\}\;\leq\;B,
\]
which costs nothing, since the constants below are already allowed to depend on $f$, and which lets every statement about a general field be applied to $f$ itself.  All constants $K_1,K_2,\dots$ and $C$ below depend only on $(f,\ghat,\delta_0,B,\rho_*,\rhoT(f))$, with $\rho_*$ the contraction threshold of Assumption~\ref{asm:cert}, and never on $\varepsilon$, $t$, or the initial condition.

\begin{assumption}[hyperbolic stable orbit]\label{asm:true}
The orbit $\ghat$ is hyperbolic and orbitally asymptotically stable; its transverse Floquet spectral radius satisfies $\rhoT(f)<1$.  Throughout, $\That$ denotes the minimal period of $\ghat$; the statements that let the target's time-$\That$ map fix every point of $\Gamma$ use this convention.  Here and below the subscript in $\rhoT(\cdot)$ and in the surrogate $\tilde\rho_T$ marks the transverse direction and is unrelated to the horizon $T$ of Section~\ref{sec:mpc}.  Floquet theory originates in \cite{Floquet1883}; asymptotic stability of a periodic orbit whose nontrivial multipliers lie inside the unit circle is \cite[Theorem~2.82]{Chicone}.
\end{assumption}

The learned model must have a well-defined return map, so its field should be $\That$-periodic in time.  We enforce this in the architecture: the scalar time input of \eqref{eq:sanode} is replaced by the encoding
\begin{equation}\label{eq:periodic-encoding}
  t \;\longmapsto\;
  \bigl(\sin(2\pi t/\That),\ \cos(2\pi t/\That)\bigr),
\end{equation}
so that $\fTh(\cdot,t+\That)=\fTh(\cdot,t)$ exactly, and the model is deployed by continuous integration with no time resets.  With smooth activations, $\fTh\in C^\infty$.  Explicitly, the deployed field is
\begin{equation}\label{eq:encoded-arch}
  \fTh(\bx,t)\;=\;\sum_{i=1}^{P} W_i\circ
  \boldsymbol\sigma\Bigl(A_i^1\bx
  + C_i\bigl(\sin(2\pi t/\That),\ \cos(2\pi t/\That)\bigr)^{\!\top}
  + B_i\Bigr),
  \qquad C_i\in\R^{d\times 2},
\end{equation}
with the remaining parameters as in \eqref{eq:sanode}: the scalar-time column $A_i^2$ is replaced by the two-column block $C_i$, so the encoded model carries $Pd(d+4)$ trainable parameters instead of $Pd(d+3)$.  It is a different architecture class from the one covered by Theorem~\ref{thm:uap}, which is stated for the scalar-time ReLU model and is not invoked for it; one-period accuracy of \eqref{eq:encoded-arch} is the measured hypothesis of Assumption~\ref{asm:learned}.  Training minimizes the one-period trajectory loss
\[
  \Ltraj(\Th)\;:=\;\int_0^{\That}\norm{\PhiTh(t;p)-\ghat(t)}^2\,
  \mathrm{d}t,
\]
whose empirical form in Algorithm~\ref{alg:floquet} is the mean squared error on the training grid.

\begin{assumption}[near-autonomy on the tube]\label{asm:aut}
The learned field is $\That$-periodic in time, $\fTh(\cdot,t+\That)=\fTh(\cdot,t)$, as the encoding \eqref{eq:periodic-encoding} enforces.  Write
\[
  \bar f_{\Th}(x)\;:=\;\frac{1}{\That}\int_0^{\That}\fTh(x,s)\,
  \mathrm{d}s
\]
for the time-averaged field, and
\begin{equation}\label{eq:eta-def}
  \eta \;:=\;
  \sup_{x\in N_{\delta_0},\,t\in\R}
  \Bigl(\norm{\fTh(x,t)-\bar f_{\Th}(x)}
  +\norm{D_x\fTh(x,t)-D_x\bar f_{\Th}(x)}\Bigr)
\end{equation}
for its $C^1$ oscillation on the tube.  The case $\eta=0$ is the autonomous case, in which we write $\fTh(x)$.
\end{assumption}

\begin{remark}[role of the oscillation]\label{rem:osc-role}
The analysis of Section~\ref{ssec:floquet-main} is carried out in the autonomous case $\eta=0$, where the certificates are exact.  Lemma~\ref{lem:nearaut} removes the restriction, at the cost of replacing $\varepsilon$ by $\varepsilon+C\eta$ in conclusions (ii)--(iii) of Theorem~\ref{thm:floquet}.
\end{remark}

\begin{assumption}[flow closeness over one period]\label{asm:learned}
$\fTh\in C^2(N_{\delta_0}\times\R;\R^d)$ with $\norm{\fTh}_{C^2(N_{\delta_0}\times\R)}\leq B$.  We extend $\fTh$ to $\R^d\times\R$ by a $C^2$ cutoff outside $N_{\delta_0}$, so that $\PhiTh$ is globally defined and the supremum below is taken over trajectories that exist for all time.  The extension multiplies the $C^2$ bound by a factor depending only on $\delta_0$, which we absorb into $B$, and it is inactive on $N_{\delta_0}$, where every trajectory the analysis uses is shown to remain.  With $r_1=r_1(f,\ghat, \delta_0)$ the confinement radius of Lemma~\ref{lem:confine}, which is built from the true flow alone and does not depend on $\Th$, the one-period flow closeness is
\[
  \varepsilon \;:=\; \sup\bigl\{\norm{\PhiTh(t;x)-\Phi_f(t;x)} :
  x\in N_{2r_1},\ t\in[0,\That+2]\bigr\},
\]
and we write $\delta:=\sup_{x\in\Gamma}\norm{\fTh(x)-f(x)}$ for the orbit mismatch.  For $\eta>0$ the supremum in $\delta$ is taken over $x\in\Gamma$ and $t\in\R$, and $\PhiTh$ denotes the flow launched at phase $0$.
\end{assumption}

Assumption~\ref{asm:learned} is a hypothesis on the trained model, not a consequence of Section~\ref{sec:prelim}: Theorem~\ref{thm:uap} is stated for the ReLU activation, which is not $C^2$, so no result of this paper asserts that the $C^2$ architectures deployed in this section achieve a prescribed $\varepsilon$.  The smooth-activation variant of the approximation theorem recalled after Theorem~\ref{thm:uap} grounds achievability for twice weakly differentiable activations, and the experiments measure $\varepsilon$ directly; the analysis below takes $\varepsilon$ as given.

The two small quantities of Assumption~\ref{asm:learned} are not independent.  Flow closeness over one period controls the orbit mismatch, at the cost of a square root.  This is what lets every smallness condition below be imposed on $\varepsilon$ alone.

\begin{lemma}[the orbit mismatch is controlled by the flow closeness]
\label{lem:delta-eps}
Let Assumption~\ref{asm:learned} hold and set $C_2:=B^2+B+L\sup_{N_{\delta_0}}\norm{f}$.  If $\varepsilon\leq\min\bigl(\delta_0/4,\ C_2(\That+2)^2/2\bigr)$, then $\delta\leq\sqrt{2C_2\,\varepsilon}+2\eta$, and $\delta\leq\sqrt{2C_2\,\varepsilon}$ in the autonomous case $\eta=0$.
\end{lemma}

The proof is in Appendix~\ref{app:pf-lemmas}.

\begin{assumption}[certified transverse contraction]\label{asm:cert}
With $\PTh:\Sigma_{r_2}\to\Sigma$ the first-return map of $\fTh$ furnished by Lemma~\ref{lem:return} for $\varepsilon\leq\varepsilon_2$, which is proved without the present assumption,
\begin{equation}\label{eq:cert-cond}
  \rho\bigl(D\PTh(p)\bigr)\;\leq\;\rho_*\;<\;1 .
\end{equation}
We write $\rhoTh$ for the spectral radius of the learned return map at its fixed point.  By Lemma~\ref{lem:persist} that fixed point lies within $O(\varepsilon)$ of $p$, and $D\PTh$ is Lipschitz, so $\rhoTh\leq\rho_*+\omega(C\varepsilon)$, with $\omega$ the modulus of continuity of the spectral radius on the norm ball $\{A:\norm{A}\leq C_M\}$ that contains all return-map differentials arising here (Lemma~\ref{lem:return}).  That modulus is linear when $D\PTh(p)$ is diagonalizable and only of order $\varepsilon^{1/m}$ at a Jordan block of size $m$.  The two conditions are therefore equivalent up to that margin rather than interchangeable.  For $\eta>0$ the condition is posed for the return map of the averaged field $\bar f_{\Th}$ (Lemma~\ref{lem:nearaut}).
\end{assumption}

Assumption~\ref{asm:cert} is the key hypothesis, and it is verifiable: Section~\ref{ssec:cert} gives two sufficient conditions computable from a trained model.  Theorem~\ref{thm:floquet} should be read as a conditional stability result; we do not claim that training always achieves \eqref{eq:cert-cond}, and the experiments quantify how often it does.

Assumption~\ref{asm:learned} deserves the same warning, and it is the more demanding of the two.  It asks for one-period flow closeness at \emph{every} initial state of the tube $N_{2r_1}$, whereas the trajectory loss $\Ltraj$ is evaluated on the single base orbit through $p$.  A small training loss therefore does not establish Assumption~\ref{asm:learned}, and neither does a small Floquet loss, which concerns stability rather than accuracy.  Uniform closeness on the tube is a separate approximation hypothesis, of the kind Theorem~\ref{thm:uap} supplies for a compact set of initial conditions, and the two certification routes of Section~\ref{ssec:cert} presuppose it rather than prove it.

\subsection{Main theorem}\label{ssec:floquet-main}
The analysis tracks three quantities: the transverse distance to $\Gamma$, contracted at rate $\rho_*$ at each return; the distance between the two return maps, of order $\varepsilon$; and the cumulative phase drift between the two clocks, which the reset cannot remove.  Six lemmas make this precise; their proofs are in Appendix~\ref{app:pf-lemmas}.

\begin{lemma}[one-period confinement]\label{lem:confine}
There exists $r_1\in(0,\delta_0/4]$, depending only on $(f,\ghat,\delta_0)$, such that for every $\varepsilon\leq\delta_0/4$, all $x\in N_{2r_1}$ and $t\in[0,\That+2]$: $\Phi_f(t;x),\PhiTh(t;x)\in N_{\delta_0}$ and $\norm{\PhiTh(t;x)-\Phi_f(t;x)}\leq\varepsilon$.
\end{lemma}

\begin{lemma}[return maps and return times]\label{lem:return}
There exist $r_2\in(0,r_1]$, $\varepsilon_2>0$ and $K_1:=1+2B/v$ such that for $\varepsilon\leq\varepsilon_2$ the first-return maps $P_f,\PTh:\Sigma_{r_2}\to\Sigma$ and return times $\tau_f,\tauTh$ are well defined and $C^1$, with $DP_f, D\PTh$ Lipschitz, $P_f(p)=p$, $\tau_f(p)=\That$, and uniformly on $\Sigma_{r_2}$
\[
  \abs{\tauTh-\tau_f}\leq\tfrac{2}{v}\varepsilon, \qquad
  \norm{\PTh-P_f}\leq K_1\varepsilon, \qquad
  \tau_f,\tauTh\in[\That/2,\That+1].
\]
\end{lemma}

\begin{lemma}[persistence of the orbit]\label{lem:persist}
Let Assumptions~\ref{asm:true}, \ref{asm:learned} and~\ref{asm:cert} hold with $\eta=0$.  Set $\bar\rho:=(1+\rho_*)/2$ and $\bar\rho_f:=(1+\rhoT(f))/2$.  There are adapted norms on $T_p\Sigma$ with equivalence constants $c_1,c_2$ and $c_1',c_2'$, and radii $r_3,r_3'\in(0,r_2]$, such that $\norm{D\PTh}_*\leq\bar\rho$ on $\Sigma_{r_3}$ and $\norm{DP_f}_{*f}\leq\bar\rho_f$ on $\Sigma_{r_3'}$.  There is $\varepsilon_3\in(0,\varepsilon_2]$ such that for $\varepsilon\leq\varepsilon_3$ the map $\PTh$ has a unique fixed point $x_\Th^*\in\Sigma_{r_3}$ with $\norm{x_\Th^*-p}\leq K_2\varepsilon$, $K_2:=c_2K_1/(c_1(1-\bar\rho))$; the orbit $\gamma_\Th:=\{\PhiTh(t;x_\Th^*): 0\leq t\leq T_\Th\}$ with period $T_\Th:=\tauTh(x_\Th^*)$ is periodic, orbitally exponentially stable, and
\[
  \dist_H(\gamma_\Th,\Gamma)\leq K_3\varepsilon, \qquad
  \abs{T_\Th-\That}\leq K_4\varepsilon .
\]
If $\Ltraj(\Th)=0$ then $x_\Th^*=p$, $\gamma_\Th=\Gamma$, $T_\Th=\That$.
\end{lemma}

\begin{lemma}[monotone winding and uniqueness]\label{lem:winding}
There exist $r_5'\in(0,\delta_0/4]$, $\delta_1>0$ and $C_U\geq 1$, depending only on $(f,\ghat,\delta_0,B)$, with the following properties.

\textup{(i)} \emph{Winding and confinement.}  Let $g\in C^2(N_{\delta_0}\times\R;\R^d)$ satisfy $\norm{g}_{C^2}\leq B$ and $\sup_{\Gamma\times\R}\norm{g-f}\leq\delta\leq\delta_1$.  Then for every $r_0\leq r_5'$ and every $g$-trajectory $x(\cdot)$ starting in $N_{r_0}$, one has $\dist(x(t),\Gamma)\leq C_U(r_0+\delta)$ for all $t\in[0,\That+1]$, so the trajectory remains in the tube $N_{\delta_0}$ \emph{on that interval}, and within that same time it crosses $\Sigma$ transversally, at a point at distance at most $C_U(r_0+\delta)$ from $p$.  Nothing is claimed beyond the first winding, and nothing can be: the hypotheses bound $g-f$ on $\Gamma$ only, so the transverse dynamics of $g$ may expand, and re-entering the estimate at the crossing replaces $r_0$ by $C_U(r_0+\delta)\geq r_0$.  Confinement for all $t\geq0$ is recovered only through the contraction of Assumption~\ref{asm:cert}, one winding at a time.  This part uses no certificate: it holds for any field within $\delta_1$ of $f$ on $\Gamma$, in particular for $f$ itself, with $\delta=0$.

\textup{(ii)} \emph{Uniqueness.}  Set $r_5:=\min\bigl(r_5',\,r_3/(2C_U)\bigr)$, which depends on $\rho_*$ through $r_3$ but on nothing else beyond $(f,\ghat,\delta_0,B)$, and suppose in addition that Assumption~\ref{asm:cert} holds and that $\delta\leq r_3/(2C_U)$, which by Lemma~\ref{lem:delta-eps} follows from a smallness condition on $\varepsilon$ alone.  Then $C_U(r_0+\delta)\leq r_3$ for every $r_0\leq r_5$, and every periodic orbit of $\fTh$ contained in $N_{r_0}$ equals $\gamma_\Th$; if moreover $K_3\varepsilon\leq r_0$, then $\gamma_\Th\subset N_{r_0}$ by Lemma~\ref{lem:persist} and $\gamma_\Th$ is that unique orbit.  In particular $r_5$ does not depend on $\Th$ or on $\varepsilon$.
\end{lemma}

\begin{lemma}[two contractions track each other]\label{lem:track}
Set $K_5:=(c_2'/c_1')\cdot 2K_1/(1-\bar\rho_f)$.  There exist $r_4\in(0,\min(r_3,r_3',r_5)]$ and $\varepsilon_4\in(0,\varepsilon_3]$ such that for $\varepsilon\leq\varepsilon_4$ and all $y_f,y_\Th\in\Sigma_{r_4}$ with $\norm{y_f-y_\Th}\leq K_1\varepsilon$, the iterates remain in $\Sigma_{r_3}\cap\Sigma_{r_3'}$ and
\[
  \norm{P_f^k(y_f)-\PTh^k(y_\Th)}\;\leq\;K_5\varepsilon
  \qquad\text{for all } k\geq 0 .
\]
\end{lemma}

\begin{lemma}[clock comparison]\label{lem:clocks}
In the setting of Lemma~\ref{lem:track}, define the cumulative return times $t_0^f:=0$, $t_{n+1}^f:=t_n^f+\tau_f(P_f^ny_f)$ and analogously $t_n^\Th$.  Then for all $n\geq 0$,
\[
  \abs{t_n^\Th-t_n^f}\;\leq\;K_6\,n\,\varepsilon,
  \qquad K_6:=\tfrac{2}{v}+\Lip(\tau_f)\,K_5 .
\]
\end{lemma}

\begin{theorem}[linear long-horizon bound for autonomous targets with a stable periodic orbit]
\label{thm:floquet}
Let Assumptions~\ref{asm:true}--\ref{asm:cert} hold with $\eta=0$ in Assumption~\ref{asm:aut}.  There exist $\varepsilon_0,r_0>0$ and $C\geq 1$, depending only on $(f,\ghat,\delta_0,B,\rho_*,\rhoT(f))$ and never on $\varepsilon$, $t$, or $x_0$, such that for all $\varepsilon\leq\varepsilon_0$:
\begin{enumerate}
\item[\textup{(i)}] \textbf{Orbit persistence.}  $\fTh$ has a unique
periodic orbit $\gamma_\Th\subset N_{r_0}$, orbitally exponentially stable, with $\dist_H(\gamma_\Th,\Gamma)\leq C\varepsilon$ and $\abs{T_\Th-\That}\leq C\varepsilon$.
\item[\textup{(ii)}] \textbf{Transverse contraction.}  For every $x_0$
with $\dist(x_0,\Gamma)\leq r_0$ and every $t\geq 0$,
  \[
    \dist\bigl(\PhiTh(t;x_0),\Gamma\bigr)
    \;\leq\;
    C\Bigl(\bar\rho^{\,t/(\That+1)}\dist(x_0,\Gamma)+\varepsilon\Bigr).
  \]
\item[\textup{(iii)}] \textbf{Linear bound.}  For every such $x_0$ and
$t\geq 0$,
  \begin{equation}\label{eq:floquet-main}
    \norm{\PhiTh(t;x_0)-\Phi_f(t;x_0)}
    \;\leq\;
    C\,\varepsilon\Bigl(1+\frac{t}{\That}\Bigr).
  \end{equation}
\end{enumerate}
\end{theorem}

The proof is in Appendix~\ref{app:pf-floquet}; we summarize the mechanism.  Gr\"onwall is applied only over single periods, contributing the bounded factor $e^{L(\That+1)}$.  Transverse errors do not accumulate: Lemma~\ref{lem:track} caps them at $K_5\varepsilon$ uniformly in the number of returns.  The accumulating quantity is the phase: each period the two clocks drift apart by $O(\varepsilon)$ (Lemma~\ref{lem:clocks}), and after $n\leq 2t/\That+1$ returns the drift is $O(n\varepsilon)$, whence \eqref{eq:floquet-main}. Figure~\ref{fig:mechanism} displays the two mechanisms.

\begin{figure}[t]
\centering
\begin{subfigure}[t]{0.48\textwidth}
  \centering
  \includegraphics[width=\linewidth]{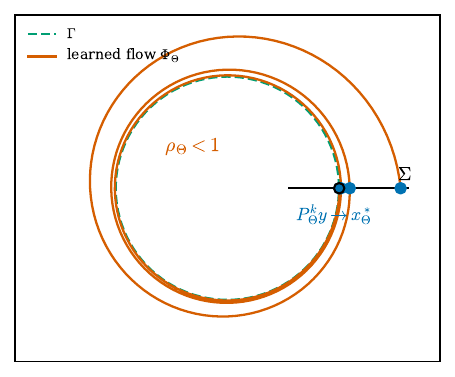}
\caption{transverse contraction}
  \label{sfig:mech-contract}
\end{subfigure}\hfill
\begin{subfigure}[t]{0.48\textwidth}
  \centering
  \includegraphics[width=\linewidth]{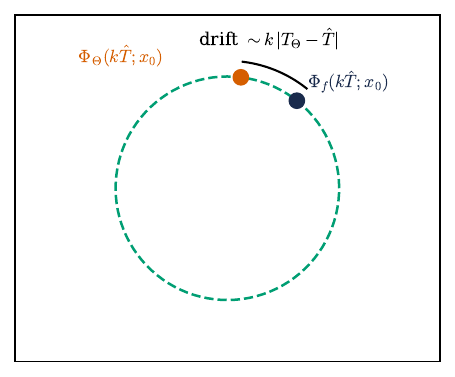}
\caption{phase drift}
  \label{sfig:mech-drift}
\end{subfigure}
\caption{The Poincar\'e reset mechanism.  (a)~Certified contraction $\rhoTh<1$: successive returns of the learned flow to the section $\Sigma$ converge geometrically to the fixed point, so transverse error cannot accumulate.  (b)~The non-resettable error: the learned clock runs at period $T_\Th\neq\That$, so the phase drifts by $k\,\abs{T_\Th-\That}$ after $k$ periods, the linear term in \eqref{eq:floquet-main}.}
\label{fig:mechanism}
\end{figure}

\begin{remark}[when the linear term is active]\label{rem:tight}
The linear term has two sources, which saturate differently.  Near the threshold $\rhoTh\to 1^-$ the transverse geometric sums lose their gap and the transverse error itself grows linearly, so the linear term is expected to be attained there, and a pure period mismatch attains it.  In polar coordinates on $\R^2$, let the target be $\dot r=-(r-1)$, $\dot\theta=2\pi/\That$, whose unit circle $\Gamma$ is a hyperbolic limit cycle with transverse multiplier $e^{-\That}$.  Let the learned field share the radial part, with $\dot\theta=2\pi/((1+\delta)\That)$ for small $\delta>0$.  Both fields are autonomous and $C^2$ near $\Gamma$, and both return maps on the section $\{\theta=0\}$ contract at rate at most $e^{-\That}$, so Assumption~\ref{asm:cert} holds.  By Gr\"onwall, the one-period flow closeness on a tube around $\Gamma$ is $\varepsilon\leq c_1\delta$, with $c_1$ depending only on $\That$.  For $x_0\in\Gamma$ both solutions remain on $\Gamma$ and differ only in phase.  After $k$ periods the angular mismatch is $2\pi k\delta/(1+\delta)$, so $\norm{\PhiTh(k\That;x_0)-\Phi_f(k\That;x_0)}=2\abs{\sin\bigl(\pi k\delta/(1+\delta)\bigr)}\geq 4k\delta/(1+\delta)$ as long as $k\delta/(1+\delta)\leq\tfrac12$, by $\sin(\pi s)\geq 2s$ on $[0,\tfrac12]$.  Since $\varepsilon\leq c_1\delta$, the stroboscopic error is at least $(2/c_1)\,\varepsilon\,k$ until the phase wraps: the envelope $C\varepsilon(1+k)$ of \eqref{eq:floquet-main} is attained up to constants.  Well inside the certified regime the transverse error saturates at $O(\varepsilon/(1-\rhoTh))$ and the phase drift dominates.  The error at period $k$ then grows like $k\abs{T_\Th-\That}$ until the phase wraps around the orbit, after which it oscillates at the level of the orbit diameter, still below the envelope \eqref{eq:floquet-main}.  When the fitted period matches $\That$ closely, the drift per period is negligible and the error saturates far below the envelope; the deployments of Section~\ref{ssec:num-cert} lie outside the scope of this remark, being governed by Theorem~\ref{thm:orbital} rather than by \eqref{eq:floquet-main}.
\end{remark}

The encoded architecture is $\That$-periodic, not autonomous, so Theorem~\ref{thm:floquet} does not apply to it verbatim.  The following lemma closes the gap.  Its proof, in Appendix~\ref{app:pf-nearaut}, has two steps.  First, the averaged field $\bar f_{\Th}$ is autonomous and satisfies Assumption~\ref{asm:learned} with $\varepsilon+C\eta$ in place of $\varepsilon$, the constant being made explicit in the lemma.  Once its return map is certified, which is the hypothesis of part~(ii) of Lemma~\ref{lem:nearaut} below, Theorem~\ref{thm:floquet} applies to it.  Second, the proof of part~(iii) of Theorem~\ref{thm:floquet} reruns with $\bar f_{\Th}$ as reference and the encoded field as the perturbed system, once the semigroup identities of the perturbed flow are replaced by evolution-family identities for the phase-indexed flows $\PhiTh^{(s)}$, whose closeness part~(i) controls uniformly in the launch phase.

\begin{lemma}[reduction of the periodic encoding]\label{lem:nearaut}
Let Assumptions~\ref{asm:true}, \ref{asm:aut} and \ref{asm:learned} hold, and set
\[
  C_B:=(\That+2)\bigl(2+B(\That+2)\bigr)e^{2B(\That+2)} .
\]
Assume that $\varepsilon+C_B\eta\leq\delta_0/4$ and that $\delta+\eta\leq v_{\min}/(2(\That+1))$.  For $s\in\R$ write $\PhiTh^{(s)}(t;x)$ for the solution of $\dot y=\fTh(y,s+t)$, $y(0)=x$: the flow launched at phase $s$.
\begin{enumerate}
\item[\textup{(i)}] For every $s\in\R$, $x\in N_{2r_1}$ and
$t\in[0,\That+2]$,
  \[
    \norm{\PhiTh^{(s)}(t;x)-\Phi_{\bar f_{\Th}}(t;x)}
    +\norm{D_x\PhiTh^{(s)}(t;x)-D_x\Phi_{\bar f_{\Th}}(t;x)}
    \;\leq\; C_B\,\eta .
  \]
In particular $\bar f_{\Th}$ satisfies Assumption~\ref{asm:learned} with closeness $\varepsilon+C_B\eta$ and orbit mismatch $\delta+\eta$.
\item[\textup{(ii)}] Suppose in addition that the averaged field is
certified, $\rho(D\bar P(p))\leq\rho_*<1$ with $\bar P$ its first-return map (the proof uses this hypothesis only through the contraction at rate $\bar\rho=(1+\rho_*)/2$ that it induces, via the adapted norm of Lemma~\ref{lem:persist}, on a neighborhood of the fixed point) and that $\varepsilon+\eta\leq\varepsilon_6$ for a threshold $\varepsilon_6>0$ depending only on $(f,\ghat,\delta_0,B,\rho_*,\rhoT(f))$.  Then three conclusions hold.  Conclusion \textup{(i)} of Theorem~\ref{thm:floquet} holds for $\bar f_{\Th}$, with $\varepsilon+C_B\eta$ in place of $\varepsilon$.  Every trajectory of the encoded flow with $\dist(x_0,\Gamma)\leq r_0$ satisfies $\dist\bigl(\PhiTh^{(0)}(t;x_0),\bar\gamma_\Th\bigr)\leq C\bigl(\bar\rho^{\,t/(\That+1)}\dist(x_0,\bar\gamma_\Th)+\eta\bigr)$ for all $t\geq0$, where $\bar\gamma_\Th$ is the orbit of $\bar f_{\Th}$; in particular the distance is eventually of order $\eta$.  Conclusions \textup{(ii)}--\textup{(iii)} hold for the encoded flow itself: for every $x_0$ with $\dist(x_0,\Gamma)\leq r_0$ and $t\geq 0$,
  \begin{equation}\label{eq:nearaut-main}
    \norm{\PhiTh(t;x_0)-\Phi_f(t;x_0)}
    \;\leq\;
    C\,(\varepsilon+\eta)\Bigl(1+\frac{t}{\That}\Bigr).
  \end{equation}
\end{enumerate}
\end{lemma}

The lemma trades the autonomy of the field for a measurable quantity: $\eta$ is a supremum of the trained network's deviation from its time average, evaluated on a grid over the tube by forward passes.  It should be read as an asymptotic bridge between the two theorems, valid only for small $\eta$.  Section~\ref{ssec:num-cert} shows that the trained models lie far outside that regime, so the deployed architecture is governed by Theorem~\ref{thm:orbital} alone.

\subsection{Certification}\label{ssec:cert}
Theorem~\ref{thm:floquet} is conditional on \eqref{eq:cert-cond}, which can be verified along two routes: an exact scalar identity in the autonomous plane, and a perturbative $C^1$ route in any dimension.  Both certify the return map of an autonomous field, which is what Assumption~\ref{asm:cert} asks for and which, when $\eta>0$, is the return map of the average $\bar f_{\Th}$.  For the time-periodic architecture we deploy we compute instead the full spectrum of the stroboscopic monodromy.  That spectrum does not verify \eqref{eq:cert-cond}, and Section~\ref{ssec:orbital} shows what it does deliver.

\begin{proposition}[$C^1$-closeness certification]\label{prop:c1-cert}
Write $g:=\fTh$ in the autonomous case $\eta=0$ and $g:=\bar f_{\Th}$ for $\eta>0$, and let $\varepsilon_1:=\norm{\fTh-f}_{C^1(U)}$ on a tubular neighborhood $U$ of $\Gamma$ (for $\eta>0$ the norm is taken on $U\times\R/\That\Z$, and averaging gives $\norm{g-f}_{C^1(U)}\leq\varepsilon_1$, so the conclusions concern the field that Assumption~\ref{asm:cert} requires).  There exist $C>0$ and $\varepsilon_*>0$, depending only on $(f,\Gamma,B)$, such that for $\varepsilon_1<\varepsilon_*$ the following hold.  The bound $B$ enters through the uniqueness radius of part~\textup{(i)}: a $C^1$ distance alone gives no modulus of continuity for $D\PTh$, and the $C^2$ bound of Assumption~\ref{asm:learned} supplies one.
\begin{enumerate}
\item[\textup{(i)}] $g$ has a unique periodic orbit $\Gamma_\Th$ near $\Gamma$, and $\dist_H(\Gamma_\Th,\Gamma)\leq C\varepsilon_1$.
\item[\textup{(ii)}] The transverse monodromy matrices satisfy $\norm{M_\Th^\perp-M_f^\perp}\leq C\varepsilon_1$.
\item[\textup{(iii)}] $\rhoTh\leq\bar\rho_f<1$, with $\bar\rho_f=(1+\rhoT(f))/2$ as in Lemma~\ref{lem:persist}.
\end{enumerate}
In particular, after one further shrinking of $\varepsilon_*$, Assumption~\ref{asm:cert} holds with $\rho_*:=(1+\bar\rho_f)/2$.  The step from \textup{(iii)} to \eqref{eq:cert-cond} is not immediate, since \textup{(iii)} bounds the spectral radius at the fixed point $x_\Th^*$ of the learned return map while \eqref{eq:cert-cond} is posed at the base point $p$.  By \textup{(i)} the two points lie within $C\varepsilon_1$ of each other and $D\PTh$ is Lipschitz on the relevant disc, so $\rho(D\PTh(p))\leq\bar\rho_f+\omega(C'\varepsilon_1)$, with $\omega$ the modulus of continuity of the spectral radius on the norm ball $\{A:\norm{A}\leq C_M\}$ of the proof of Lemma~\ref{lem:persist}; as $\bar\rho_f<1$, this falls below $(1+\bar\rho_f)/2$ once $\varepsilon_*$ is small enough.  The behavior of $\omega$ is as described after \eqref{eq:cert-cond}, so the smallness demanded of $\varepsilon_*$ depends on $f$ through it.  The conclusions concern an autonomous field: when $\eta=0$ they apply to $\fTh$ itself, and when $\eta>0$ to the average $\bar f_{\Th}$, which is what Assumption~\ref{asm:cert} requires.  For $\eta>0$ nothing is asserted here about a periodic orbit or a monodromy of the time-periodic field $\fTh$; that object is treated in Section~\ref{ssec:orbital}.
\end{proposition}

The proof, a $C^1$ perturbation argument along the reference orbit, is in Appendix~\ref{app:pf-cert}.
The second route is computational and integrates with training.  In dimension two the Liouville--Abel identity gives the derivative of the return map at its fixed point exactly (Lemma~\ref{lem:liouville}, Appendix~\ref{app:pf-lemmas}):
\[
  D\PTh(x_\Th^*)
  \;=\;
  \exp\Bigl(\int_0^{T_\Th}\operatorname{div}\fTh(\gamma_\Th(t))\,
  \mathrm{d}t\Bigr).
\]

\begin{definition}[Floquet loss]\label{def:floquet-loss}
The surrogate evaluates the divergence integral along the \emph{learned} orbit over one nominal period,
\begin{equation}\label{eq:rho-surrogate}
  \tilde\rho_T(\Th)
  \;:=\;
  \exp\Bigl(\int_0^{\That}
  \operatorname{div}_x \fTh(\ghat_\Th(t),t)\,\mathrm{d}t\Bigr),
\end{equation}
where $\ghat_\Th$ is the curve obtained by relaxing the learned flow onto its attractor and sampling one nominal period.  Two provisos attach to this definition.  First, it presupposes that the relaxation returns a closed $\fTh$-orbit contained in $N_{r_5'}$, with $r_5'$ the certificate-free radius of Lemma~\ref{lem:winding}\textup{(i)}.  Relaxation onto an attractor does not guarantee this by itself.  The implementation tests closure, rejecting the model when the endpoint gap exceeds a quarter of the orbit diameter, and rejects trajectories that leave a fixed sanity ball.  It does not test membership of the tube $N_{r_5'}$ around $\Gamma$, so the locality half of this proviso is assumed rather than verified (Section~\ref{ssec:num-cert}).  The radius is stated in the certificate-free form on purpose, since the surrogate is what the certificate will be derived from.  Second, even for a genuine orbit the surrogate is not the exact multiplier.  The Liouville--Abel identity integrates the divergence over the learned period $T_\Th$, whereas \eqref{eq:rho-surrogate} integrates over the nominal period $\That$.  The two agree when $\Ltraj(\Th)=0$, in which case $T_\Th=\That$; in general they differ by $O(\varepsilon)$, quantified in the proof of Corollary~\ref{cor:lf-cert}.  The surrogate is differentiated in $\Th$ through the analytic Jacobian of \eqref{eq:sanode}; the implicit dependence of $\ghat_\Th$ on $\Th$ is not differentiated through.  The \emph{Floquet loss} is
\begin{equation}\label{eq:floquet-loss}
  \LF(\Th)
  \;:=\;
  \mu_F\,\max\bigl(0,\ \tilde\rho_T(\Th)-\rho_*\bigr)^2,
  \qquad \mu_F>0,\ \rho_*\in(0,1).
\end{equation}
\end{definition}

\begin{corollary}[Floquet-loss certification, $d=2$]\label{cor:lf-cert}
Let $d=2$ and Assumptions~\ref{asm:true}--\ref{asm:learned} hold with $\eta=0$ in Assumption~\ref{asm:aut}.  Assume that the relaxation defining $\ghat_\Th$ in \eqref{eq:rho-surrogate} returns a closed $\fTh$-orbit contained in $N_{r}$, for a radius $r\leq r_5'$ depending only on $(f,\ghat,\delta_0,B,\rhoT(f))$ and fixed in the proof, and let $\varepsilon$ be small enough that $\delta\leq\delta_1$, which Lemma~\ref{lem:delta-eps} guarantees.  If the trained model satisfies $\tilde\rho_T(\Th)\leq\rho_*$, equivalently $\LF(\Th)=0$, then: \textup{(i)} if $\Ltraj(\Th)=0$ and the relaxed orbit passes through the base point, so that $\ghat_\Th=\Gamma$, then $\rho(D\PTh(p))=\tilde\rho_T(\Th)\leq\rho_*$ exactly; \textup{(ii)} in general the learned orbit has a return-map fixed point $x_\Th^*$ with $\norm{x_\Th^*-p}=O(\varepsilon)$ and $\abs{\rho(D\PTh(x_\Th^*))-\tilde\rho_T(\Th)}\leq C\varepsilon$, so Assumption~\ref{asm:cert} holds with contraction constant $\rho_*+C\varepsilon<1$ whenever $\varepsilon\leq(1-\rho_*)/(2C)$.  In either case Theorem~\ref{thm:floquet} applies, and the chain \eqref{eq:main-chain} holds.
\end{corollary}

The proof, in Appendix~\ref{app:pf-cert}, constructs $x_\Th^*$ from Assumptions~\ref{asm:true}--\ref{asm:learned} alone, by a fixed-point argument for the perturbed return map, so the contraction is derived and not assumed.

Under near-autonomy the surrogate transfers as well.  The divergences of $\fTh(\cdot,t)$ and of $\bar f_{\Th}$ differ by at most $d\eta$ on the tube, so along the learned orbit the one-period divergence integrals differ by at most $d\That\eta$ and the surrogates by the factor $e^{d\That\eta}=1+O(\eta)$.  The $O(\eta)$ displacement of the averaged field's orbit contributes at the same order, and is absorbed in the margin of Corollary~\ref{cor:lf-cert}(ii).  The transfer is effective only while $d\That\eta\ll 1$.  For the models of Section~\ref{ssec:num-cert}, with the measured $C^1$ oscillations, this product is about $6$ on Stuart--Landau and $35$ on van der Pol, so there the surrogate must be replaced by the spectral computation of Remark~\ref{rem:det}.

\begin{remark}[what the surrogate does and does not control when $\eta>0$]\label{rem:det}
For a deployed encoded model the field $\fTh(\cdot,t)$ is genuinely $\That$-periodic.  Let $M$ denote the one-period stroboscopic monodromy of the encoded flow along the locked orbit, with Floquet multipliers $\lambda_1,\lambda_2$.  By Liouville's formula the divergence integral \eqref{eq:rho-surrogate} returns the determinant,
\begin{equation}\label{eq:det-identity}
  \tilde\rho_T(\Th)\;=\;\det M\;=\;\lambda_1\lambda_2,
\end{equation}
whereas the stability of the locked orbit is governed by the spectral radius $\rho(M):=\max_i\abs{\lambda_i}$.  A periodic non-autonomous system has no trivial unit multiplier, so $\tilde\rho_T\leq\rho_*$ bounds the product of the multipliers and not their largest modulus.  Only in the autonomous limit $\eta=0$ does the surrogate measure a single multiplier.  The tangent multiplier is then pinned at one and excluded from the stability question, and $\det M$ equals the transverse multiplier itself (Lemma~\ref{lem:liouville}); this is the regime that Lemma~\ref{lem:nearaut} certifies.  When $\eta$ is not small the surrogate can pass a model with $\lambda_1>1>\lambda_1\lambda_2$, so a small $\tilde\rho_T$ is not by itself a proof of stability.  The reliable diagnostic is to compute both eigenvalues of $M$, one $2\times2$ variational solve along the orbit; Section~\ref{ssec:num-cert} scores every trained model this way.
\end{remark}

Beyond the autonomous plane no scalar identity controls the spectral radius.  This is the case for $d>2$, where the transverse monodromy $M_\Th^\perp\in\R^{(d-1)\times(d-1)}$ is a matrix, and already for $d=2$ when the field is periodic in time (Remark~\ref{rem:det}).  Certification then proceeds by the $C^1$ route.  Let
\[
  \delta_*
  \;:=\;
  \sup\bigl\{\delta>0:\rho(M_f^\perp+\Delta)<1
  \text{ for every }\norm{\Delta}\leq\delta\bigr\}
\]
be the stability margin of the true transverse monodromy.  For diagonalizable $M_f^\perp=V\Lambda V^{-1}$ the Bauer--Fike theorem \cite{BauerFike1960} gives $\delta_*\geq(1-\rhoT(f))/\varkappa(V)$, with $\varkappa(V):=\norm{V}\,\norm{V^{-1}}$ the eigenvector condition number.  The bound of Proposition~\ref{prop:c1-cert}(ii) then yields $\rhoTh<1$ whenever $C\varepsilon_1<\delta_*$ (Problem~\ref{prob:highdim}).

Both routes are post hoc certificates for an autonomous field, and neither asserts that training achieves a small $\LF$ or a small $\varepsilon_1$.  For the time-periodic architecture that we deploy the situation is weaker still: a small $\LF$ controls $\det M$ and not $\rho(M)$, so it does not by itself certify stability (Remark~\ref{rem:det}).  For such models we compute the full monodromy spectrum instead.  It is not the hypothesis \eqref{eq:cert-cond}, but it is exactly hypothesis \textup{(B1)} of Theorem~\ref{thm:orbital}, which is where the deployed architecture is treated.  Algorithm~\ref{alg:floquet} therefore reports both quantities at convergence, and how often the certificate is met is an empirical question, answered in Section~\ref{ssec:num-cert}.

\begin{algorithm}[t]
\caption{Floquet-regularized SA-NODE training, with post hoc monodromy diagnostics}\label{alg:floquet}
\begin{algorithmic}[1]
\Require period $\That$, one-period trajectory data from a base point on
$\Gamma$, threshold $\rho_*\in(0,1)$, weight $\mu_F>0$, epochs $N_{\mathrm{ep}}$,
orbit-refresh interval $R$
\State build the SA-NODE with the periodic encoding
\eqref{eq:periodic-encoding}; initialize $\Th$ randomly
\For{epoch $=1,\dots,N_{\mathrm{ep}}$}
  \State $\Ltraj(\Th)\gets$ MSE between the model trajectory over
  $[0,\That]$ and the data
  \State every $R$ epochs: relax the learned flow onto its attractor
  and resample the learned orbit $\ghat_\Th$ (detached)
  \State $\tilde\rho_T(\Th)\gets\exp\bigl(\int_0^{\That}
  \operatorname{div}_x\fTh(\ghat_\Th(t),t)\,\mathrm{d}t\bigr)$ by
  quadrature, differentiable through $\operatorname{div}_x\fTh$
  \State update $\Th$ by Adam on
  $\Ltraj(\Th)+\mu_F\max(0,\tilde\rho_T(\Th)-\rho_*)^2$
\EndFor
\Ensure trained model; the surrogate $\tilde\rho_T(\Th)$, the quantity
trained against; and the monodromy spectrum from one variational solve
along the converged orbit.  The model meets hypothesis \textup{(B1)} of
Theorem~\ref{thm:orbital} if $\rho(M)=\max_i\abs{\lambda_i}\leq\rho_*$,
which $\tilde\rho_T(\Th)\leq\rho_*$ does not imply (Remark~\ref{rem:det});
this is not the certificate \eqref{eq:cert-cond}
\end{algorithmic}
\end{algorithm}

The oscillation-penalized runs of Section~\ref{ssec:num-cert} add $\lambda_{\mathrm{osc}}\hat\eta^{\,2}$ to the loss of line~6, with $\hat\eta^{\,2}$ the mean over a fixed grid of tube points of the variance of $\fTh(x,\cdot)$ across a uniform grid of phases in $[0,\That)$.  It is a $C^0$ surrogate for \eqref{eq:eta-def}, which is a $C^1$ supremum, so it penalizes the oscillation without bounding it.  The grid is evaluated, matching Definition~\ref{def:floquet-loss}.  While the early learned flow is still unstable, the true training orbit serves as a fallback anchor set; the implementation records how many epochs the fallback stays active. Figure~\ref{fig:floquet-train} places the two losses in the training loop: the trajectory loss $\Ltraj$ pins the orbit to the data, the Floquet loss $\LF$ pushes the surrogate $\tilde\rho_T(\Th)$ below $\rho_*$, and the gradient of their sum, taken by backpropagation through the solver, updates $\Th$ until the surrogate is below the threshold; the spectrum is computed afterwards.

\begin{figure}[t]
\centering
\begin{subfigure}[b]{0.37\textwidth}
  \centering
  \begin{tikzpicture}[line width=0.6pt, every node/.style={font=\small}]
  \draw[cgreen, very thick, dashed] (0,0) ellipse (1.9cm and 1.25cm);
  \node[cgreen!60!black] at (1.35, 1.12) {$\Gamma$};
  \draw[->, cgreen, very thick] (0.05,1.25) -- (0.48,1.23);
  \draw[cverm, thick] (0.10,0.07) ellipse (1.88cm and 1.40cm);
  \draw[->, cverm, thick] (0.15,1.47) -- (-0.28,1.46);
  \draw[cnavy, thick] (1.9,-0.62) -- (1.9,0.62);
  \node[cnavy, font=\footnotesize, right] at (1.94,0.55) {$\Sigma$};
  \filldraw[cgreen!60!black] (1.9,0) circle (2.2pt);
  \node[cnavy, font=\footnotesize, below right] at (1.9,-0.02) {$p$};
  \foreach \dy in {0.30,0.16,0.04,-0.12,-0.24}{%
    \filldraw[cverm] (1.9,\dy) circle (1.8pt);}
  \node[cverm, font=\footnotesize, anchor=south] at (0.35,1.56)
    {returns $\PTh^k p$};
  \draw[->, cverm, thin, shorten >=3pt] (1.0,1.50) -- (1.87,0.28);
  \draw[cgreen, very thick, dashed] (-1.85,-2.35) -- (-1.2,-2.35);
  \node[right, font=\footnotesize] at (-1.15,-2.35) {true orbit $\ghat$};
  \draw[cverm, thick] (-1.85,-2.68) -- (-1.2,-2.68);
  \node[right, font=\footnotesize] at (-1.15,-2.68) {SA-NODE $\fTh$};
\end{tikzpicture}
\caption{geometric setting}
  \label{sfig:floquet-geom}
\end{subfigure}\hfill
\begin{subfigure}[b]{0.59\textwidth}
  \centering
  \resizebox{0.72\linewidth}{!}{%
  \begin{tikzpicture}[
  >={Stealth[length=2mm]}, line width=0.6pt,
  every node/.style={font=\small},
  fwdbox/.style={draw=cblue!70!black, rounded corners=4pt, fill=cblue!7,
                 minimum width=3.6cm, minimum height=0.7cm,
                 text centered, inner sep=4pt},
  lossbox/.style={draw=camber!80!black, rounded corners=4pt, fill=camber!10,
                  minimum width=2.2cm, minimum height=0.7cm,
                  text centered, inner sep=4pt},
  certbox/.style={draw=cgreen!60!black, rounded corners=4pt, fill=cgreen!9,
                  minimum width=3.6cm, minimum height=0.7cm,
                  text centered, inner sep=4pt},
  arr/.style={->, thick, cnavy},
  bkarr/.style={->, thick, dashed, cverm},
]
  \node[fwdbox] (ic)   at (0,2.7) {base point $p\in\Gamma$};
  \node[fwdbox] (fwd)  at (0,1.65){integrate $\fTh$ over $[0,\That]$};
  \node[fwdbox] (true) at (0,0.6) {true orbit $\ghat|_{[0,\That]}$};
  \node[lossbox] (lt) at (-1.6,-0.55) {$\Ltraj$};
  \node[lossbox] (lf) at ( 1.5,-0.55) {$\LF:\tilde\rho_T(\Th)$};
  \node[fwdbox] (ltot) at (0,-1.6) {$\Ltraj+\LF$};
  \node[fwdbox] (adj)  at (0,-2.65){backpropagation $\nabla_{\!\Th}\mathcal{L}$};
  \node[certbox] (upd) at (0,-3.7) {$\Th\!\leftarrow\!\Th-\alpha\nabla_{\!\Th}\mathcal{L}$};
  \draw[arr] (ic) -- (fwd);
  \draw[arr] (fwd) -- (true);
  \draw[arr] (true.south) -- ++(0,-0.24) -| (lt.north);
  \draw[arr] (true.south) -- ++(0,-0.24) -| (lf.north);
  \draw[arr] (lt) -- (ltot);
  \draw[arr] (lf) -- (ltot);
  \draw[arr] (ltot) -- (adj);
  \draw[arr] (adj) -- (upd);
  \draw[bkarr] (upd.east) -- ++(1.45,0) -- ++(0,5.35) -- (fwd.east);
  \node[cverm, font=\footnotesize, right] at (3.3,-1.0) {update $\Th$};
  \node[certbox, font=\footnotesize] (cert) at (0,-4.75)
    {training target: $\tilde\rho_T(\Th)\leq\rho_*$};
  \draw[arr, cgreen!60!black] (upd) -- (cert);
\end{tikzpicture}}
\caption{optimization loop}
  \label{sfig:floquet-loop}
\end{subfigure}
\caption{Floquet-regularized training.  (a)~Geometric setting: the base-point orbit is the training target, and the returns $\PTh^k p$ to $\Sigma$ illustrate the contraction $\rhoTh<1$.  (b)~Optimization loop: $\Ltraj$ pins the base-point orbit to the data, $\LF$ penalizes $\tilde\rho_T(\Th)>\rho_*$, and the monodromy spectrum is computed post hoc (Remark~\ref{rem:det}).}
\label{fig:floquet-train}
\end{figure}
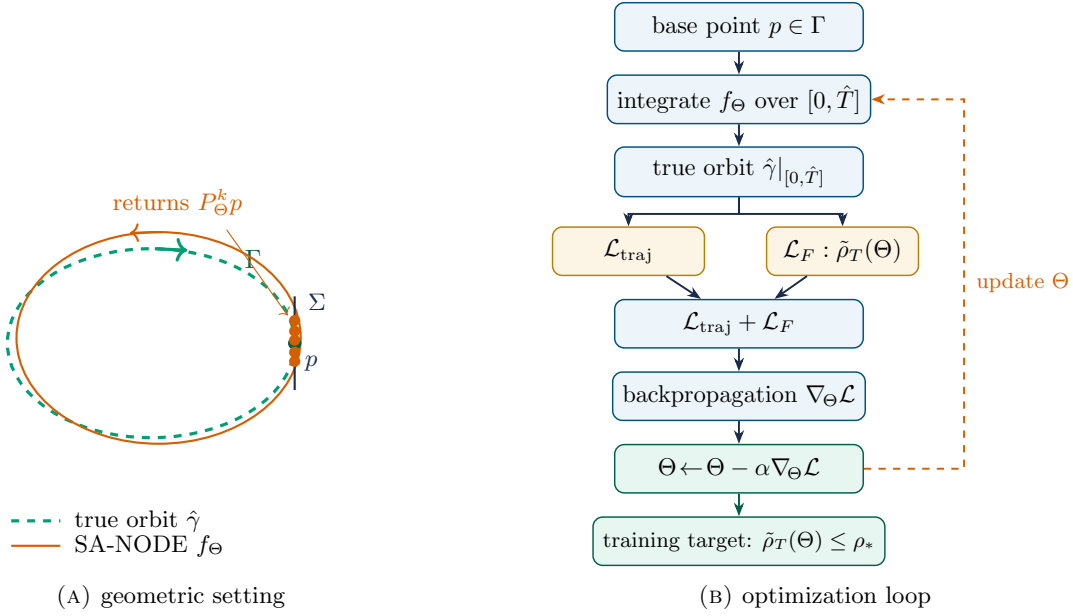

\subsection{The deployed architecture: an orbital guarantee}\label{ssec:orbital}
Theorem~\ref{thm:floquet} governs an autonomous learned field, and Lemma~\ref{lem:nearaut} carries it to the encoded field only while the oscillation $\eta$ is small.  The models we deploy are not in that regime, as Section~\ref{ssec:num-cert} reports.  This subsection settles what the encoded architecture does obey.  We first show that the obstruction is structural rather than a defect of training.  For a field that is exactly $\That$-periodic, one-period closeness on a set containing $\Gamma$ and a contraction of the stroboscopic map on that same set are quantitatively incompatible, in the norm in which the contraction holds (Proposition~\ref{prop:obstruction}).  We then prove the guarantee that survives, which is orbital rather than trajectory-wise and whose hypotheses are exactly the two quantities Algorithm~\ref{alg:floquet} already returns (Theorem~\ref{thm:orbital}).

Throughout this subsection $\fTh\in C^2(\R^d\times\R;\R^d)$ satisfies $\norm{\fTh}_{C^2}\leq B$ and is exactly $\That$-periodic in time, $\fTh(\cdot,t+\That)=\fTh(\cdot,t)$, as the encoding \eqref{eq:periodic-encoding} enforces; $\PhiTh$ is the flow launched at phase $0$.  Write
\begin{equation}\label{eq:strobo}
  \Sth \;:=\; \PhiTh(\That;\cdot)
\end{equation}
for the \emph{stroboscopic map} at phase $0$.  Periodicity gives the semigroup identity
\begin{equation}\label{eq:strobo-semigroup}
  \PhiTh(n\That+s;x) \;=\; \PhiTh\bigl(s;\Sth^{\,n}x\bigr),
  \qquad n\in\mathbb{N},\quad s\geq 0,
\end{equation}
proved by induction on $n$: for $n=1$ the function $z(s):=\PhiTh(\That+s;x)$ solves $\dot z=\fTh(z,\That+s)=\fTh(z,s)$ with $z(0)=\PhiTh(\That;x)=\Sth x$, so uniqueness gives $\PhiTh(\That+s;x)=\PhiTh(s;\Sth x)$; evaluating this at $s=\That$ advances the induction.  Note that $\Phi_f(\That;\cdot)$ is the \emph{identity} on $\Gamma$, because $\Gamma$ is a periodic orbit of $f$ of period exactly $\That$.  The target's stroboscopic map therefore fixes a whole curve, while a contracting $\Sth$ fixes one point.  That contrast is the content of the next proposition.

{\begin{proposition}[contraction obstructs one-period closeness]\label{prop:obstruction}
Let $\Gamma$ be a $\That$-periodic orbit of \eqref{eq:target-aut} and let $\fTh$ be exactly $\That$-periodic.  Suppose there are a norm $\norm{\cdot}_\diamond$, a set $D\supseteq\Gamma$ and a constant $q<1$ with $\norm{\Sth x-\Sth y}_\diamond\leq q\norm{x-y}_\diamond$ for all $x,y\in D$, and suppose $\Sth$ has a fixed point $x_\Th^\dagger\in D$.  Then
\begin{equation}\label{eq:obstruction}
  \sup_{x\in\Gamma}\norm{\PhiTh(\That;x)-\Phi_f(\That;x)}_\diamond
  \;\geq\;
  \tfrac12(1-q)\,\operatorname{diam}_\diamond(\Gamma).
\end{equation}
In particular, since $\Gamma\subseteq N_{2r_1}$ and $\That\in[0,\That+2]$, if $c_1\norm{\cdot}\leq\norm{\cdot}_\diamond\leq c_2\norm{\cdot}$ then the constant $\varepsilon$ of Assumption~\ref{asm:learned} satisfies $\varepsilon\geq\tfrac12(c_1/c_2)(1-q)\operatorname{diam}(\Gamma)$, and $\varepsilon\geq\tfrac12(1-q)\operatorname{diam}(\Gamma)$ when the contraction holds in the Euclidean norm.
\end{proposition}}

\begin{proof}
The proof is three lines: the target's stroboscopic map fixes every point of $\Gamma$, so one-period closeness forces $\Sth$ to move every point of $\Gamma$ by little, while a contraction pulls every point of $\Gamma$ to within $\sigma/(1-q)$ of its fixed point.  Write $\sigma:=\sup_{x\in\Gamma}\norm{\Sth x-x}_\diamond$.  Since $\Phi_f(\That;x)=x$ for $x\in\Gamma$, the left side of \eqref{eq:obstruction} is exactly $\sigma$.  For $x\in\Gamma\subseteq D$,
\[
  \norm{x-x_\Th^\dagger}_\diamond
  \;\leq\;\norm{x-\Sth x}_\diamond+\norm{\Sth x-\Sth x_\Th^\dagger}_\diamond
  \;\leq\;\sigma+q\norm{x-x_\Th^\dagger}_\diamond ,
\]
so $\norm{x-x_\Th^\dagger}_\diamond\leq\sigma/(1-q)$ for every $x\in\Gamma$, whence $\operatorname{diam}_\diamond(\Gamma)\leq 2\sigma/(1-q)$.  This is \eqref{eq:obstruction}.  For the last claim, Assumption~\ref{asm:learned} bounds $\norm{\PhiTh(t;x)-\Phi_f(t;x)}$ by $\varepsilon$ for $x\in N_{2r_1}\supseteq\Gamma$ and $t\in[0,\That+2]\ni\That$, and $\norm{\cdot}_\diamond\leq c_2\norm{\cdot}$ converts \eqref{eq:obstruction} into a bound on $\varepsilon$ after dividing by $c_2$ and using $\operatorname{diam}_\diamond\geq c_1\operatorname{diam}$.
\end{proof}

Proposition~\ref{prop:obstruction} assumes a contraction on a set containing $\Gamma$, which is more than the ablation of Section~\ref{ssec:num-cert} measures: there only the spectral radius at the fixed point is computed.  The next proposition removes that gap.  Its hypothesis is exactly the measured quantity, and its conclusion is the same incompatibility, at the cost of a square root and a $d$-th root in the constants.

{\begin{proposition}[local form of the obstruction]\label{prop:obstruction-local}
Let Assumption~\ref{asm:true} hold, so that the confinement radius $r_1$ of Lemma~\ref{lem:confine} is defined, and let $\norm{f}_{C^2(N_{\delta_0})}\leq B$; let $\fTh$ be exactly $\That$-periodic with $\norm{\fTh}_{C^2}\leq B$, and let $\varepsilon$ be the constant of Assumption~\ref{asm:learned}.  Let $x^\sharp$ satisfy $\dist(x^\sharp,\Gamma)\leq\varrho$ and put $M:=D\Sth(x^\sharp)$.  There are $\Lambda_1,\Lambda_2\geq1$, depending only on $(d,\That,B)$, such that if $\varepsilon\leq r_1^2\Lambda_2/2$ then
\begin{equation}\label{eq:obstruction-local}
  \rho(M)\;\geq\;1-\Bigl[\bigl(\Lambda_1\varrho+\sqrt{2\varepsilon\Lambda_2}\bigr)\bigl(1+\norm{M}\bigr)^{d-1}\Bigr]^{1/d}.
\end{equation}
Consequently, for every $\rho_*\in[0,1)$, $\rho(M)\leq\rho_*$ forces $\Lambda_1\varrho+\sqrt{2\varepsilon\Lambda_2}\geq(1-\rho_*)^{d}(1+\norm{M})^{1-d}$.
\end{proposition}}

The proof is in Appendix~\ref{app:pf-orbital}.

Both forms say the same thing, and neither dominates the other.  Proposition~\ref{prop:obstruction} is quantitatively sharp but assumes a contraction on a set containing $\Gamma$, which we do not verify.  Proposition~\ref{prop:obstruction-local} assumes only the spectral radius at one point, which is what Algorithm~\ref{alg:floquet} computes, but pays for it with $\Lambda_1$ and $\Lambda_2$, which are governed by the $C^2$ size of the trained field and are not controlled during training.  In the norm in which the contraction holds, \eqref{eq:obstruction} carries no constants: the tighter the contraction, the larger the one-period error on $\Gamma$ must be, and the bound degenerates only as $q\to1$.  Passing to the Euclidean $\varepsilon$ of Assumption~\ref{asm:learned} costs the factor $c_1/c_2$, which we do not estimate; the two quantities the paper would like to be small at once are nonetheless in direct conflict.  Section~\ref{ssec:num-cert} measures both sides on every trained model.  What survives is a guarantee about the \emph{orbit}, and it needs neither Assumption~\ref{asm:learned} nor a section, a return time or a clock comparison.

{\begin{theorem}[orbital guarantee for the deployed architecture]\label{thm:orbital}
Let Assumption~\ref{asm:true} hold and let $\fTh$ be as above.  Assume
\begin{enumerate}
\item[\textup{(B1)}] $\Sth$ has a fixed point $x_\Th^\dagger$ whose monodromy $M:=D\Sth(x_\Th^\dagger)$ satisfies $\rho(M)\leq\rho_*<1$, where $\rho_*$ is a free parameter of this theorem and not the training threshold of \eqref{eq:floquet-loss};
\item[\textup{(B2)}] the closed orbit $\gamma_\Th^\dagger:=\{\PhiTh(s;x_\Th^\dagger):s\in[0,\That]\}$ satisfies $\dist_H(\gamma_\Th^\dagger,\Gamma)\leq\varepsilon_\Gamma$.
\end{enumerate}
Set $\bar\rho:=(1+\rho_*)/2<1$.  Then there exist $r_0>0$ and $C\geq1$, depending only on $(d,\That,B,\rho_*)$, such that for every $x_0$ with $\norm{x_0-x_\Th^\dagger}\leq r_0$ and every $t\geq0$,
\begin{equation}\label{eq:orbital-bound}
  \dist\bigl(\PhiTh(t;x_0),\Gamma\bigr)
  \;\leq\;
  \varepsilon_\Gamma \;+\; C\,\bar\rho^{\,\lfloor t/\That\rfloor}\,\norm{x_0-x_\Th^\dagger} .
\end{equation}
\end{theorem}}

{Both hypotheses are measurable rather than structural: \textup{(B1)} is the spectral radius that Algorithm~\ref{alg:floquet} returns from one variational solve, and \textup{(B2)} is a Hausdorff distance between two sampled curves.  What the implementation returns are floating-point eigenvalues of a numerically integrated monodromy and a distance between finite samples, without residual or quadrature bounds, so the reported values are measurements of the hypotheses and not certificates of them.  The radius $r_0$ is likewise not evaluated for the initial conditions used in Section~\ref{ssec:num-cert}.  The standard route to close this gap is validated integration of the monodromy together with interval bounds on the sampled distances, which would upgrade both measurements to certificates; we do not pursue it here.  Neither involves $\varepsilon$, so Proposition~\ref{prop:obstruction} does not obstruct them, and Section~\ref{ssec:num-cert} reports both.  The bound is uniform in time with no linear term: the clock mismatch of Lemma~\ref{lem:clocks}, which produces the factor $1+t/\That$ in \eqref{eq:floquet-main}, is absent here because both the target and the learned orbit close in exactly $\That$.}

\begin{remark}[what the encoding buys, and what it gives up]\label{rem:orbital-vs-traj}
Theorem~\ref{thm:orbital} bounds the distance to $\Gamma$, not the distance to the target trajectory through the same initial state, and Proposition~\ref{prop:obstruction} shows that the stronger conclusion is incompatible with a contraction of $\Sth$ on a set containing $\Gamma$.  The reason is visible in \eqref{eq:strobo-semigroup}: $\Phi_f(\That;\cdot)$ fixes every point of $\Gamma$, so the target retains the phase of its initial condition forever, whereas a contracting $\Sth$ entrains every nearby state to the single orbit $\gamma_\Th^\dagger$ and forgets that phase.  The periodic encoding buys orbital stability with no data at deployment, and pays for it in phase fidelity.  Section~\ref{ssec:num-cert} exhibits both halves by varying the launch phase.  Both terms of \eqref{eq:orbital-bound} are moreover of the correct order.  In the polar setting of Remark~\ref{rem:tight}, let the learned field be $\dot r=-(r-1-\varepsilon_\Gamma)$, $\dot\theta=2\pi/\That+\kappa\sin\bigl(2\pi t/\That-\theta\bigr)$ with $\kappa>0$: an exactly $\That$-periodic field whose flow locks to the circle of radius $1+\varepsilon_\Gamma$ traversed in phase with the drive.  In the co-rotating variable $\psi:=\theta-2\pi t/\That$ the system reads $\dot r=-(r-1-\varepsilon_\Gamma)$, $\dot\psi=-\kappa\sin\psi$, so the locked point is a fixed point of $\Sth$ with monodromy spectrum $\{e^{-\That},e^{-\kappa\That}\}$.  Hypothesis \textup{(B1)} therefore holds with $\rho_*=\max(e^{-\That},e^{-\kappa\That})$, and \textup{(B2)} holds with Hausdorff distance exactly $\varepsilon_\Gamma$.  Every nearby trajectory converges to the locked circle, so its distance to $\Gamma$ tends to $\varepsilon_\Gamma$: the time-independent floor in \eqref{eq:orbital-bound} is attained in the limit and cannot be removed, while the transient term decays geometrically at the rate the monodromy prescribes.
\end{remark}

\section{Numerical experiments}\label{sec:numerics}

This section tests each theoretical claim on a dedicated experiment, including a width sweep for the budget of Theorem~\ref{thm:linearT} and an autonomous arm for the regime of Theorem~\ref{thm:floquet}.  The cold-started autonomous arm is a negative result: training there produces no certifiable closed orbit.  A warm-started arm, fine-tuned from the converged encoded models, repairs it (Section~\ref{ssec:num-cert}).  Section~\ref{ssec:num-mpc} treats the data-assisted composite bound of Theorem~\ref{thm:linearT} on a forced Duffing equation.  Section~\ref{ssec:num-budget} works on the pendulum: it exhibits the exponential barrier, shows the drift of Proposition~\ref{prop:novel-ic} when the resets are withheld, and measures the growth of the window count in the horizon and the tolerance.  Section~\ref{ssec:num-cert} measures the monodromy spectrum of the deployed architecture, which is hypothesis \textup{(B1)} of Theorem~\ref{thm:orbital} and not a verification of \eqref{eq:cert-cond}, through an ablation over the time encoding and the Floquet loss, and Section~\ref{ssec:num-h2h} compares the two strategies on the van der Pol oscillator.  The implementation, the configuration and archived results of every run, the trained models, and the scripts generating every figure and table are available at \url{https://github.com/DCN-FAU-AvH/SA-NODE-MPC-Floquet.git}.

\subsection{Setup}\label{ssec:num-setup}
We record the shared configuration once, together with the points at which the computations depart from the hypotheses of the theory.  The experiments are proof-of-concept by design: low-dimensional canonical benchmarks on which each theoretical mechanism can be tested in isolation and every reported quantity can be recomputed independently; they are not application-scale validation.  All models are trained in PyTorch with the Adam optimizer.  The MPC experiments of Sections~\ref{ssec:num-mpc}--\ref{ssec:num-budget} use ReLU activation, per-window width $P=128$, and forward Euler integration at the data resolution $\Delta t=0.1$; these choices follow the numerical protocol of \cite[\S 4]{LLLZ24}.  Theorem~\ref{thm:uap} is existential in the parameters and bounds the exact SA-NODE flow rather than its forward Euler discretization, and the Duffing field of \eqref{eq:duffing} is only locally Lipschitz, so the per-window budget is read as a scaling guide rather than a guaranteed count.  The Floquet experiments use $\tanh$ activation, width $64$, and fixed-step RK4, consistent with the $C^2$ hypotheses of Section~\ref{sec:floquet}.  Reference trajectories are computed by an adaptive Runge--Kutta method at integrator tolerances several orders of magnitude below the errors reported.  Monolithic baselines use the width of a single window and the same total number of gradient steps as the composite, so the comparisons are matched in training cost rather than capacity.  Theorem~\ref{thm:uap} and the width scaling of Section~\ref{ssec:barrier} certify a sufficient width for an equal-parameter monolithic model that still carries the factor $e^{2LT}$; they do not assert that the model must incur it.

\subsection{Composite approximation of a forced Duffing equation}
\label{ssec:num-mpc}
The first experiment instantiates the data-assisted bound \eqref{eq:main-mpc} of Theorem~\ref{thm:linearT}: it runs Algorithm~\ref{alg:mpc} on the forced Duffing equation
\begin{equation}\label{eq:duffing}
  \dot x_1 \;=\; x_2, \qquad
  \dot x_2 \;=\; x_1 - x_1^3 + 0.1\,\cos(\pi t),
\end{equation}
a non-autonomous target that activates the time bias of \eqref{eq:sanode}, on $[0,15]$ with a $5\times5$ grid of initial conditions in $[-1.5,1.5]^2$, tolerance $\varepsilon=0.5$, and training horizon $H=10$.  The algorithm selects $2$ windows, of lengths $10.4$ and $4.6$, and realizes the composite error $0.519$, which exceeds the tolerance by $3.8\,\%$: the safeguards bind, so Assumption~\ref{asm:partition}, checked on the training grid, holds with $\max_k\varepsilon_k^{\mathrm{win}}$ in place of $\varepsilon$; the sampled analogue of \eqref{eq:linearT} then holds with that value.  Meanwhile the monolithic SA-NODE baseline, trained with the same number of gradient steps, reaches $1.98$, about $4.0$ times the tolerance.  Figure~\ref{fig:duffing}(a) displays sample trajectories of the composite against the truth, and Figure~\ref{fig:duffing}(b) the error over time.  The data reset holds the composite near the tolerance on every window, while the monolithic error exceeds it from early times onward by a factor of about $4.0$ and oscillates at the scale of the reachable tube.  At each switch time the reported error is the left limit: the composite is reset to the data there, so the sawtooth peaks just before the switch.  On a bounded target the barrier appears as a persistent excess rather than unbounded growth; the growth mechanism itself is exhibited on the pendulum below.

\begin{figure}[t]
\centering
\begin{subfigure}[t]{0.48\textwidth}
  \centering
  \includegraphics[width=\linewidth]{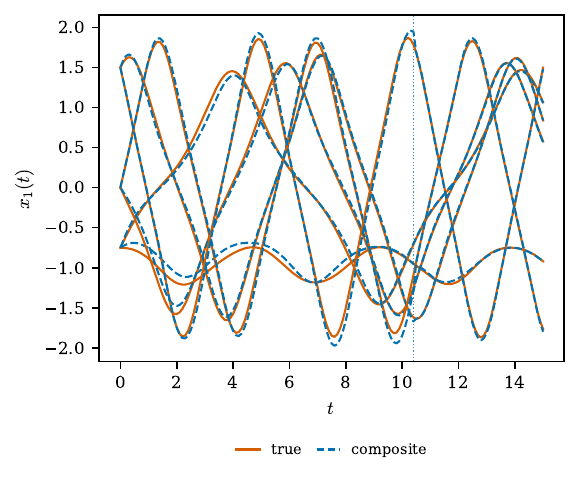}
\caption{sample trajectories}
  \label{sfig:duffing-traj}
\end{subfigure}\hfill
\begin{subfigure}[t]{0.48\textwidth}
  \centering
  \includegraphics[width=\linewidth]{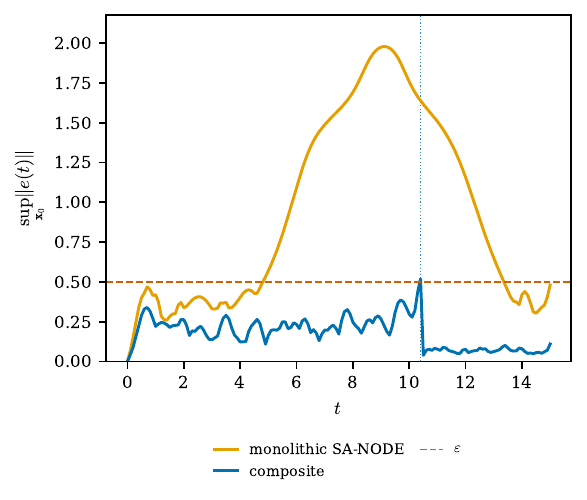}
\caption{error over time}
  \label{sfig:duffing-err}
\end{subfigure}
\caption{Forced Duffing equation \eqref{eq:duffing}, $T=15$, $\varepsilon=0.5$.  (a)~Sample trajectories, true versus composite; the dotted vertical line marks the switch time.  (b)~Error over time, $\sup_{\bx_0}\norm{e(t)}$ on the initial-condition grid; the sawtooth peaks are left limits at the switch times.}
\label{fig:duffing}
\end{figure}

\subsection{Window count, tolerance, and the exponential barrier}
\label{ssec:num-budget}
This experiment has three purposes: to exhibit the exponential error growth of Section~\ref{ssec:barrier} on a monolithic model, to show the drift that Proposition~\ref{prop:novel-ic} predicts when the data resets are withheld, and to measure how the window count of Theorem~\ref{thm:linearT} grows with the horizon and the tolerance.  The benchmark is the pendulum $\dot x_1=x_2$, $\dot x_2=-\sin x_1$, run with the protocol of Section~\ref{ssec:num-mpc} on a $5\times5$ grid of initial conditions in $[-2,2]^2$; the horizon sweep uses training horizon $H=3$, and the tolerance sweep trains each window on the full remaining horizon.  On $\R^2$ eight of these initial conditions carry energy above the separatrix value $2$, so their trajectories rotate and the infinite-time reachable set is unbounded there.  Read on the cylinder the reachable set is compact, and all horizons run here are finite; Assumption~\ref{asm:partition} is in any case verified on the output, so the sweep instantiates the reset mechanism rather than the literal infinite-time tube hypothesis.

Figure~\ref{fig:pend-sweeps}(a) shows the error growth at $T=10$.  The monolithic neural ODE error grows exponentially along the horizon, at the measured rate $e^{0.39\,t}$ (correlation $r=0.88$ over the growth phase), before saturating at the scale of the reachable tube.  The composite is reset at each switch time and stays at the tolerance.  Its excess over $\varepsilon$ (realized error $0.583$, about $16.6\,\%$ above the tolerance) is the safeguard excess anticipated after Assumption~\ref{asm:partition}.

The same figure isolates the value of the reset.  We take the eight trained windows of this run and chain them without data: each window starts from the state the previous one predicted, which is the predicted-IC deployment of Proposition~\ref{prop:novel-ic}.  No retraining is involved, so the two composites differ only in what happens at the switch times.  The error now accumulates across windows and reaches $3.74$, that is $7.5$ times the tolerance, while the data-IC evaluation of the same models stays at $0.58$.  The bound of Proposition~\ref{prop:novel-ic} is honored but far from sharp: the weight-norm Lipschitz bound of the trained ReLU windows is $\bar L=29$, which puts the envelope near $10^{89}$.  The contrast between the two curves is the content of the proposition: the resets, not the trained windows, carry the uniform guarantee.

The partitions produced here satisfy the mesh condition of Assumption~\ref{asm:partition} with a horizon-independent bound: across the sweep the realized mesh is $3.00$ at every horizon, so $\Tmax$ can be taken independent of $T$ as the theory requires.

Figure~\ref{fig:pend-sweeps}(b) and Table~\ref{tab:budget} report the window count against the horizon $T\in\{5,10,15,20\}$ at $\varepsilon=0.5$. The least-squares fit $N\approx1.34\,T-4.50$ has correlation $r=0.995$, consistent with the linear growth allowed by Theorem~\ref{thm:linearT}, and far below the exponential width that \eqref{eq:uap-const} certifies as sufficient for a monolithic model.  Two caveats temper this reading.  The increments of $N$ are 5, 7, and 8 and mildly convex.  The realized error also drifts from $0.58$ to $0.72$ across the sweep as the safeguards bind more often, so four points do not separate linear from mildly super-linear growth.  Every entry exceeds $\varepsilon=0.5$: the tolerance is met per window only up to that excess.  Problem~\ref{prob:NT} records the general question.  The width side of the budget is probed separately: a sweep over widths $16$ to $256$ on a single fixed window, three seeds each, gives a realized error decaying with slope $-0.75$ against the width on a log-log scale ($\abs{r}=0.97$), consistent with, indeed faster than, the $P^{-1/2}$ rate that Theorem~\ref{thm:uap} certifies.

\begin{table}[t]
\centering
\caption{Window count and composite error versus horizon, pendulum, $\varepsilon=0.5$.}
\label{tab:budget}
\begin{tabular}{lcccc}
\toprule
$T$ & $5$ & $10$ & $15$ & $20$ \\
\midrule
$N$ & 3 & 8 & 15 & 23 \\
$\mathrm{err}_{L^\infty}$ & 0.58 & 0.58 & 0.72 & 0.72 \\
\bottomrule
\end{tabular}
\end{table}

Figure~\ref{fig:pend-sweeps}(c) reports the tolerance sweep at $T=10$: the count $N=(4,\,8,\,12,\,15,\,15)$ for $\varepsilon\in\{1.0,0.7,0.5,0.35,0.25\}$, with each window trained on the full remaining horizon and window cap $S=15$.  At $\varepsilon=0.35\text{ and }0.25$ the cap binds, producing the admissibility failure anticipated after Assumption~\ref{asm:partition}.  The crossing fires so early that the minimum window length dominates, and the horizon is left uncovered from $t=9.1$ and $t=9.0$ onward in the two runs, where the composite is not defined and no error can be reported.  Assumption~\ref{asm:partition} fails outright, and Theorem~\ref{thm:linearT} does not apply.  In this regime the remedy is to add width per window rather than more windows.

\begin{figure}[t]
\centering
\begin{subfigure}[t]{0.34\textwidth}
  \centering
  \includegraphics[width=\linewidth]{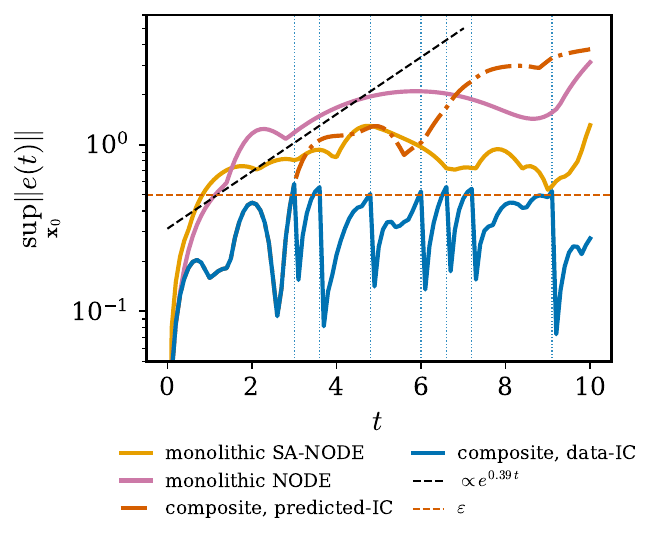}
\caption{error over time, $T=10$}
  \label{sfig:pend-err}
\end{subfigure}\hfill
\begin{subfigure}[t]{0.32\textwidth}
  \centering
  \includegraphics[width=\linewidth]{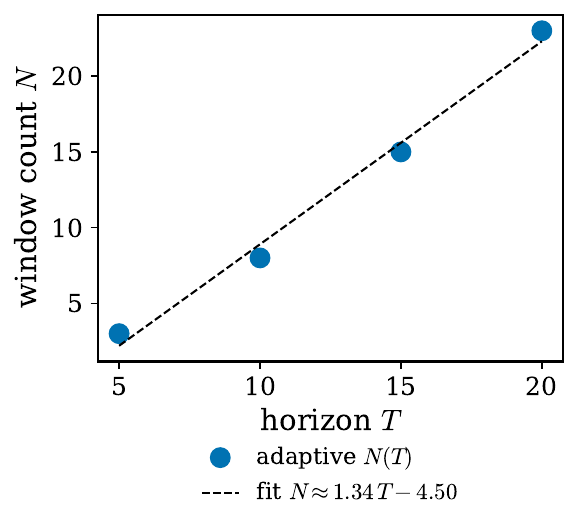}
\caption{window count vs.\ horizon}
  \label{sfig:NvsT}
\end{subfigure}\hfill
\begin{subfigure}[t]{0.32\textwidth}
  \centering
  \includegraphics[width=\linewidth]{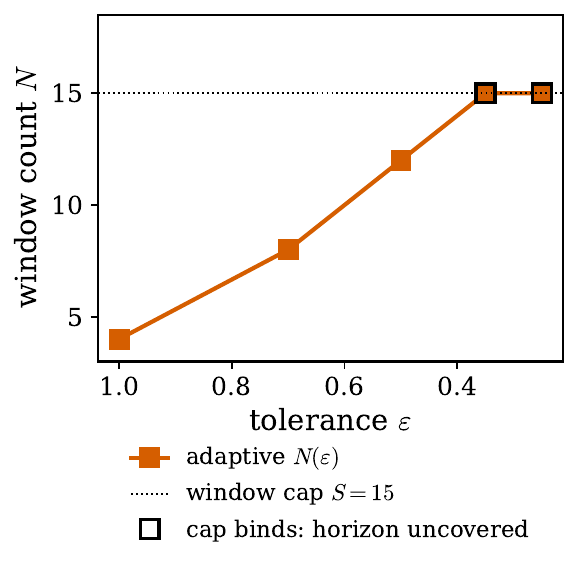}
\caption{window count vs.\ tolerance}
  \label{sfig:NvsEps}
\end{subfigure}
\caption{Pendulum sweeps.  (a)~Error over time at $T=10$, logarithmic scale; the dashed guide is the least-squares fit $e^{0.39\,t}$ to the neural ODE baseline.  (b)~Window count versus horizon at $\varepsilon=0.5$, with least-squares fit.  (c)~Window count versus tolerance at $T=10$; the dotted line is the window cap $S$, which binds at the two smallest tolerances and leaves the horizon uncovered from $t\approx9$: an admissibility failure, not a budget limit.}
\label{fig:pend-sweeps}
\end{figure}

\subsection{Certification: time encoding and Floquet loss}
\label{ssec:num-cert}
{Corollary~\ref{cor:lf-cert} links the trained surrogate to the certificate \eqref{eq:cert-cond} only for an autonomous planar field.  The architecture ablated here is periodic in time, and by Remark~\ref{rem:det} the surrogate then controls $\det M$ rather than $\rho(M)$, so the experiment is scored on the latter.}  It measures which ingredients produce contraction.  The ablation is two-by-two: the periodic encoding \eqref{eq:periodic-encoding} against the raw linear time bias of \eqref{eq:sanode}, and training with $\LF$ against training without it.  A third, autonomous arm removes the time features altogether, which is the regime of Theorem~\ref{thm:floquet} and Corollary~\ref{cor:lf-cert}; it is reported at the end of this subsection.  The two benchmarks have true multipliers an order of magnitude apart: the Stuart--Landau system
\begin{equation}\label{eq:sl}
  \dot r \;=\; r\,(0.1-r^2), \qquad \dot\theta \;=\; 1,
\end{equation}
with $\rho_{\mathrm{true}}=e^{-0.4\pi}\approx 0.285$ and threshold $\rho_*=0.30$, and the van der Pol oscillator
\begin{equation}\label{eq:vdp}
  \dot x_1 \;=\; x_2, \qquad
  \dot x_2 \;=\; \mu\,(1-x_1^2)\,x_2 - x_1, \qquad \mu=0.5,
\end{equation}
with period $\That\approx6.3807$, $\rho_{\mathrm{true}}\approx 0.039$, and the strict threshold $\rho_*=0.05$.  For each system, encoding, and loss condition ($\Ltraj$ only, condition~A, or $\Ltraj+\LF$, condition~B), we train $15$ models, one per seed, and run each of them autonomously for $300$ periods from an initial condition at distance $0.2$ from the orbit.  The linear-encoding models reset their clock at each period, the deployment protocol of the original architecture.  The reset makes the field piecewise-defined in time and in general discontinuous at the period boundary, hence outside the $C^2$-periodic class of Theorem~\ref{thm:orbital}; this arm is an empirical ablation only.

Every model is assessed by the spectral radius $\rho(M)$ of the full one-period stroboscopic monodromy along its locked orbit, obtained from one variational solve; for a $\That$-periodic field this is what governs the linear stability of that orbit (Remark~\ref{rem:det}).  It is neither the return-map radius $\rhoTh$ of Assumption~\ref{asm:cert} nor the surrogate $\tilde\rho_T=\det M$ of \eqref{eq:det-identity}, which is reported separately.

Table~\ref{tab:cert} reports the outcome, Figure~\ref{fig:cert}(a) the per-seed values, and Figure~\ref{fig:cert}(b) the resulting long-horizon errors.  Three readings must be kept apart.

The first is whether the locked orbit is linearly stable, that is whether $\rho(M)<1$, which is hypothesis \textup{(B1)} of Theorem~\ref{thm:orbital}.  It does, in $119$ of the $120$ runs.  The single exception is a linear-encoding van der Pol model trained without $\LF$, with $\rho(M)=1.116$, and it is also the only divergent deployment, reaching error $51$ after $300$ periods.  Linear stability of the locked orbit is therefore the rule rather than the exception.  This measurement supplies hypothesis \textup{(B1)} and no more: it is not the certificate \eqref{eq:cert-cond}, which concerns the averaged field and is out of reach at the measured oscillation $\eta$, and it leaves Assumption~\ref{asm:learned} unverified, since no training loss controls closeness on the whole tube.

The second is the contraction rate, and there the encoding is decisive.  With the periodic encoding the median spectral radius is $0.044$ on Stuart--Landau and $0.101$ on van der Pol; with the raw linear bias it is $0.532$ and $0.695$.  The encoding thus improves the rate by roughly an order of magnitude on the first system and a factor of about five to seven, depending on the loss condition, on the second.  The Floquet loss tightens the rate further on the unprotected architecture, from $0.695$ to $0.387$ on van der Pol, and it removes the one unstable run.

The third is whether a model reaches the threshold $\rho_*$ used in the loss; the per-cell counts are in Table~\ref{tab:cert}.  Every Stuart--Landau model with the periodic encoding meets $\rho_*=0.30$, while on van der Pol few models reach $\rho_*=0.05$: their spectral radii sit at two to three times the target's own multiplier $\rho_{\mathrm{true}}\approx 0.039$, from which the threshold was set.  The learned orbits contract, but by less than the true orbit does.  Scored on $\det M$ instead, every periodic-encoding van der Pol seed would pass while few pass on $\rho(M)$; the Stuart--Landau linear cells show the same gap (Remark~\ref{rem:det}).  We do not adjust $\rho_*$ after the fact; it is a training target inherited from the surrogate.

The orbital reading is the one Theorem~\ref{thm:orbital} governs, and its two hypotheses are measured on all $60$ periodic-encoding models.  Hypothesis \textup{(B2)}, the Hausdorff distance $\varepsilon_\Gamma$ between the locked orbit and the target cycle, has median $7.9\times 10^{-4}$ on Stuart--Landau and $1.5\times 10^{-2}$ on van der Pol; hypothesis \textup{(B1)} is the spectral radius already reported.  The constant of Assumption~\ref{asm:learned}, evaluated over the tube rather than at the base point, is of another order: median $0.88$ and $4.84$, that is $1124$ and $328$ times $\varepsilon_\Gamma$, above the diameters $0.63$ and $4.74$ of the target cycles.  This is not a training failure but the obstruction at work.  Proposition~\ref{prop:obstruction-local} explains the mechanism under exactly the measured hypotheses, $\rho(M)$ at the fixed point and $\varrho\leq\varepsilon_\Gamma$.  For the magnitude, the operator norm of the Jacobian of the stroboscopic map over the $0.2$-tube, sampled on a grid for the two representative models, is at most $0.052$ on Stuart--Landau and $0.186$ on van der Pol; at these contraction factors Proposition~\ref{prop:obstruction} gives the lower bounds $\varepsilon\geq0.30$ and $\varepsilon\geq1.93$, and the measured tube errors lie above both.  The grid values are a finite sampling, not a certificate, so this is a consistency check: on these models one-period closeness on the tube, and with it Theorem~\ref{thm:floquet} and Lemma~\ref{lem:nearaut}, is not observed at any tuning.

The launch phase separates the two readings.  We deploy the same threshold-passing model from the same state at distance $0.2$ from the cycle, moving the launch phase from the training phase to half a period later.  The supremum of the trajectory error over eight periods rises from $0.21$ to $0.63$ on Stuart--Landau, and from $0.26$ to $4.90$ on van der Pol.  The distance to the target cycle stays at the $\varepsilon_\Gamma$ level: $5.5\times 10^{-4}$ against $3.0\times 10^{-4}$, and $1.0\times 10^{-2}$ against $3.3\times 10^{-3}$, at the end of the run.  The trained models track the cycle and not the trajectory, exactly as Remark~\ref{rem:orbital-vs-traj} describes, and the earlier protocol did not see it because it launched only at the training phase.  The transient of \eqref{eq:orbital-bound} is also visible: before saturating at the $\varepsilon_\Gamma$ floor, the distance to the cycle contracts per period by a measured factor $0.036$ on Stuart--Landau and $0.044$ on van der Pol, of the order of the measured $\rho(M)$ and well below $\bar\rho$.

The autonomous arm fails, in the way the encoding was designed to prevent.  In none of the $10$ runs does training produce a certifiable closed orbit: every model relaxes to a stable equilibrium instead of a periodic attractor, and with no nondegenerate closed orbit no periodic-orbit certificate can be defined.  The one-period tube errors range from $0.43$ to $5.1$.  This is consistent with the saturation obstruction \cite{Matzakos26} that motivated the periodic encoding, and it leaves the chain of Corollary~\ref{cor:lf-cert} without a trained positive instance under cold-started autonomous training.

A warm-started arm repairs this failure.  Transferring the weights of the converged periodic-encoding models to the autonomous architecture (the shared matrices verbatim, the constant bias as the phase average of the periodic bias branch) and fine-tuning under the same condition-B loss produces closed learned orbits where cold-started training produced none.  On no seed does the transferred field carry a closed orbit before fine-tuning, so the repair comes from the warm-started optimization, not from the initialization alone.  On van der Pol, four of five seeds yield a closed, transversally stable orbit with $\tilde\rho_T$ between $0.14$ and $0.19$ and transverse multiplier $\rho(D\PTh(x_\Th^*))$ between $0.13$ and $0.18$; the surrogate--multiplier gap of Corollary~\ref{cor:lf-cert}(ii) is $0.002$--$0.011$ against a one-period tube error $\varepsilon$ of $0.61$--$0.78$, smaller by factors of $65$ to $320$.  On Stuart--Landau every seed yields a closed, transversally stable orbit with $\rho(D\PTh)\approx0.66$ and tube error $\approx0.13$; the closure test of the surrogate's relaxation does not pass there, so $\tilde\rho_T$ is not evaluated.  An independent recomputation with an adaptive integrator reproduces every reported multiplier to within $10^{-3}$ and closes the learned orbits to residuals below $10^{-10}$ on van der Pol and $3\times10^{-5}$ on Stuart--Landau: the orbits are genuine, and the failed closure test reflects the strictness of the check, not the absence of a cycle.  At deployment the stroboscopic error grows at $0.09$--$0.16$ per period before the phase wraps, consistent with the envelope $C\varepsilon(1+k)$ of Theorem~\ref{thm:floquet}(iii).  The arm is a measured instantiation of the certification mechanism of Corollary~\ref{cor:lf-cert}, not a certified instance: the smallness conditions and constants are not evaluated, and the contraction is read from the measurement rather than from the training target $\rho_*$.

\begin{table}[t]
\caption{Warm-started autonomous arm, per seed: divergence surrogate $\tilde\rho_T$ on the relaxed orbit, transverse multiplier $\rho(D\PTh)$, the absolute gap $\abs{\tilde\rho_T-\rho(D\PTh)}$ (the quantity of Corollary~\ref{cor:lf-cert}(ii), computed from unrounded values), one-period tube error $\varepsilon$, and fitted stroboscopic growth per period.  Van der Pol seed~2 relaxes to an equilibrium and is excluded; on Stuart--Landau the relaxation closure test does not pass, so $\tilde\rho_T$ is not evaluated there.}
\label{tab:warmstart}
\centering
\begin{tabular}{llccccc}
\toprule
System & seed & $\tilde\rho_T$ & $\rho(D\PTh)$ & gap & $\varepsilon$ & slope \\
\midrule
Stuart--Landau & 0 & --- & 0.660 & --- & 0.128 & 0.108 \\
               & 1 & --- & 0.654 & --- & 0.136 & 0.107 \\
               & 2 & --- & 0.656 & --- & 0.132 & 0.107 \\
               & 3 & --- & 0.660 & --- & 0.128 & 0.107 \\
               & 4 & --- & 0.654 & --- & 0.136 & 0.107 \\
\midrule
van der Pol    & 0 & 0.144 & 0.134 & 0.010 & 0.782 & 0.103 \\
               & 1 & 0.151 & 0.143 & 0.008 & 0.606 & 0.089 \\
               & 3 & 0.168 & 0.170 & 0.002 & 0.623 & 0.157 \\
               & 4 & 0.193 & 0.183 & 0.011 & 0.692 & 0.108 \\
\bottomrule
\end{tabular}
\end{table}

Two further measurements complete the picture.  The first is the oscillation quantified by Lemma~\ref{lem:nearaut}.  On the tube of radius $0.2$ the representative condition-B models have $C^1$ oscillation $0.46$ (Stuart--Landau) and $2.75$ (van der Pol).  The lemma requires $\varepsilon+C_B\eta\leq\delta_0/4$; already for $B=1$ the constant $C_B$ is of order $10^{9}$, so the admissible $\eta$ is below $10^{-10}$ and the measured values miss the regime by ten orders of magnitude.  Deployment succeeds because the trajectories stay locked to the periodic drive, not because the field is near-autonomous.  An oscillation penalty on the loss reduces $\eta$ by a factor of five to seven at the cost of a coarser one-period fit, and cannot close a gap of ten orders.

The second is the reference computation behind Remark~\ref{rem:det}, carried out on the representative models rather than on the whole ablation.  Without the oscillation penalty, the Stuart--Landau model reports $\tilde\rho_T=3.4\times 10^{-4}$ while its multipliers are $(0.048,0.007)$; the van der Pol spectral radius $0.118$ already exceeds $\rho_*=0.05$ although its determinant $0.014$ passes.  Under the penalty the van der Pol determinant $0.155$ stays below one while one multiplier lies above it, at $1.256$.  That is the failure Remark~\ref{rem:det} warns of, and it is the reason the ablation above is scored on $\rho(M)$ rather than on $\tilde\rho_T$.

\begin{table}[t]
\centering
\caption{Spectral radius $\rho(M)=\max_i\abs{\lambda_i}$ of the one-period stroboscopic monodromy of the encoded flow (Remark~\ref{rem:det}), over $15$ seeds per cell of the encoding-by-loss ablation.  The column ${<}1$ counts the models whose locked orbit is linearly stable; the column ${\leq}\rho_*$ counts those reaching the threshold used in the Floquet loss.  Medians are over the seeds of the cell.}
\label{tab:cert}
\begin{tabular}{llccccccc}
\toprule
System & encoding & $\rho_*$ & \multicolumn{3}{c}{without $\LF$} & \multicolumn{3}{c}{with $\LF$} \\
\cmidrule(lr){4-6}\cmidrule(lr){7-9}
 & & & med.\ $\rho(M)$ & ${<}1$ & ${\leq}\rho_*$ & med.\ $\rho(M)$ & ${<}1$ & ${\leq}\rho_*$ \\
\midrule
Stuart--Landau & periodic & 0.30 & 0.044 & 15/15 & 15/15 & 0.040 & 15/15 & 15/15 \\
 & linear &  & 0.532 & 15/15 & 0/15 & 0.483 & 15/15 & 0/15 \\
Van der Pol & periodic & 0.05 & 0.101 & 15/15 & 1/15 & 0.080 & 15/15 & 4/15 \\
 & linear &  & 0.695 & 14/15 & 0/15 & 0.387 & 15/15 & 0/15 \\
\bottomrule
\end{tabular}
\end{table}

\begin{figure}[t]
\centering
\begin{subfigure}[t]{0.48\textwidth}
  \centering
  \includegraphics[width=\linewidth]{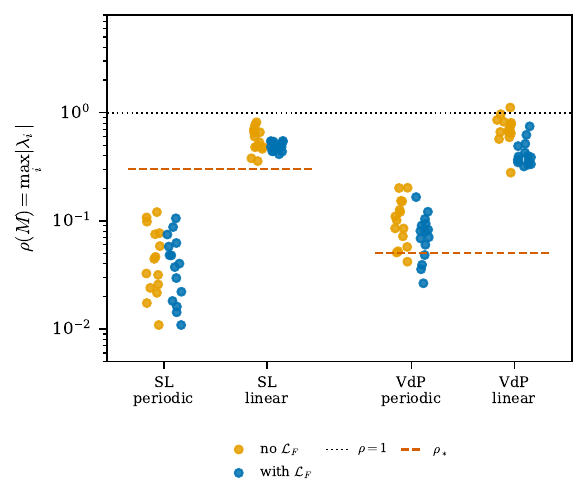}
\caption{per-seed values}
  \label{sfig:cert-scatter}
\end{subfigure}\hfill
\begin{subfigure}[t]{0.48\textwidth}
  \centering
  \includegraphics[width=\linewidth]{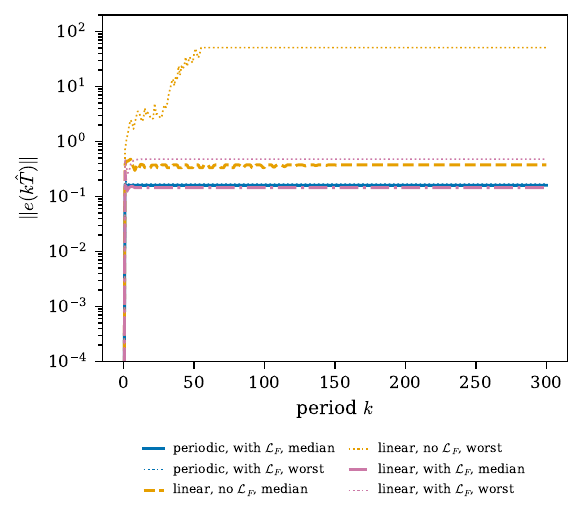}
\caption{stroboscopic error, van der Pol}
  \label{sfig:cert-strobe}
\end{subfigure}
\caption{Certification ablation.  (a)~Per-seed monodromy spectral radius $\rho(M)$, by system, encoding, and loss; dashed lines are the thresholds $\rho_*$, the dotted line is $\rho=1$.  (b)~Stroboscopic error across seeds on the van der Pol oscillator: the periodic encoding stays bounded, the linear-encoding worst case diverges.}
\label{fig:cert}
\end{figure}

\subsection{Comparison of the two strategies on the van der Pol
oscillator}
\label{ssec:num-h2h}
The final experiment places both strategies on the van der Pol oscillator \eqref{eq:vdp} over eight periods under one protocol.  Their guarantees are not directly comparable (data-assisted against autonomous); the comparison is of operating characteristics, not of sample efficiency.  Each method receives one training trajectory, but not the same one.  MPC trains on the eight-period trajectory from $x_0$, at distance $0.2$ from the cycle; Floquet trains on a single period from the base point $p$ on the cycle.  MPC uses $\tanh$ activation and width $64$ (Algorithm~\ref{alg:mpc}, $H=5$, $\varepsilon=0.5$) and is evaluated in data-IC mode; Floquet uses condition~B and runs autonomously.

Table~\ref{tab:h2h} and Figure~\ref{fig:h2h} summarize the comparison.  MPC selects $9$ windows and reaches $\sup_t\norm{e(t)}=0.537$, just above the tolerance $\varepsilon=0.5$, the excess being the safeguard excess of Section~\ref{ssec:num-mpc}.  This run uses $\tanh$ rather than the ReLU that Assumption~\ref{asm:sobolev} fixes, so only the error half of Theorem~\ref{thm:linearT}, verified through Assumption~\ref{asm:partition} on the output, is claimed for it.  The sawtooth of Figure~\ref{fig:h2h}(a) reflects the external reset: the error restarts near zero at each switch time and ends the run at $8.8\times 10^{-3}$, the final window's own error.  The locked orbit of the Floquet model is linearly stable, with monodromy spectral radius $\rho(M)=0.106$, and the model runs open loop on the periodic clock, with no state observations.  As in Section~\ref{ssec:num-cert} it does not reach the threshold $\rho_*=0.05$ inherited from the surrogate.  Its two multipliers form a complex conjugate pair of equal modulus, so the surrogate returns $\tilde\rho_T=\det M=\rho(M)^2=0.011$.  This is the gap of Remark~\ref{rem:det} in its plainest form: the surrogate is the square of the quantity that governs stability, and reading $0.011$ as a contraction rate would understate it by an order of magnitude.  Its error peaks at $0.269$ during the initial transient and settles to $0.169$ at the period marks (Figure~\ref{fig:h2h}(b)).  The deployed encoded model is governed by the uniform bound \eqref{eq:orbital-bound}, and the internal reset corrects toward the orbit rather than toward the reference trajectory. Both trajectories shadow the limit cycle (Figure~\ref{fig:h2h}(c)).  For horizons of $K$ periods the MPC strategy trains one network per window, so its budget grows with the realized window count $N$, about one window per period in the comparison experiment, while the Floquet cost is that of a single training run; the data-versus-autonomy trade-off is discussed in Section~\ref{ssec:mpc-scope}.

\begin{table}[t]
\centering
\caption{Comparison on the van der Pol oscillator ($\mu=0.5$, eight periods, width $64$, $\tanh$).}
\label{tab:h2h}
\begin{tabular}{lcc}
\toprule
 & MPC--SA-NODE & Floquet--SA-NODE \\
\midrule
models trained & 9 & 1 \\
evaluation & data-IC resets & observation-free, time-periodic \\
$\sup_t\|e(t)\|$ & 0.537 & 0.269 \\
$\|e(8\That)\|$ & $8.8\times 10^{-3}$ & 0.169 \\
$\rho(M)$ & -- & 0.106 \\
$\tilde\rho_T=\det M$ & -- & 0.011 \\
\bottomrule
\end{tabular}
\end{table}

\begin{figure}[t]
\centering
\begin{subfigure}[t]{0.32\textwidth}
  \centering
  \includegraphics[width=\linewidth]{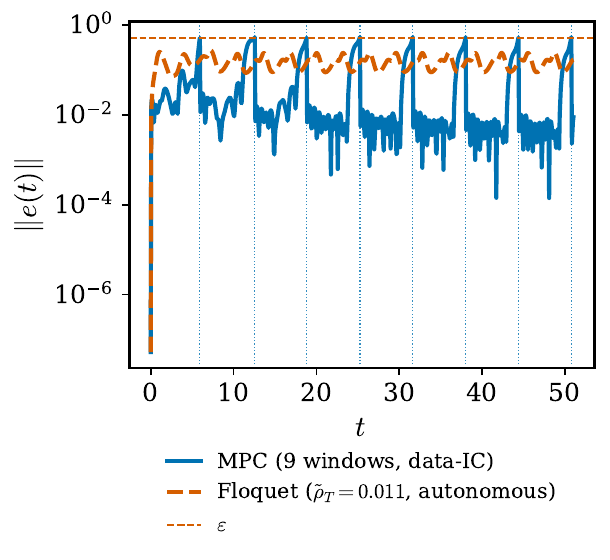}
\caption{continuous error}
  \label{sfig:h2h-err}
\end{subfigure}\hfill
\begin{subfigure}[t]{0.32\textwidth}
  \centering
  \includegraphics[width=\linewidth]{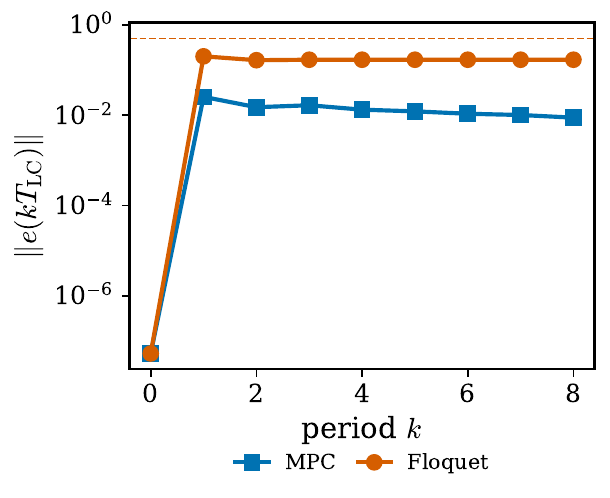}
\caption{stroboscopic error}
  \label{sfig:h2h-strobe}
\end{subfigure}\hfill
\begin{subfigure}[t]{0.32\textwidth}
  \centering
  \includegraphics[width=\linewidth]{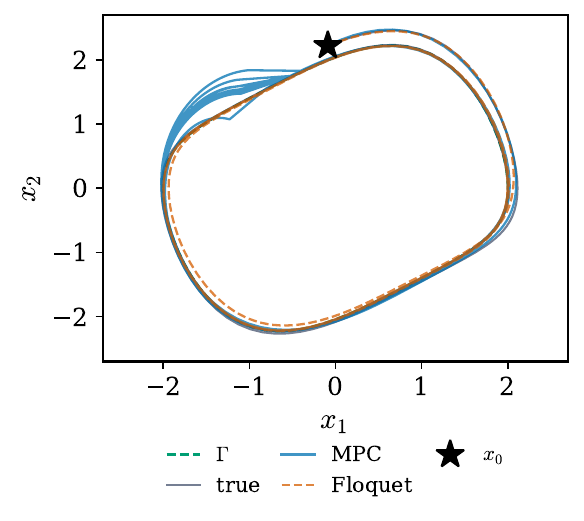}
\caption{phase portrait}
  \label{sfig:h2h-phase}
\end{subfigure}
\caption{Comparison on the van der Pol oscillator.  (a)~Continuous error: the MPC sawtooth against the Floquet transient followed by saturation.  (b)~Stroboscopic error at the period marks.  (c)~Phase portrait: both learned trajectories shadow the limit cycle $\Gamma$.}
\label{fig:h2h}
\end{figure}

\FloatBarrier
\section{Conclusions and perspectives}\label{sec:conclusions}

We first summarize what has been established and what its limits are, and then state three open problems.

\subsection{Conclusions}\label{ssec:conclusions}
This paper studied how the certified error bound of a learned flow deteriorates with the horizon $T$ and how that deterioration can be stopped.  The barrier is a property of the explicit constant \eqref{eq:uap-const} for a single network trained once on $[0,T]$: Gr\"onwall amplification of the field error and growth of the reachable tube, both attained on the linear system $\dot\bx=L\bx$ (Section~\ref{ssec:barrier}).

The model predictive strategy resets the state with data.  Provided training realizes the tolerance $\varepsilon$ on every window, the run covers the horizon, and the reachable tube is bounded with uniformly bounded data (Assumption~\ref{asm:reach}), the composite keeps the error at $\varepsilon$ uniformly in time, up to the quantified safeguard excess, with a parameter budget proportional to the window count (Theorem~\ref{thm:linearT}).  The reset confines the Gr\"onwall amplification to a single window but does not remove the growth of the tube.  Run from its own predictions instead of the data, the same composite obeys only the exponential bound of Proposition~\ref{prop:novel-ic}: the resets, not the trained windows, carry the uniform guarantee.

The Floquet strategy replaces data by stability.  For an autonomous learned field whose return map is certified to contract (Assumption~\ref{asm:cert}), one-period accuracy $\varepsilon$ on a tube around the cycle propagates to the global bound $C\varepsilon(1+t/\That)$: transverse errors are contracted at every return, and only the phase drift accumulates (Theorem~\ref{thm:floquet}, Remark~\ref{rem:tight}).  For the time-periodic architecture we deploy, the scalar certificate degenerates to the determinant of the monodromy (Remark~\ref{rem:det}), and strong stroboscopic contraction even forces a lower bound on the one-period error (Propositions~\ref{prop:obstruction} and~\ref{prop:obstruction-local}).  What survives is the orbital guarantee of Theorem~\ref{thm:orbital}, uniform in time, whose two hypotheses are measured on the trained model: the periodic encoding buys orbital stability and gives up phase fidelity.

The experiments confirm this picture.  The composite meets its tolerance up to the safeguard excess while a monolithic baseline of matched training cost exceeds it about fourfold, the window count grows linearly with the horizon, and withholding the resets produces the predicted drift.  In the certification ablation the periodic encoding is the dominant factor setting the contraction rate, and both hypotheses of Theorem~\ref{thm:orbital} are measured on every periodic-encoding model.  Cold-started autonomous training yields no certifiable orbit; warm-starting from the encoded models repairs it, with a surrogate--multiplier gap far below the measured $\varepsilon$ (Corollary~\ref{cor:lf-cert}(ii)).  Long-horizon accuracy is thus governed by the mechanism that removes accumulated error, not by network size: the width certified by \eqref{eq:uap-const} grows double exponentially in $T$, while the composite adds fixed-width networks one window at a time and the Floquet model runs at a single fixed width.

\subsection{Perspectives}\label{ssec:perspectives}
Three problems mark the most important continuations: one on the barrier itself, one on the window count of the model predictive strategy, and one on certification in higher dimension.

\begin{problem}[a lower bound for the barrier]\label{prob:lower}
Exhibit a Lipschitz constant $L$, a compact $\Ksc$, and a family of targets $\Fsc$, $L$-Lipschitz in $\bx$ uniformly in $t$ and with a fixed $C^2$ bound on the reachable set, such that every SA-NODE \eqref{eq:sanode} of width $P$ whose flow satisfies
\[
  \sup_{\bx_0\in\Ksc,\,t\in[0,T]}\norm{\widehat\Phi_{\bTheta}(t;\bx_0)-\Phi(t;\bx_0)}\;\leq\;\varepsilon
\]
obeys $P\geq\phi(T,\varepsilon)$ with $\phi$ super-linear in $T$ at fixed $\varepsilon$; or prove that no such family exists.  By Lemma~\ref{lem:flow-to-field} it suffices, within the subclass whose fields obey a fixed $C^2$ bound $B$, to bound below the width needed to approximate $\Fsc$ within $2\sqrt{\varepsilon e^{B}C_2}$ uniformly on the spacetime reachable set by a shallow network in $(\bx,t)$.  The restriction to a fixed $B$ is not cosmetic: for unrestricted parameters $B$ may grow with the width, and the reduction then says nothing.  A super-linear answer would make the barrier intrinsic and a decomposition of the horizon necessary rather than merely convenient; a linear answer would make the composite of Theorem~\ref{thm:linearT} optimal in $T$ up to constants and confine the exponential growth to the certificate.
\end{problem}

\begin{problem}[window count]\label{prob:NT}
Prove, or refute, that the partition generated by the unsafeguarded crossing rule \eqref{eq:eps-crossing}, with $s_{\min}=0$ and $S=\infty$, satisfies $N\leq cT+C$ at fixed $\varepsilon$ for Lipschitz targets, with $c$ controlled by the local Lyapunov exponents along the trajectory rather than by the sampling step.  The problem is posed with a fixed bound on the window weights: without one, arbitrarily poor windows cross arbitrarily fast, and already for $f=0$ infinitely many windows fit in a finite horizon.  The safeguarded algorithm gives $N=\Theta(T)$ trivially, with a constant set by the data resolution; Section~\ref{ssec:num-budget} supports the linear law on the pendulum.
\end{problem}

\begin{problem}[higher dimension]\label{prob:highdim}
The Liouville--Abel identity certifies only in the autonomous plane.  In state dimension $d>2$ the transverse monodromy $M_\Th^\perp$ of the learned orbit, the linearization of its return map, is a matrix, and already for $d=2$ a time-periodic field escapes the identity, since the divergence integral then returns the determinant rather than the spectral radius.  Construct differentiable surrogates for the spectral radius $\rho(M_\Th^\perp)$ that are tight enough to train against, beyond the Bauer--Fike route of Section~\ref{ssec:cert}.
\end{problem}

\FloatBarrier
\appendix

\section{\texorpdfstring{Proofs}{Proofs}}\label{app:proofs}

This appendix collects all proofs, one part per result.  Appendix~\ref{app:pf-flowfield} proves Lemma~\ref{lem:flow-to-field}; Appendix~\ref{app:pf-mpc} proves the results of Section~\ref{sec:mpc}; Appendices~\ref{app:pf-lemmas}--\ref{app:pf-orbital} prove the results of Section~\ref{sec:floquet}.

\subsection{\texorpdfstring{Proof of Lemma~\ref{lem:flow-to-field}}{Proof of the flow-to-field lemma}}\label{app:pf-flowfield}

This part proves the reduction of the width lower bound from flows to fields.

\begin{proof}[Proof of Lemma~\ref{lem:flow-to-field}]
Write $\bx=\Phi(t_0;\bx_0)$ with $\bx_0\in\Ksc$, and set
\[
  g(h)\;:=\;\Psi_{\bTheta}(t_0+h;t_0,\bx)-\Psi(t_0+h;t_0,\bx),
  \qquad h\in[0,1],
\]
so that $g(0)=0$ and $\dot g(0)=f_{\bTheta}(\bx,t_0)-\Fsc(t_0,\bx)$.  Differentiating once more along the two flows gives $\ddot g=D_{\bx}f_{\bTheta}f_{\bTheta}+\partial_tf_{\bTheta}-D_{\bx}\Fsc\,\Fsc-\partial_t\Fsc$, whence $\norm{\ddot g}\leq2(B^2+B)=C_2$.

For the size of $g$, put $\hat{\mathbf{y}}:=\widehat\Phi_{\bTheta}(t_0;\bx_0)$, so that $\norm{\bx-\hat{\mathbf{y}}}\leq\varepsilon$ by hypothesis.  The semigroup property gives $\Psi(t_0+h;t_0,\bx)=\Phi(t_0+h;\bx_0)$ and $\Psi_{\bTheta}(t_0+h;t_0,\hat{\mathbf{y}})=\widehat\Phi_{\bTheta}(t_0+h;\bx_0)$.  Hence
\begin{align*}
  \norm{g(h)}
  &\;\leq\;\norm{\Psi_{\bTheta}(t_0+h;t_0,\bx)-\Psi_{\bTheta}(t_0+h;t_0,\hat{\mathbf{y}})}
   \;+\;\norm{\widehat\Phi_{\bTheta}(t_0+h;\bx_0)-\Phi(t_0+h;\bx_0)}\\
  &\;\leq\;\varepsilon e^{Bh}+\varepsilon\;\leq\;2\varepsilon e^{B},
\end{align*}
the first term by Gr\"onwall for the single field $f_{\bTheta}$, whose Lipschitz constant in $\bx$ is at most $B$, and the second by hypothesis.  Gr\"onwall is legitimate here by a continuity argument.  The trajectory $\Psi(\cdot;t_0,\bx)$ lies on $\Ksc_T$, and $\Psi_{\bTheta}(\cdot;t_0,\hat{\mathbf{y}})=\widehat\Phi_{\bTheta}(\cdot;\bx_0)$ lies within $\varepsilon$ of it by hypothesis; only $\Psi_{\bTheta}(\cdot;t_0,\bx)$ needs the argument.  On the maximal interval where it stays in the $r$-neighborhood of $\Ksc_T$, the display bounds its distance to $\Ksc_T$ by $(1+e^{B})\varepsilon\leq r/2<r$, so it does not reach the boundary before $h=1$ and the interval is all of $[0,1]$.

Taylor expansion at $h=0$ with the integral remainder, which is the form valid for a vector-valued $g$, gives $a\dot g(0)=g(a)-\int_0^a(a-s)\ddot g(s)\,\mathrm{d}s$, hence $\norm{\dot g(0)}\leq 2\varepsilon e^{B}/a+aC_2/2$ for every $a\in(0,1]$.  The right side is minimized at $a_*=2\sqrt{\varepsilon e^{B}/C_2}$, which lies in $(0,1]$ precisely under the standing bound $\varepsilon\leq C_2e^{-B}/4$, and the minimum value is $2\sqrt{\varepsilon e^{B}C_2}$.
\end{proof}

\subsection{\texorpdfstring{Proofs of Theorem~\ref{thm:linearT} and Proposition~\ref{prop:novel-ic}}{Proofs of the model predictive bounds}}\label{app:pf-mpc}

This part proves the two results of Section~\ref{sec:mpc}: the data-assisted composite bound and the predicted-IC bound.

\begin{proof}[Proof of Theorem~\ref{thm:linearT}]
The two conclusions have different status, and the proof makes it explicit.  Two things are assumed and not proved: that the optimizer realizes the tolerance on each window, and that Algorithm~\ref{alg:mpc} terminates having covered the horizon; they are the content of Assumption~\ref{asm:partition}, verified a posteriori on each run and reported in Section~\ref{sec:numerics}.  The theorem itself assumes the continuous supremum \eqref{eq:per-window}.  A further discretization gap attaches only to the experimental verification: the suprema are checked on finite grids of initial conditions and times (the standing gap of Section~\ref{ssec:num-setup}), so what the runs verify is the sampled analogue of Assumption~\ref{asm:partition}, not the continuous supremum the theorem assumes.  Estimate \eqref{eq:linearT} unpacks Assumption~\ref{asm:partition}: the content of the data reset is that the global error is a maximum over windows and not a sum, so once each window meets the tolerance from the true state, so does the composite.  The substance is \eqref{eq:budget}, which uses no part of Assumption~\ref{asm:partition} beyond the mesh, and which is where Assumption~\ref{asm:reach} does its work.  On $I_k$ the composite is $\widehat\Phi_{\bTheta_k}$ evaluated from the true state, so no upstream error is amplified and \eqref{eq:linearT} is the maximum over $k$ of the per-window bounds \eqref{eq:per-window}.  For \eqref{eq:budget}, Theorem~\ref{thm:uap} on window $k$ has initial-condition set $\Phi(\tau_k;\Ksc)\subset\Kinf$ and horizon at most $\Tmax$, and its constant obeys
\[
  C_{\Tmax,\Phi(\tau_k;\Ksc),\Fsc(\cdot,\cdot+\tau_k)}
  \;\leq\;
  C_{\Tmax,\Kinf,\Fsc},
\]
because window $k$ is the term $\tau=\tau_k$ of the supremum \eqref{eq:window-const} and Assumption~\ref{asm:reach} makes the right side finite and independent of $k$ and $T$.  Hence $P_k \lesssim C^2_{\Tmax,\Kinf,\Fsc}\varepsilon^{-2}$ on every window, with a constant independent of $k$ and $T$; summing over $k$ concludes.
\end{proof}

\begin{proof}[Proof of Proposition~\ref{prop:novel-ic}]
We derive a one-step recursion for the error accumulated at the switch times and unroll it.  Let $E_k$ denote the accumulated error at the switch time $\tau_k$ and $\ell_k:=\tau_{k+1}-\tau_k$.  On $I_k$ the upstream error $E_k$ enters the learned flow and is amplified by at most $e^{\bar L\ell_k}$; the window adds its own training error at most $\varepsilon$, by \eqref{eq:per-window} of Assumption~\ref{asm:partition}.  Hence
\[
  E_{k+1}\;\leq\; e^{\bar L\ell_k}E_k+\varepsilon,
  \qquad E_0=0 .
\]
Unrolling the recursion gives
\[
  E_N\;\leq\;\varepsilon\sum_{k=1}^{N}e^{\bar L(\tau_N-\tau_k)}
  \;=\;\varepsilon\sum_{k=1}^{N}e^{\bar L(T-\tau_k)},
\]
where the $k=N$ term is the last window's own $\varepsilon$; this is the first bound in \eqref{eq:novel-ic}, and between switch times the same amplification argument applies.  For the second bound, with $\bar L>0$, we use $T-\tau_k\leq(N-k)\Tmax$ termwise, and the geometric sum $\sum_{m=0}^{N-1}e^{\bar Lm\Tmax}$ evaluates to the stated quotient.
\end{proof}

\subsection{\texorpdfstring{Supporting lemmas for the Floquet strategy}{Supporting lemmas for the Floquet strategy}}\label{app:pf-lemmas}

This part and the four that follow prove the results of Section~\ref{sec:floquet}.  Here we prove Lemma~\ref{lem:delta-eps}, the six lemmas preceding Theorem~\ref{thm:floquet}, an auxiliary hitting estimate (Lemma~\ref{lem:hitting}), and the scalar multiplier identity (Lemma~\ref{lem:liouville}).  Throughout Appendices~\ref{app:pf-lemmas}--\ref{app:pf-cert}, Assumptions~\ref{asm:true}--\ref{asm:cert} are in force with $\eta=0$ in Assumption~\ref{asm:aut}, and we write $\fTh(x)$; the proof of Corollary~\ref{cor:lf-cert} uses only Assumptions~\ref{asm:true}--\ref{asm:learned}, as noted there.  Lemma~\ref{lem:nearaut}, proved in Appendix~\ref{app:pf-nearaut}, removes the $\eta=0$ restriction, and Appendix~\ref{app:pf-orbital} stands outside this convention altogether.  All constants depend only on $(f,\ghat,\delta_0,B,\rho_*,\rhoT(f))$.  Table~\ref{tab:radii} records the dependency chain of the radii and smallness thresholds.

\begin{table}[t]
\caption{Dependency chain of the radii and smallness thresholds of Section~\ref{sec:floquet}.  Each row may shrink quantities from earlier rows; no entry depends on $\Th$, $\varepsilon$, $t$, or the initial condition.  The constraint on $\varepsilon_3$ anticipates the certificate-free constants $\delta_1$ and $C_U$ of Lemma~\ref{lem:winding}(i).}
\label{tab:radii}
\centering
\resizebox{\linewidth}{!}{%
\begin{tabular}{llll}
\toprule
Quantity & Introduced in & Constraint & Depends on \\
\midrule
$r_1$ & Lemma~\ref{lem:confine} & $(0,\delta_0/4]$ & $(f,\ghat,\delta_0)$ \\
$r_2,\ \varepsilon_2$ & Lemma~\ref{lem:return} & $r_2\in(0,r_1]$ & $(f,\ghat,\delta_0,B)$ \\
$r_3,\ r_3',\ \varepsilon_3$ & Lemma~\ref{lem:persist} & $r_3,r_3'\in(0,r_2]$, $\varepsilon_3\leq\varepsilon_2$, $\sqrt{2C_2\varepsilon_3}\leq\min(\delta_1,r_3/(2C_U))$ & above; $\rho_*$ via $r_3$, $\rhoT(f)$ via $r_3'$ \\
$r_5',\ \delta_1$ & Lemma~\ref{lem:winding}(i) & $r_5'\in(0,\delta_0/4]$ & $(f,\ghat,\delta_0,B)$, certificate-free \\
$r_5$ & Lemma~\ref{lem:winding}(ii) & $\min(r_5',\,r_3/(2C_U))$ & $\rho_*$ through $r_3$ only \\
$r_4,\ \varepsilon_4$ & Lemma~\ref{lem:track} & $r_4\in(0,\min(r_3,r_3',r_5)]$, $\varepsilon_4\leq\varepsilon_3$ & all of the above \\
$r_6,\ \varepsilon_5$ & Lemma~\ref{lem:hitting} & $r_6\in(0,r_5/2]$, $\varepsilon_5\leq\varepsilon_4$ & all of the above \\
$\varepsilon_0,\ r_0$ & Theorem~\ref{thm:floquet} & $r_0:=r_6$ ($N_{r_0}\subseteq N_{2r_1}$), $\varepsilon_0:=\min\{\varepsilon_5,r_0/K_3,v/2\}$ & $(f,\ghat,\delta_0,B,\rho_*,\rhoT(f))$ \\
\bottomrule
\end{tabular}}
\end{table}

\begin{proof}[Proof of Lemma~\ref{lem:delta-eps}]
The bound is a second-order Taylor expansion of the flow difference, optimized in the expansion time.  Since $\varepsilon\leq\delta_0/4$, Lemma~\ref{lem:confine} keeps both trajectories in $N_{\delta_0}$, where the bounds $B$ and $L$ apply.  Fix $x\in\Gamma$ and set $g(t):=\PhiTh(t;x)-\Phi_f(t;x)$, so that $g(0)=0$ and $\dot g(0)=\fTh(x,0)-f(x)$.  Differentiating once more along the two flows,
\[
  \ddot g \;=\; D_x\fTh\,\fTh+\partial_t\fTh-Df\,f,
  \qquad\text{so}\qquad
  \norm{\ddot g}\leq C_2
\]
by Assumption~\ref{asm:learned} and the $C^1$ bound on $f$.  Taylor expansion at $0$ gives, for every $h\in(0,\That+2]$,
\[
  \norm{\dot g(0)}\,h
  \;\leq\;\norm{g(h)}+\tfrac{1}{2}C_2h^2
  \;\leq\;\varepsilon+\tfrac{1}{2}C_2h^2 .
\]
Choosing $h=\sqrt{2\varepsilon/C_2}$, which is admissible under the stated bound on $\varepsilon$, yields $\norm{\dot g(0)}\leq\sqrt{2C_2\varepsilon}$.  This bounds the mismatch at phase $0$.  For $\eta>0$ the supremum defining $\delta$ also runs over $t\in\R$, and \eqref{eq:eta-def} gives $\norm{\fTh(x,t)-\fTh(x,0)}\leq 2\eta$ for every $t$; adding this to the phase-$0$ bound and taking the supremum over $x\in\Gamma$ gives the claim.
\end{proof}

\begin{proof}[Proof of Lemma~\ref{lem:confine}]
Both claims come from the orbital stability of $\Gamma$ and the definition of $\varepsilon$.  By orbital asymptotic stability and compactness there is $r_1\in(0,\delta_0/4]$ with $\Phi_f(t;x)\in N_{\delta_0/2}$ for all $x\in N_{2r_1}$, $t\in[0,\That+2]$.  Since $\norm{\PhiTh(t;x)-\Phi_f(t;x)}\leq\varepsilon\leq\delta_0/4$ (Assumption~\ref{asm:learned}), also $\PhiTh(t;x)\in N_{3\delta_0/4}\subset N_{\delta_0}$.  The flow closeness is the definition of $\varepsilon$.  The length $\That+2$ accommodates the time shifts of size at most $1+2\varepsilon/v$ arising in the proof of Theorem~\ref{thm:floquet}(iii).  There a time argument in $[-2,\That+3]$ occurs, outside the interval on which $\varepsilon$ is defined.  No estimate is needed outside it.  The point in question lies on the already constructed forward trajectory, and re-indexing it at as many preceding or following crossings as needed exhibits it as $\PhiTh(t'';z)$ with $z\in\Sigma_{r_2}\subset N_{2r_1}$ and $t''\in[0,\That+2]$: each step changes the time argument by a return time in $[\That/2,\That+1]$ (Lemma~\ref{lem:return}), less than the length of the target interval, so some index lands in it.  Only the confinement claim is used there; the closeness bound is invoked only for $t\in[0,\That+1]$.
\end{proof}

\begin{proof}[Proof of Lemma~\ref{lem:return}]
The return maps come from the implicit function theorem, applied to a transversality function whose time-slope is bounded below uniformly on the section.  We construct the true map first and treat the learned one as its perturbation.

\emph{The true map.}  Let $\mathbf n:=f(p)/v$ and $h_f(t,x):=\mathbf n^\top(\Phi_f(t;x)-p)$.  At $(\That,p)$ we have $h_f=0$ and $\partial_t h_f=\mathbf n^\top f(p)=v>0$, so the implicit function theorem yields $C^1$ maps $\tau_f$ and $P_f(x)=\Phi_f(\tau_f(x);x)$ on a disc $\Sigma_{r_2'}$ with $\tau_f(p)=\That$, $P_f(p)=p$; since $f\in C^2$ the flow is $C^2$ and $D\tau_f$, $DP_f$ are Lipschitz.  It remains to exclude earlier returns.  Fix $\varsigma>0$ so small that $\norm{\ghat(t)-p}\leq r$ for $t\in[0,\varsigma]\cup[\That-\varsigma,\That]$, where $r>0$ is chosen with $\mathbf n^\top f(y)\geq v/2$ whenever $\norm{y-p}\leq 2r$; this is possible by continuity of $f$ and $f(p)\neq 0$.  On the two end intervals $h_f(\cdot,x)$ is then strictly increasing for $x$ close to $p$, so it has no zero there other than the one near $\That$.  On the middle interval set
\[
  d_1\;:=\;\min_{t\in[\varsigma,\That-\varsigma]}\norm{\ghat(t)-p}\;>\;0,
\]
positive because $\Gamma$ is a simple closed curve and the minimum is over a compact set.  Trajectories starting in $\Sigma_{r_2'}$ spread by at most $e^{L\That}$ over one period, so $\norm{\Phi_f(t;x)-\ghat(t)}\leq e^{L\That}r_2'$ there.  Imposing $r_2'\,(1+e^{L\That})<d_1/2$ gives $\norm{\Phi_f(t;x)-p}\geq d_1-e^{L\That}r_2'>r_2'$ on the middle interval, so the true trajectory does not meet the disc there.  Its first return is therefore the crossing near time $\That$ and no other.

\emph{The learned map.}  Let $x\in\Sigma_{r_2}$ with $r_2\leq\min(r_2'/2,r_1)$ and $h_\Th(t,x):=\mathbf n^\top(\PhiTh(t;x)-p)$.  By Lemma~\ref{lem:confine}, $\abs{h_\Th(\tau_f(x),x)}\leq\varepsilon$.  For the time derivative, decompose through the base point $p\in\Gamma$:
\[
  \partial_t h_\Th
  \;=\;\mathbf n^\top \fTh\bigl(\PhiTh(t;x)\bigr)
  \;\geq\;
  \underbrace{\mathbf n^\top f(p)}_{=\,v}
  \;-\;\underbrace{\norm{\fTh(p)-f(p)}}_{\leq\,\delta}
  \;-\;\underbrace{\norm{\fTh(\PhiTh(t;x))-\fTh(p)}}_{\leq\,
  B\,\norm{\PhiTh(t;x)-p}} .
\]
Lemma~\ref{lem:delta-eps} gives $\delta\leq\sqrt{2C_2\varepsilon}$, so after shrinking $\varepsilon_2$ we have $\delta\leq v_{\min}/(2(\That+1))\leq v/(2(\That+1))$, a fixed margin below $v/2$.  For $(t,x)$ with $\abs{t-\tau_f(x)}$, $\norm{x-p}$ and $\varepsilon$ small, $\norm{\PhiTh(t;x)-p}\leq\varepsilon+\norm{\Phi_f(t;x)-p}$ is as small as desired by continuity of the true flow; after one shrinking of $r_2,\varepsilon_2$ and of the time interval, $B\norm{\PhiTh(t;x)-p}\leq v/2-\delta$, hence $\partial_t h_\Th\geq v/2$.  A scalar function with derivative at least $v/2$ and value of modulus at most $\varepsilon$ at $\tau_f(x)$ has a unique nearby zero $\tauTh(x)$, with $\abs{\tauTh(x)-\tau_f(x)}\leq 2\varepsilon/v$.  No earlier crossing of the disc occurs.  On the middle interval the $\fTh$-trajectory stays within $\varepsilon$ of the $f$-trajectory, hence at distance at least $d_1/2-\varepsilon\geq d_1/4>r_2'$ from $p$ once $\varepsilon_2\leq d_1/4$; note $r_2'<d_1/4$ from the choice $r_2'(1+e^{L\That})<d_1/2$.  On the two end intervals the trajectory remains in the $2r$-ball around $p$; after shrinking $r$ so that in addition $B(2r+\varepsilon)\leq v/2-\delta$, the slope bound $\partial_t h_\Th\geq v/2$ holds there, so $h_\Th$ is strictly increasing on both end intervals.  In particular the zero at $t=0$, where the trajectory starts on the section, is left immediately, and the only zero with a crossing of the disc is the one near $\tau_f(x)$.  Smoothness of $\tauTh,\PTh$ and the Lipschitz bound on $D\PTh$ follow from the implicit function theorem and the $C^2$ bounds on $\PhiTh$ from the variational equations, which depend only on $(B,\That)$.  Finally,
\[
  \norm{\PTh(x)-P_f(x)}
  \leq \norm{\PhiTh(\tauTh;x)-\PhiTh(\tau_f;x)}
      +\norm{\PhiTh(\tau_f;x)-\Phi_f(\tau_f;x)}
  \leq B\cdot\tfrac{2}{v}\varepsilon+\varepsilon
  = K_1\varepsilon,
\]
and $\tau_f,\tauTh\in[\That/2,\That+1]$ after one more shrinking.
\end{proof}

\begin{proof}[Proof of Lemma~\ref{lem:persist}]
The proof has two parts.  An adapted norm turns the spectral bound of Assumption~\ref{asm:cert} into a genuine contraction, uniformly over the matrices that can arise; Banach's fixed point theorem then produces the orbit and its bounds.

\emph{Adapted norms.}  For a matrix $A$ with $\rho(A)\leq\rho_*$ and $\varsigma>0$ there is a norm with $\norm{A}_*\leq\rho_*+\varsigma$ (\cite[Proposition~2.84]{Chicone}; a Jordan basis with scaled nilpotent part).  Apply this to $A=D\PTh(p)$ with $\varsigma=(\bar\rho-\rho_*)/2$; by the Lipschitz continuity of $D\PTh$ (constant $\Lambda$), for $\norm{x-p}\leq r_3:=\min(r_2,\,c_1(\bar\rho-\rho_*)/(2c_2\Lambda))$ we get $\norm{D\PTh(x)}_*\leq\bar\rho$, and the mean value inequality makes $\PTh$ a $\bar\rho$-contraction on $\Sigma_{r_3}$ in the adapted norm. The same construction for $f$ (whose $DP_f(p)$ has spectral radius $\rhoT(f)<1$ by Assumption~\ref{asm:true}) gives $\norm{\cdot}_{*f}$, $c_1',c_2'$, $r_3'$.  For $d=2$ the section is one-dimensional and one may take the absolute value, with $c_1=c_2=1$.  For $d\geq 3$ the equivalence constants have to be uniform over the family
\[
  \mathcal{M}:=\bigl\{A\in\R^{(d-1)\times(d-1)}:
  \rho(A)\leq\rho_*,\ \norm{A}\leq C_M\bigr\},
  \qquad C_M:=C(f,B,\That),
\]
which is compact because the spectral radius is continuous.  We claim that one exponent serves the whole family.  Write $\theta:=\rho_*+\varsigma$.  For each $A_0\in\mathcal{M}$, Gelfand's formula gives an exponent $j_0$ with $\norm{A_0^{j_0}}^{1/j_0}\leq\rho_*+\varsigma/2$; the map $A\mapsto\norm{A^{j_0}}^{1/j_0}$ is continuous, so $\norm{A^{j_0}}^{1/j_0}\leq\theta$ on an open ball around $A_0$.  Finitely many such balls cover $\mathcal{M}$; let $j_1,\dots,j_m$ be the associated exponents and put $j_*:=\operatorname{lcm}(j_1,\dots,j_m)$.  A finite subcover alone would only give finitely many exponents; the passage to a common one uses submultiplicativity.  Indeed, if $A\in\mathcal{M}$ lies in the $i$-th ball and $j_*=q\,j_i$, then
\[
  \norm{A^{j_*}}\;\leq\;\norm{A^{j_i}}^{q}
  \;\leq\;\theta^{\,j_i q}\;=\;\theta^{\,j_*} .
\]
For such an $A$ define the adapted norm
\[
  \norm{x}_{*}:=\sum_{i=0}^{j_*-1}\theta^{-i}\,\norm{A^{i}x} .
\]
Summing the same series shifted by one index and using $\norm{A^{j_*}x}\leq\theta^{\,j_*}\norm{x}$ gives $\norm{Ax}_*\leq\theta\norm{x}_*$, that is $\norm{A}_*\leq\theta$.  The $i=0$ term gives $\norm{x}\leq\norm{x}_*$, and $\norm{A^i}\leq C_M^i$ gives $\norm{x}_*\leq c_2\norm{x}$ with $c_2:=\sum_{i=0}^{j_*-1}(C_M/\theta)^i$.  Both $c_1=1$ and $c_2$ depend only on $(j_*,C_M,\rho_*,\varsigma)$, hence are uniform over $\mathcal{M}$.

\emph{Fixed point.}  By Lemma~\ref{lem:return}, $\norm{\PTh(p)-p}_*\leq c_2K_1\varepsilon$.  On the ball of radius $R:=c_2K_1\varepsilon/(1-\bar\rho)$ around $p$,
\[
  \norm{\PTh(x)-p}_*
  \leq \bar\rho\norm{x-p}_*+\norm{\PTh(p)-p}_*
  \leq \bar\rho R + c_2K_1\varepsilon = R,
\]
provided $R\leq c_1r_3$; together with the requirement $\sqrt{2C_2\varepsilon}\leq\min(\delta_1,\,r_3/(2C_U))$, where $\delta_1$ and $C_U$ are the certificate-free constants of Lemma~\ref{lem:winding}(i), used in the stability step below and in Lemma~\ref{lem:winding}(ii), this defines $\varepsilon_3$.  Banach's fixed point theorem gives the unique $x_\Th^*$ with $\norm{x_\Th^*-p}\leq c_1^{-1}R=K_2\varepsilon$ and geometric convergence of the iterates.  For orbital exponential stability, apply Lemma~\ref{lem:winding}\textup{(i)}, which uses neither Assumption~\ref{asm:cert} nor the present lemma: it sends every point of a neighborhood of $\gamma_\Th$ to $\Sigma_{r_3}$ within time $\That+1$.  The section distance to $x_\Th^*$ is then contracted by $\bar\rho$ at each return, while deviations between consecutive crossings are amplified by at most $e^{B(\That+1)}$.  Together these convert the geometric decay on the section into exponential decay of the distance to $\gamma_\Th$.  For the Hausdorff bound, for $t\in[0,\That+1]$,
\[
  \norm{\PhiTh(t;x_\Th^*)-\ghat(t)}
  \leq \norm{\PhiTh(t;x_\Th^*)-\Phi_f(t;x_\Th^*)}
  +\norm{\Phi_f(t;x_\Th^*)-\Phi_f(t;p)}
  \leq \varepsilon+e^{L(\That+1)}K_2\varepsilon =: K_3\varepsilon,
\]
and every point of $\gamma_\Th$ and of $\Gamma$ is of this form.  For the period, $\abs{T_\Th-\That}\leq\abs{\tauTh(x_\Th^*)-\tauTh(p)} +\abs{\tauTh(p)-\tau_f(p)} \leq\Lip(\tauTh)K_2\varepsilon+\tfrac{2}{v}\varepsilon=:K_4\varepsilon$. If $\Ltraj(\Th)=0$, then $\fTh=f$ on $\Gamma$, so $\ghat$ solves $\dot x=\fTh(x)$ and $\PTh(p)=p$; uniqueness gives $x_\Th^*=p$, $\gamma_\Th=\Gamma$, $T_\Th=\That$.
\end{proof}

\begin{proof}[Proof of Lemma~\ref{lem:winding}]
Write $g$ for the field of part~(i); the case of interest is $g=\fTh$.  The strategy is to control the phase and the transverse coordinate of a $g$-trajectory simultaneously.  The two estimates are coupled, because the bound on the phase speed needs the trajectory to stay near $\Gamma$ while the transverse bound needs the phase speed to be of order one, so we derive both on a maximal interval and close the loop by continuation.

Parametrize the tube by $(s,u)\mapsto\ghat(s)+U(s)u$, with $U(s)$ an orthonormal frame of the normal bundle along $\ghat$.  For a small enough radius this is a diffeomorphism onto its image, with distortion constants depending only on $(f,\ghat,\delta_0)$.  Differentiating $x(t)=\ghat(s(t))+U(s(t))u(t)$ along a $g$-trajectory in the tube and projecting on the tangent and the normal directions gives
\begin{equation}\label{eq:tube-ode}
  \dot s
  =\frac{\langle\dot\ghat(s),g(x,t)\rangle}{\norm{\dot\ghat(s)}^2}
   \bigl(1+O(\abs{u})\bigr),
  \qquad
  \dot u
  = U(s)^{\!\top}\bigl[g(x,t)-\dot\ghat(s)\,\dot s
                        -U'(s)u\,\dot s\bigr].
\end{equation}

\emph{Step 1: the continuation interval.}  Fix $\varrho\leq\min(\delta_0,r_2)$ small enough that the phase-speed estimate \eqref{eq:phase-speed} below holds on $N_\varrho$ and that $\varrho<d_1/2$, with $d_1$ the middle-interval distance from the proof of Lemma~\ref{lem:return}; shrinking $\varrho$ only strengthens the estimates below, and require $r_5'\leq\varrho/2$ from the outset.  Set
\[
  t^*:=\sup\bigl\{t\geq 0 : \abs{u(\tau)}\leq\varrho
  \ \text{ for all }\tau\in[0,t]\bigr\},
\]
which is positive because $\abs{u(0)}\leq r_0\leq r_5'\leq\delta_0/4$.  All estimates of Steps 2 and 3 are derived on $[0,t^*)$, where $\abs{u}\leq\varrho$ is available, and Step 4 shows $t^*>\That+1$.

\emph{Step 2: the phase advances.}  At $u=0$ one has $\dot\ghat(s)=f(\ghat(s))$, so the quotient in \eqref{eq:tube-ode} equals $\langle f,g\rangle/\norm{f}^2$ evaluated at $\ghat(s)$.  For the true field this is exactly one.  For $g$ it is $1+O(\delta/v_{\min})$, because $\norm{g-f}\leq\delta$ on $\Gamma$ and $\norm{f}\geq v_{\min}$ there; the two agree only when $g=f$ on $\Gamma$.  Off the orbit we use $\norm{g(x,t)-f(\ghat(s))}\leq\delta+(B{+}2L)\abs{u}$, which combines $\norm{g-f}\leq\delta+(B{+}L)\dist(x,\Gamma)$ with the $L$-Lipschitz continuity of $f$, and obtain on $[0,t^*)$
\begin{equation}\label{eq:phase-speed}
  \dot s \;\geq\;
  1-\frac{\delta+(B{+}2L)\varrho}{v_{\min}}-C\varrho
  \;\geq\;\frac{\That}{\That+1},
  \qquad
  \dot s\;\leq\;2,
\end{equation}
after shrinking $\varrho$, and with it $r_5'$ and $\delta_1$, once; the constant $C$ absorbs the $O(\abs{u})$ distortion of the chart.  In particular $\dot s>0$, so the phase increases strictly and the trajectory cannot stall.

\emph{Step 3: the transverse coordinate.}  Positivity of $\dot s$ constrains the phase only, and by itself does not prevent the trajectory from leaving the tube through its lateral boundary.  The transverse coordinate must be estimated separately, and its equation carries the orbit mismatch as a forcing term.  Insert $g=f+(g-f)$ and $f(x)=f(\ghat(s))+Df(\ghat(s))U(s)u+O(\abs{u}^2)$ into the second equation of \eqref{eq:tube-ode}.  The term proportional to $f(\ghat(s))=\dot\ghat(s)$ is tangent, so $U(s)^{\!\top}\dot\ghat(s)=0$ and that whole contribution, including the part carrying $1-\dot s$, drops out under the projection.  The remaining terms are bounded on $[0,t^*)$, using $\dist(x,\Gamma)\leq\abs{u}\leq\varrho$ and the bound $\dot s\leq 2$ of \eqref{eq:phase-speed}, by
\begin{equation}\label{eq:transverse-ode}
  \abs{\dot u}\;\leq\;C_1\abs{u}+\delta,
  \qquad
  C_1:=L+C\varrho+(B{+}L)+2\sup_s\norm{U'(s)} .
\end{equation}
Gr\"onwall's inequality applied to \eqref{eq:transverse-ode} gives, for $t\in[0,t^*)$ with $t\leq\That+1$,
\begin{equation}\label{eq:winding-transverse}
  \abs{u(t)}\;\leq\;e^{C_1(\That+1)}\bigl(r_0+\delta(\That+1)\bigr)
  \;\leq\;C_U\,(r_0+\delta),
  \qquad
  C_U:=(\That+1)\,e^{C_1(\That+1)} .
\end{equation}
The additive term $\delta$ is the forcing produced by $g-f$ and cannot be removed.  For the true field $\delta=0$ and the estimate reduces to $\abs{u(t)}\leq e^{C_1(\That+1)}r_0$, the form used in Lemma~\ref{lem:hitting}.

\emph{Step 4: closing the continuation.}  Shrink $r_5'$ and $\delta_1$ once more so that $C_U(r_5'+\delta_1)<\varrho$.  Then \eqref{eq:winding-transverse} gives $\abs{u(t)}<\varrho$ strictly, for every $t<\min(t^*,\That+1)$.  Were $t^*\leq\That+1$, continuity would force $\abs{u(t^*)}=\varrho$, contradicting that strict inequality; hence $t^*>\That+1$ and both estimates hold on all of $[0,\That+1]$.  In particular the trajectory stays in $N_{\varrho}\subset N_{\delta_0}$ and never reaches the lateral boundary, on $[0,\That+1]$ and not beyond.

\emph{Step 5: arrival at the section.}  By \eqref{eq:phase-speed} the phase increases at rate at least $\That/(\That+1)$, so $s$ traverses a full period, and the trajectory therefore meets $\Sigma$, within time $\That+1$.  The crossing is transversal because $\dot s>0$ there.  At the crossing $s\equiv 0$ modulo $\That$, so the point is $p+U(0)u$ with $U(0)u\perp\dot\ghat(0)=f(p)$; it therefore lies in the hyperplane $\Sigma$, and its distance to $p$ equals $\abs{u}$, bounded by $C_U(r_0+\delta)$ through \eqref{eq:winding-transverse}.  This proves part~(i).

\emph{Step 6: uniqueness.}  Assume now Assumption~\ref{asm:cert} and $C_U(r_0+\delta)\leq r_3$, and let $\gamma\subset N_{r_0}$ be a periodic orbit of $\fTh$.  By part~(i) it crosses $\Sigma$ at some $\bar x$ with $\norm{\bar x-p}\leq C_U(r_0+\delta)\leq r_3$, so $\bar x\in\Sigma_{r_3}$, and periodicity gives $\PTh^m(\bar x)=\bar x$ for some $m\geq 1$.  Every iterate $\PTh^j(\bar x)$ is again a crossing of the same orbit $\gamma\subset N_{r_0}$, so part~(i), applied from the preceding crossing, keeps it in $\Sigma_{r_3}$ as well.  Each step then contracts in the adapted norm of Lemma~\ref{lem:persist} along the segment to $x_\Th^*$, which stays in the convex disc $\Sigma_{r_3}$, so $\norm{\bar x-x_\Th^*}_*=\norm{\PTh^m(\bar x)-\PTh^m(x_\Th^*)}_*\leq\bar\rho^{\,m}\norm{\bar x-x_\Th^*}_*$, which forces $\bar x=x_\Th^*$ and $\gamma=\gamma_\Th$.
\end{proof}

\begin{proof}[Proof of Lemma~\ref{lem:track}]
The two orbits are compared through the contraction of the true return map, with the difference of the maps as a forcing term.  We first record that the iterates stay where the contraction is available.  Write $r_7:=\min(r_3,r_3')$, so that $\Sigma_{r_7}=\Sigma_{r_3}\cap\Sigma_{r_3'}$, and recall $\norm{x_\Th^*-p}\leq K_2\varepsilon$ from Lemma~\ref{lem:persist}.  Both balls below are calibrated to $\Sigma_{r_7}$, and the one centred at $x_\Th^*$ carries the offset $K_2\varepsilon$ so that it lies in $\Sigma_{r_7}$ and not merely within $r_7$ of $x_\Th^*$.  In the adapted norms of Lemma~\ref{lem:persist} the balls
\[
  B_*:=\{y:\norm{y-x_\Th^*}_*\leq c_1(r_7-K_2\varepsilon)\},
  \qquad
  B_{*f}:=\{y:\norm{y-p}_{*f}\leq c_1'r_7\}
\]
both lie in $\Sigma_{r_7}$, and each is invariant under its own map.  For the inclusions, $y\in B_{*f}$ gives $\norm{y-p}\leq(c_1')^{-1}\norm{y-p}_{*f}\leq r_7$, while $y\in B_*$ gives $\norm{y-p}\leq\norm{y-x_\Th^*}+\norm{x_\Th^*-p}\leq(r_7-K_2\varepsilon)+K_2\varepsilon=r_7$, the offset being what the radius of $B_*$ was reduced by.  Each ball is convex and centred at the fixed point of its own map, and $\Sigma_{r_7}\subseteq\Sigma_{r_3}\cap\Sigma_{r_3'}$ is where the two contraction bounds of Lemma~\ref{lem:persist} hold, so the mean value inequality along the segment to the centre gives invariance.  Choose
\[
  r_4\;\leq\;\tfrac12\min\{c_1/c_2,\ c_1'/c_2'\}\,r_7,
  \qquad r_4\leq r_5,
\]
and $\varepsilon_4$ so small that $K_2\varepsilon\leq\tfrac{r_7}{6}\min\{c_1/c_2,\ c_1'/c_2'\}$.  The disc $\Sigma_{r_4}$, enlarged by $K_2\varepsilon$, then lies in $B_*\cap B_{*f}$: this inclusion runs from the Euclidean norm into the adapted ones and so costs $c_2$ and $c_2'$, requiring $c_2'(r_4+K_2\varepsilon)\leq c_1'r_7$ and $c_2(r_4+2K_2\varepsilon)\leq c_1(r_7-K_2\varepsilon)$.  The first holds since $c_2'r_4\leq\tfrac12 c_1'r_7$ and $c_2'K_2\varepsilon\leq\tfrac16 c_1'r_7$; the second since $c_2r_4\leq\tfrac12 c_1r_7$, $2c_2K_2\varepsilon\leq\tfrac13 c_1r_7$, and $c_1(r_7-K_2\varepsilon)\geq\tfrac56 c_1r_7$.  Then every iterate of either sequence remains in $\Sigma_{r_7}=\Sigma_{r_3}\cap\Sigma_{r_3'}$, which is what the estimate below uses.  Set $\tilde d_k:=\norm{P_f^k(y_f)-\PTh^k(y_\Th)}_{*f}$.  Then
\[
  \tilde d_{k+1}
  \leq \norm{P_f(P_f^ky_f)-P_f(\PTh^ky_\Th)}_{*f}
  +\norm{(P_f-\PTh)(\PTh^ky_\Th)}_{*f}
  \leq \bar\rho_f\,\tilde d_k + c_2'K_1\varepsilon,
\]
by Lemma~\ref{lem:persist} and Lemma~\ref{lem:return}.  With $\tilde d_0\leq c_2'K_1\varepsilon$, induction gives $\tilde d_k\leq c_2'\,\tfrac{2K_1}{1-\bar\rho_f}\varepsilon$ for all $k$, and norm equivalence concludes.
\end{proof}

\begin{lemma}[hitting the section]\label{lem:hitting}
There are $r_6\in(0,r_5/2]$ and $\varepsilon_5\in(0,\varepsilon_4]$ such that for $\varepsilon\leq\varepsilon_5$ and every $x_0\in N_{r_6}$: the $f$- and $\fTh$-trajectories from $x_0$ first cross $\Sigma_{r_4}$ transversally at times $\sigma_f,\sigma_\Th\leq\That+1$ and points $y_f,y_\Th\in\Sigma_{r_4}$, with $\abs{\sigma_\Th-\sigma_f}\leq\tfrac{2}{v}\varepsilon$ and $\norm{y_\Th-y_f}\leq K_1\varepsilon$.
\end{lemma}

\begin{proof}
The estimate is Lemma~\ref{lem:winding}\textup{(i)} for the true field, sharpened so that the crossing lands in $\Sigma_{r_4}$, together with a simple-zero argument for the learned crossing.  Lemma~\ref{lem:winding}\textup{(i)} applied to $f$ itself, the case $\delta=\varepsilon=0$, yields the first $f$-crossing within time $\That+1$, transversal, at $y_f$ with $\norm{y_f-p}\leq C_Ur_6\leq r_4/2$ after shrinking $r_6$; then $\norm{y_\Th-p}\leq r_4/2+K_1\varepsilon\leq r_4$ for small $\varepsilon$.  The crossing is a simple zero with slope at least $v/2$ of $h_f(t):=\mathbf n^\top(\Phi_f(t;x_0)-p)$; by Lemma~\ref{lem:confine}, $\abs{h_\Th(\sigma_f)}\leq\varepsilon$ and $\partial_t h_\Th\geq v/2$ near $\sigma_f$ as in the proof of Lemma~\ref{lem:return}, so $h_\Th$ has a unique nearby zero $\sigma_\Th$ with $\abs{\sigma_\Th-\sigma_f}\leq 2\varepsilon/v$.  That this is the \emph{first} $\fTh$-crossing of the disc uses more than monotonicity of the phase, since the trajectory also meets the hyperplane at phases away from $0$.  Since $\varrho<d_1/2$ by construction in the proof of Lemma~\ref{lem:winding}, shrink $r_4$ so that $r_4<d_1/2$, whence $r_4+\varrho<d_1$, with $d_1$ the middle-interval distance of the proof of Lemma~\ref{lem:return} and $\varrho$ the tube radius of Lemma~\ref{lem:winding}.  Any point of the tube within $r_4$ of $p$ then has phase within $\varsigma$ of $0$ modulo $\That$, so the hyperplane crossings at intermediate phases fall outside $\Sigma_{r_4}$, and monotonicity of the phase gives the first crossing.  The point estimate follows as in Lemma~\ref{lem:return}.
\end{proof}

\begin{proof}[Proof of Lemma~\ref{lem:clocks}]
Each return adds a clock difference of order $\varepsilon$, and summing gives the bound.  Each increment satisfies
\[
  \abs{\tauTh(\PTh^ky_\Th)-\tau_f(P_f^ky_f)}
  \leq \abs{(\tauTh-\tau_f)(\PTh^ky_\Th)}
  +\abs{\tau_f(\PTh^ky_\Th)-\tau_f(P_f^ky_f)}
  \leq \tfrac{2}{v}\varepsilon+\Lip(\tau_f)K_5\varepsilon,
\]
by Lemmas~\ref{lem:return} and~\ref{lem:track}; summing over $k<n$ gives the claim.
\end{proof}

\begin{lemma}[scalar multiplier in dimension two]\label{lem:liouville}
Let $d=2$ and let Assumptions~\ref{asm:true}--\ref{asm:learned} hold with $\eta=0$ and $\varepsilon\leq\varepsilon_2$.  If $\bar x\in\Sigma_{r_2}$ is a fixed point of the return map $\PTh$, with orbit $\bar\gamma:=\{\PhiTh(t;\bar x):0\leq t\leq\bar\tau\}$ of period $\bar\tau:=\tauTh(\bar x)$, then
\[
  D\PTh(\bar x)
  \;=\;
  \exp\Bigl(\int_0^{\bar\tau}\operatorname{div}\fTh(\bar\gamma(t))\,
  \mathrm{d}t\Bigr)\;>\;0 .
\]
No contraction hypothesis is used; the identity holds at any periodic orbit of the autonomous planar field $\fTh$.
\end{lemma}

\begin{proof}
The identity is Liouville's formula for the variational equation along the periodic orbit, read on the section.  By Lemma~\ref{lem:return} the return map $\PTh$ is $C^1$ at $\bar x$ and $\fTh(\bar x)\neq 0$, so $\bar\gamma$ is a genuine periodic orbit of the autonomous field $\fTh$.  Let $X(t)$ solve the variational equation $\dot X=D\fTh(\bar\gamma(t))X$, $X(0)=I$, and $M:=X(\bar\tau)$.  Since $\fTh$ is autonomous, $t\mapsto\fTh(\bar\gamma(t))$ solves the variational equation and is $\bar\tau$-periodic, so $M\fTh(\bar x)=\fTh(\bar x)$: the tangent direction is a Floquet direction with multiplier one.  In the basis $\{\fTh(\bar x),w\}$ with $w$ spanning $T_{\bar x}\Sigma$, the matrix $M$ is block triangular with diagonal $(1,m)$ and $m=D\PTh(\bar x)$ \cite[Proposition~2.80]{Chicone}.  Liouville's formula gives $\det M=\exp\bigl(\int_0^{\bar\tau}\operatorname{tr} D\fTh(\bar\gamma)\,\mathrm{d}t\bigr)$, and $\det M=1\cdot m$; positivity follows from the exponential.
\end{proof}

\subsection{\texorpdfstring{Proof of Theorem~\ref{thm:floquet}}{Proof of the linear long-horizon bound}}\label{app:pf-floquet}

We now prove the linear long-horizon bound, assembling the lemmas of Appendix~\ref{app:pf-lemmas}.

\begin{proof}[Proof of Theorem~\ref{thm:floquet}]
The proof assembles the lemmas.  Transverse error is contracted at every return (Lemmas~\ref{lem:persist} and~\ref{lem:track}), the two clocks drift apart by $O(\varepsilon)$ per return (Lemma~\ref{lem:clocks}), and Gr\"onwall estimates are applied over single periods only, through Lemma~\ref{lem:confine}.  We reduce to the section, then prove the three parts in order.

\emph{Reduction to the section.}  Set $r_0:=r_6$, with $r_6$ shrunk once more if necessary so that $N_{r_6}\subseteq N_{2r_1}$, and $\varepsilon_0:=\min\{\varepsilon_5,\ r_0/K_3,\ v/2\}$, so that the hypothesis $K_3\varepsilon\leq r_0$ of Lemma~\ref{lem:winding}(ii) and the bound $2\varepsilon/v\leq1$ used in part (iii) hold for every $\varepsilon\leq\varepsilon_0$; the $\delta$-smallness conditions of Lemmas~\ref{lem:persist} and~\ref{lem:winding}(ii) are folded into $\varepsilon_3$, and every statement used below holds a fortiori for smaller radii and thresholds.  For $t\leq\That+1$, parts (ii)--(iii) follow from Lemma~\ref{lem:confine}, part (ii) together with the comparison $\dist(\Phi_f(t;x_0),\Gamma)\leq e^{L(\That+1)}\dist(x_0,\Gamma)$, absorbed into $C$.  For $t>\That+1$, Lemma~\ref{lem:hitting} provides crossing data $\sigma_f,\sigma_\Th\leq\That+1$, $\abs{\sigma_\Th-\sigma_f}\leq 2\varepsilon/v$, and $y_f,y_\Th\in\Sigma_{r_4}$ with $\norm{y_f-y_\Th}\leq K_1\varepsilon$, such that $\Phi_f(t;x_0)=\Phi_f(t-\sigma_f;y_f)$ and $\PhiTh(t;x_0)=\PhiTh(t-\sigma_\Th;y_\Th)$.  Write $t':=t-\sigma_f$.  Below we evaluate both flows at the common time $t'$; for the learned flow this replaces $t-\sigma_\Th$ by $t'$, which costs $\norm{\PhiTh(t-\sigma_\Th;y_\Th)-\PhiTh(t';y_\Th)}\leq B\abs{\sigma_\Th-\sigma_f}\leq 2B\varepsilon/v$, an $O(\varepsilon)$ term absorbed into the constants of parts (ii)--(iii) without further mention.

\emph{Part (i).}  This is Lemmas~\ref{lem:persist} and~\ref{lem:winding}.

\emph{Part (ii).}  Let $n$ be the number of $\fTh$-crossings in $(0,t']$, so $t'=\tilde t_n^\Th+s$ with $s\in[0,\That+1)$, where $\tilde t_n^\Th$ are the $\fTh$-crossing times from $y_\Th$.  Then $\PhiTh(t';y_\Th)=\PhiTh(s;\PTh^ny_\Th)$ and
\[
  \dist\bigl(\PhiTh(t';y_\Th),\gamma_\Th\bigr)
  \leq \norm{\PhiTh(s;\PTh^ny_\Th)-\PhiTh(s;x_\Th^*)}
  \leq e^{B(\That+1)}\norm{\PTh^ny_\Th-x_\Th^*}
  \leq C\bar\rho^{\,n}\norm{y_\Th-x_\Th^*} .
\]
Since return times are at most $\That+1$, $n\geq t'/(\That+1)-1$; together with $\dist_H(\gamma_\Th,\Gamma)\leq K_3\varepsilon$ and $\norm{y_\Th-x_\Th^*}\leq C(\dist(x_0,\Gamma)+\varepsilon)$ this proves (ii); the shifts from $t'$ to $t$ and from $\sigma_f$ to $\sigma_\Th$ are absorbed in $C$.

\emph{Part (iii).}  Let $n$ be the number of $f$-crossings in $(0,t']$: $t'=t_n^f+s$, $s\in[0,\That+1)$, and $\Phi_f(t';y_f)=\Phi_f(s;P_f^ny_f)$.  Set $\delta_n:=(t_n^f-t_n^\Th)+(\sigma_f-\sigma_\Th)$, so that $\PhiTh(t;x_0)=\PhiTh(s+\delta_n;\PTh^ny_\Th)$; by Lemma~\ref{lem:clocks}, $\abs{\delta_n}\leq K_6n\varepsilon+2\varepsilon/v$.  Assume first $K_6n\varepsilon\leq 1$, so that $\abs{\delta_n}\leq 1+2\varepsilon/v\leq 2$.  Then $s+\delta_n\in[-2,\That+3]$.  The point $\PhiTh(s+\delta_n;\PTh^ny_\Th)$ lies on the already constructed forward trajectory, and re-indexed at finitely many preceding or following crossings, until the shifted time first lands in $[0,\That+2]$, which is possible because consecutive crossings are at most $\That+1$ apart, it is $\PhiTh(t'';z)$ with $z\in\Sigma_{r_2}$ and $t''\in[0,\That+2]$, so the confinement of Lemma~\ref{lem:confine} applies to it and keeps all points below in $N_{\delta_0}$.  The closeness bound of that lemma is used only at the time $s\in[0,\That+1)$; hence
\begin{align*}
  \norm{\PhiTh(t;x_0)-\Phi_f(t;x_0)}
  &\leq \norm{\PhiTh(s+\delta_n;\PTh^ny_\Th)-\PhiTh(s;\PTh^ny_\Th)}\\
  &\quad+\norm{\PhiTh(s;\PTh^ny_\Th)-\Phi_f(s;\PTh^ny_\Th)}
  +\norm{\Phi_f(s;\PTh^ny_\Th)-\Phi_f(s;P_f^ny_f)}\\
  &\leq B\abs{\delta_n}+\varepsilon+e^{L(\That+1)}K_5\varepsilon
  \;\leq\; C\varepsilon\Bigl(1+\frac{t}{\That}\Bigr),
\end{align*}
using Lemma~\ref{lem:confine} for the middle term (the first term needs no confinement, since it uses only the global bound $\norm{\fTh}\leq B$ supplied by the cutoff extension), Lemma~\ref{lem:track} for the last, and $n\leq 2t/\That+1$ from the lower bound $\That/2$ on return times; one may take $C=1+e^{L(\That+1)}K_5+2B(K_6+1/v)$.  If $K_6n\varepsilon>1$, both trajectories lie in $N_{\delta_0}$ by confinement and (ii), so $\norm{e(t)}\leq\operatorname{diam}(N_{\delta_0}) \leq\operatorname{diam}(N_{\delta_0})K_6n\varepsilon \leq C\varepsilon(1+t/\That)$ upon enlarging $C$.
\end{proof}

\subsection{\texorpdfstring{Proofs of Proposition~\ref{prop:c1-cert} and Corollary~\ref{cor:lf-cert}}{Proofs of the certification results}}\label{app:pf-cert}

This part proves the two certification results of Section~\ref{ssec:cert}: the $C^1$ route in any dimension and the Floquet-loss route in the autonomous plane.

\begin{proof}[Proof of Proposition~\ref{prop:c1-cert}]
We prove the monodromy comparison (ii) first, then the orbit persistence (i), and then the spectral consequence (iii); (ii) comes first because the comparison along the reference orbit is what the other two parts perturb.

\emph{Part (ii).}  Let $X_\Th(t),X(t)$ be the fundamental matrices of the variational equations of $\fTh,f$ along $\ghat$, with $X_\Th(0)=X(0)=I$.  Writing the difference by variation of constants and using $\norm{\fTh-f}_{C^1(U)}\leq\varepsilon_1$ and the $C^1$-bound of $f$ on $U$ gives $\norm{X_\Th(t)-X(t)}\leq\That\,e^{C\That}\varepsilon_1$ for $t\in[0,\That]$.  Restricting to the section $\Sigma\perp f(p)$ and projecting out the tangent direction $f(p)$, which is the orthogonal projection onto $\Sigma=f(p)^{\perp}$ and so has norm one, bounds the transverse monodromies along $\ghat$, $\norm{M_{\Th,\ghat}^\perp-M_f^\perp}\leq C\varepsilon_1$, where $M_{\Th,\ghat}^\perp$ is the projected fundamental matrix of the $\fTh$-variational equation along $\ghat$; the transition to the monodromy $M_{\Th}^\perp:=M_{\Th,\Gamma_\Th}^\perp$ along the learned orbit itself, which is the object in~(ii), is made at the end of part~(i).

\emph{Part (i).}  The true return map satisfies $P_f(p)=p$ with $DP_f(p)=M_f^\perp$ having no eigenvalue $1$ (hyperbolicity), so $I-DP_f(p)$ is invertible.  Since $\fTh$ is $C^1$-close to $f$, its return map $\PTh$ is a $C\varepsilon_1$-perturbation of $P_f$ in $C^1$.  The implicit function theorem applied to $x\mapsto\PTh(x)-x$ then yields a unique fixed point $x_\Th^*$ with $\norm{x_\Th^*-p}\leq C\varepsilon_1$.  Its orbit $\Gamma_\Th$ satisfies $\dist_H(\Gamma_\Th,\Gamma)\leq C\varepsilon_1$.  It remains to compare the monodromy of $\fTh$ along its own orbit $\Gamma_\Th$ with the one along $\ghat$.  Boundedness of $D\fTh$ does not suffice, because the two variational equations are posed along different curves and over different periods.  What is needed is the Lipschitz continuity of the coefficient, and this follows from $f\in C^2$ alone: writing $\Lambda_f:=\Lip(Df)$ on the compact tube, a quantity determined by $f$,
\[
  \norm{D\fTh(y)-D\fTh(z)}
  \;\leq\;2\varepsilon_1+\Lambda_f\norm{y-z}
  \qquad\text{for } y,z\in U,
\]
since $\norm{D\fTh-Df}\leq\varepsilon_1$ at both points.  No bound on $\norm{\fTh}_{C^2}$ is used, so the constants below still depend only on $(f,\Gamma)$.  Parametrize $\Gamma_\Th$ by $\gamma_\Th(t)=\PhiTh(t;x_\Th^*)$.  On $U$ the learned field is $(L+\varepsilon_1)$-Lipschitz, so
\begin{align*}
  \norm{\gamma_\Th(t)-\ghat(t)}
  &\leq\norm{\PhiTh(t;x_\Th^*)-\PhiTh(t;p)}
   +\norm{\PhiTh(t;p)-\Phi_f(t;p)}\\
  &\leq e^{(L+\varepsilon_1)t}\norm{x_\Th^*-p}
   +t\,e^{(L+\varepsilon_1)t}\varepsilon_1
  \;\leq\; C\varepsilon_1
\end{align*}
for $t\in[0,\That+1]$.  Let $Y$ and $X_\Th$ solve the variational equations of $\fTh$ along $\gamma_\Th$ and along $\ghat$, both from the identity.  By the display above their coefficients differ by at most $2\varepsilon_1+\Lambda_f C\varepsilon_1$, so variation of constants and Gr\"onwall give $\norm{Y(t)-X_\Th(t)}\leq C\varepsilon_1$ on $[0,\That+1]$.  The two periods differ by $\abs{T_\Th-\That}\leq C\varepsilon_1$, and $\norm{\dot Y}\leq(L+\varepsilon_1)e^{(L+\varepsilon_1)(\That+1)}$, so evaluating at $T_\Th$ rather than $\That$ costs a further $C\varepsilon_1$.  One step remains, because the transverse monodromies are read off through different projections: the return map at $x_\Th^*$ projects along $\fTh(x_\Th^*)$ and the one at $p$ along $f(p)$, onto the same hyperplane $\Sigma$.  The two projectors differ by at most $C\norm{\fTh(x_\Th^*)-f(p)}\leq C\varepsilon_1$.  Combining the four estimates with~(ii) yields $\norm{M_{\Th,\Gamma_\Th}^\perp-M_f^\perp}\leq C\varepsilon_1$.  Uniqueness of $\Gamma_\Th$ among the periodic orbits of a tube around $\Gamma$ is deferred, because it is not a consequence of the uniqueness of the fixed point on $\Sigma_r$: an orbit might cross the section several times, or outside that disc.  It follows from Lemma~\ref{lem:winding}\textup{(ii)}, whose hypotheses are $\delta\leq\delta_1$, satisfied here since $\delta\leq\varepsilon_1<\varepsilon_*$, and Assumption~\ref{asm:cert}.  The latter is not available yet, so we return to it after part~(iii).

\emph{Part (iii).}  The spectral radius is continuous at $M_f^\perp$, so with $\bar\rho_f=(1+\rhoT(f))/2$, the constant of Lemma~\ref{lem:persist}, there is $\varepsilon_*>0$, depending only on $(f,\Gamma)$, such that $\rhoTh=\rho\bigl(M_{\Th,\Gamma_\Th}^\perp\bigr)\leq\bar\rho_f<1$ for $\varepsilon_1<\varepsilon_*$.  This is a bound at the fixed point of the learned return map, whereas \eqref{eq:cert-cond} is posed at the base point $p$, so one further step is needed.  By part~\textup{(i)} the two points lie within $C\varepsilon_1$ of each other and $D\PTh$ is Lipschitz on the disc, so $\rho(D\PTh(p))\leq\bar\rho_f+\omega(C'\varepsilon_1)$, with $\omega$ the modulus of continuity of the spectral radius on the norm ball $\{A:\norm{A}\leq C_M\}$ of the proof of Lemma~\ref{lem:persist}.  As $\bar\rho_f<1$, shrinking $\varepsilon_*$ once more makes the right side at most $(1+\bar\rho_f)/2<1$, so Assumption~\ref{asm:cert} holds with $\rho_*:=(1+\bar\rho_f)/2$.  Lemma~\ref{lem:winding}\textup{(ii)} now applies and completes the uniqueness claim of part~(i); the constant it contributes depends on $\bar\rho_f$, hence on $\rhoT(f)$, in addition to $(f,\Gamma)$.  In dimension two the transverse monodromy is a scalar and the rate is linear, $\rhoTh\leq\rhoT(f)+C\varepsilon_1$ with $\varepsilon_*=(1-\rhoT(f))/(2C)$.  For $d>2$ no linear rate holds in general: a defective $M_f^\perp$ can move its spectral radius by $O(\varepsilon_1^{1/k})$ under an $O(\varepsilon_1)$ perturbation.  When $M_f^\perp=V\Lambda V^{-1}$ is diagonalizable, the Bauer--Fike theorem \cite{BauerFike1960} restores the quantitative threshold:
\[
  \rhoTh\;\leq\;\rhoT(f)+\varkappa(V)\,C\varepsilon_1,
  \qquad
  \varkappa(V):=\norm{V}\,\norm{V^{-1}},
\]
with $\varepsilon_*=(1-\rhoT(f))/(2\varkappa(V)C)$.  This is the route used for certification in Section~\ref{ssec:cert}.
\end{proof}

\begin{proof}[Proof of Corollary~\ref{cor:lf-cert}]
Throughout $d=2$ and $\eta=0$, so $\fTh$ is autonomous; only Assumptions~\ref{asm:true}--\ref{asm:learned} are used, never the certificate being derived.

\emph{Step 1: the return map, without a certificate.} By Lemma~\ref{lem:return}, $\PTh:\Sigma_{r_2}\to\Sigma$ is well defined and $C^1$ with $\norm{\PTh-P_f}_{C^0(\Sigma_{r_2})}\leq K_1\varepsilon$; this step uses only Assumptions~\ref{asm:true} and~\ref{asm:learned}. The construction of the adapted norm for the \emph{true} map in the proof of Lemma~\ref{lem:persist} uses only Assumption~\ref{asm:true} and provides $\norm{\cdot}_{*f}$ with equivalence constants $c_1',c_2'$ and a radius $r_3'$ on which $\norm{DP_f}_{*f}\leq\bar\rho_f:=(1+\rhoT(f))/2<1$.

\emph{Step 2: existence and location of the learned orbit.} Since $P_f(p)=p$, the closed $\norm{\cdot}_{*f}$-ball $\bar B:=\{y:\norm{y-p}_{*f}\leq R\}$ with $R:=c_2'K_1\varepsilon/(1-\bar\rho_f)$ lies inside $\Sigma_{r_3'}$ once $\varepsilon$ is small, and there $P_f(\bar B)\subseteq\{\,\norm{\cdot-p}_{*f}\leq\bar\rho_f R\,\}$ by the mean value inequality; $\norm{\PTh-P_f}_{*f}\leq c_2'K_1\varepsilon=(1-\bar\rho_f)R$ then gives $\PTh(\bar B)\subseteq\bar B$.  In dimension two $\Sigma$ is a line and $\bar B$ a compact interval, so the continuous self-map $\PTh$ has a fixed point in $\bar B$ by the intermediate value theorem.  Moreover every fixed point $\bar x\in\Sigma_{r_3'}$ obeys $\norm{\bar x-p}_{*f}=\norm{\PTh(\bar x)-P_f(p)}_{*f} \leq c_2'K_1\varepsilon+\bar\rho_f\norm{\bar x-p}_{*f}$, the segment $[p,\bar x]$ staying in $\Sigma_{r_3'}$, hence $\norm{\bar x-p}\leq c_1'^{-1}R=:K_2'\varepsilon$.  The relaxed orbit $\ghat_\Th$ of Definition~\ref{def:floquet-loss}, once it closes into a periodic orbit inside the tube $N_{r_5'}$ (the implementation tests closure; membership of the tube is assumed, as discussed after \eqref{eq:rho-surrogate}), crosses $\Sigma$ within distance $C_U(r_5'+\delta)$ of $p$ by Lemma~\ref{lem:winding}\textup{(i)}, which uses neither Assumption~\ref{asm:cert} nor an adapted radius.  Both terms are at our disposal: $r_5'$ may be shrunk, and $\delta\leq\sqrt{2C_2\varepsilon}$ by Lemma~\ref{lem:delta-eps}, so $C_U(r_5'+\delta)\leq r_3'$ once $r_5'$ and $\varepsilon$ are small.  A closed orbit may a priori meet the section several times before closing, so that crossing is only a periodic point, $\PTh^m(\bar x)=\bar x$ for some $m\geq 1$.  In dimension two we can rule out $m\geq2$ without a certificate.  By Step~2 of the proof of Lemma~\ref{lem:winding} the phase increases strictly along every trajectory in the tube, so $\ghat_\Th$ winds monotonically around the annulus $N_{r_5'}$.  Being a periodic orbit of the autonomous field $\fTh$, it is a simple closed curve, by uniqueness of solutions; a simple closed curve in the annulus has winding number zero or $\pm1$ about $\Gamma$, and strict phase monotonicity excludes zero.  Its winding number is therefore one, so it meets the transverse section exactly once per period, whence $m=1$.  The crossing is therefore a fixed point of $\PTh$ in $\Sigma_{r_3'}$, and within $K_2'\varepsilon$ of $p$ by the a-priori bound.  We take $x_\Th^*$ to be this crossing and $\gamma_\Th:=\{\PhiTh(t;x_\Th^*)\}=\ghat_\Th$ its orbit, of period $T_\Th:=\tauTh(x_\Th^*)$ with $\abs{T_\Th-\That}\leq K_4'\varepsilon$, where $K_4':=\Lip(\tau_f)K_2'+2/v$ is built from the certificate-free constant $K_2'$ of this step and Lemma~\ref{lem:return}.  When the relaxation degenerates instead, the implementation declares the surrogate unevaluable and no certificate is claimed; the intermediate-value fixed point above still guarantees that the periodic orbit exists.

\emph{Step 3: the scalar multiplier.} Lemma~\ref{lem:liouville} applied to $x_\Th^*$ gives
\[
  D\PTh(x_\Th^*)
  =\exp\Bigl(\int_0^{T_\Th}\operatorname{div}\fTh(\gamma_\Th)\Bigr)>0,
\]
which in dimension two equals $\rho(D\PTh(x_\Th^*))$.

\emph{Step 4: surrogate versus multiplier.} The surrogate is $\tilde\rho_T(\Th)=\exp\bigl(\int_0^{\That}\operatorname{div}_x\fTh(\ghat_\Th)\bigr)$, the same integrand along the same orbit $\ghat_\Th=\gamma_\Th$, re-parametrized to start at $x_\Th^*$, a phase shift that changes neither integral by periodicity of the integrand; the two differ only in the integration length $\That$ versus $T_\Th$.  With $\abs{\operatorname{div}\fTh}\leq dB$ on the tube,
\[
  \bigl|\log\tilde\rho_T(\Th)-\log\rho(D\PTh(x_\Th^*))\bigr|
  \;=\;\Bigl|\int_{T_\Th}^{\That}\operatorname{div}\fTh(\gamma_\Th)\Bigr|
  \;\leq\; dB\,K_4'\varepsilon ,
\]
so, using $\abs{a-b}\leq\max(a,b)\,(e^{\kappa}-1)$ with $\kappa:=dB K_4'\varepsilon$ and $\max(a,b)\leq e^{dB(\That+1)}$, we get $\abs{\tilde\rho_T(\Th)-\rho(D\PTh(x_\Th^*))}\leq C\varepsilon$ with $C:=dB\,K_4'\,e^{dB(\That+2)}$, valid once $\kappa\leq dB$; the constant involves no certificate.

\emph{Step 5: conclusion.} If $\Ltraj(\Th)=0$ then $\Phi_\Th(\cdot;p)$ is $\That$-periodic, so $p$ is a fixed point of $\PTh$ and $\Gamma$ is an orbit of $\fTh$.  This alone does not identify $p$ with $x_\Th^*$: the map $\PTh$ is not yet known to be a contraction, and a second cycle of $\fTh$ could carry the relaxation elsewhere.  The hypothesis $\ghat_\Th=\Gamma$ of~(i) supplies the identification, and the implementation checks it by testing that the relaxed curve passes near $p$.  Under it $x_\Th^*=p$ and $T_\Th=\That$, so Step~4 collapses to equality and $\rho(D\PTh(p))=\tilde\rho_T(\Th)\leq\rho_*$, proving~(i).  In general, Step~4 and $\tilde\rho_T(\Th)\leq\rho_*$ give $\rho(D\PTh(x_\Th^*))\leq\rho_*+C\varepsilon<1$ once $\varepsilon\leq(1-\rho_*)/(2C)$, which is Assumption~\ref{asm:cert} at $x_\Th^*$; since $\norm{x_\Th^*-p}\leq K_2'\varepsilon$ and $D\PTh$ is Lipschitz (Lemma~\ref{lem:return}), the same holds at $p$ with $C':=C+\Lip(D\PTh)K_2'$, so $\rho(D\PTh(p))\leq\rho_*+C'\varepsilon<1$ once $\varepsilon\leq(1-\rho_*)/(2C')$, proving~(ii).  Theorem~\ref{thm:floquet} then applies and the chain \eqref{eq:main-chain} holds.
\end{proof}

\subsection{\texorpdfstring{Proof of Lemma~\ref{lem:nearaut}}{Proof of the near-autonomy reduction}}\label{app:pf-nearaut}

This part proves the reduction of the periodic encoding to the autonomous case.  The convention $\eta=0$ of Appendix~\ref{app:pf-lemmas} is dropped here.

\begin{proof}[Proof of Lemma~\ref{lem:nearaut}]
The proof has two parts.  Part (i) compares the encoded flow with the flow of the averaged field, in $C^0$ and in $C^1$, by a Gr\"onwall estimate whose forcing is the oscillation $\eta$.  Part (ii) applies Theorem~\ref{thm:floquet} to the averaged field and transfers its conclusions back to the encoded flow.

\emph{Part (i): the averaged flow is $C_B\eta$-close.}  Averaging preserves the $C^2$ bound, so $\norm{\bar f_{\Th}}_{C^2(N_{\delta_0})}\leq B$ and $D\bar f_{\Th}$ is $B$-Lipschitz.  Fix $s\in\R$ and $x\in N_{2r_1}$, and set $w(t):=\PhiTh^{(s)}(t;x)-\Phi_{\bar f_{\Th}}(t;x)$.  While both trajectories remain in $N_{\delta_0}$,
\[
  \dot w
  =\bigl[\fTh(\PhiTh^{(s)},s+t)-\bar f_{\Th}(\PhiTh^{(s)})\bigr]
  +\bigl[\bar f_{\Th}(\PhiTh^{(s)})-\bar f_{\Th}(\Phi_{\bar f_{\Th}})\bigr],
\]
so $\norm{\dot w}\leq\eta+B\norm{w}$, and Gr\"onwall with $w(0)=0$ gives $\norm{w(t)}\leq\eta\,t\,e^{Bt}$ on $[0,\That+2]$; confinement of both trajectories follows as in Lemma~\ref{lem:confine} once $\varepsilon+C_B\eta\leq\delta_0/4$.  For the derivative, the variational solutions $V^{(s)}:=D_x\PhiTh^{(s)}$ and $\bar V:=D_x\Phi_{\bar f_{\Th}}$ satisfy $\norm{V^{(s)}}, \norm{\bar V}\leq e^{Bt}$ and
\[
  \Bigl\|\tfrac{\mathrm{d}}{\mathrm{d}t}\bigl(V^{(s)}-\bar V\bigr)\Bigr\|
  \;\leq\;
  \bigl(\eta+B\norm{w}\bigr)e^{Bt}+B\norm{V^{(s)}-\bar V},
\]
using $\norm{D_x\fTh(\cdot,s+t)-D\bar f_{\Th}}\leq\eta$ and the $B$-Lipschitz continuity of $D\bar f_{\Th}$; Gr\"onwall again gives $\norm{V^{(s)}(t)-\bar V(t)}\leq \eta\,t\,(1+B t)\,e^{2Bt}$.  The sum of the two bounds is at most $C_B\eta$ on $[0,\That+2]$.  The closeness of $\bar f_{\Th}$ to $f$ follows from the case $s=0$ and the triangle inequality; on $\Gamma$, $\norm{\bar f_{\Th}-f}\leq\delta+\eta$ since $\bar f_{\Th}-\fTh(\cdot,t)$ has norm at most $\eta$.

\emph{(ii)} \emph{Step 1: the averaged field.} By (i), $\bar f_{\Th}$ is an autonomous $C^2$ field satisfying Assumptions~\ref{asm:learned}--\ref{asm:cert} with closeness $\varepsilon+C_B\eta$, orbit mismatch $\delta+\eta$, and contraction constant $\rho_*$; after decreasing $\varepsilon_6$ these place it in the scope of Theorem~\ref{thm:floquet}.  Its conclusions hold for $\Phi_{\bar f_{\Th}}$: a unique orbit $\bar\gamma_\Th$ with the persistence bounds of part (i), and
\[
  \norm{\Phi_{\bar f_{\Th}}(t;x_0)-\Phi_f(t;x_0)}
  \;\leq\;
  C(\varepsilon+C_B\eta)\Bigl(1+\frac{t}{\That}\Bigr).
\]

\emph{Step 2: the encoded field against the averaged field.} Rerun the proof of Theorem~\ref{thm:floquet}(iii) with reference $\bar f_{\Th}$, whose orbit is hyperbolic and stable with $\rho(D\bar P)\leq\rho_*$, and with the encoded flow as the perturbed system, one-period closeness $C_B\eta$ uniform in the launch phase by (i).  The barred constants $\bar K_1$ and $\bar v$ below are the analogues of $K_1$ and $v$ for the averaged field $\bar f_{\Th}$; by part (i) and Lemma~\ref{lem:delta-eps} they are finite and $O(\sqrt{\varepsilon}+\eta)$-close to the originals, while $\bar\rho=(1+\rho_*)/2$ keeps its value.  Four replacements are needed; no other step of that proof uses the autonomy of the perturbed field. First, for each phase $s$ the argument of Lemma~\ref{lem:return} applied to the pair $(\bar f_{\Th},\PhiTh^{(s)})$ yields a return map $\PTh^{(s)}$ with $\norm{\PTh^{(s)}-\bar P}_{C^0(\Sigma_{r_2})}\leq \bar K_1C_B\eta$ uniformly in $s$; the $k$-th crossing of an encoded trajectory is $z_{k+1}=\PTh^{(s_k)}(z_k)$ with $s_k$ the accumulated crossing time modulo $\That$. Second, in the proof of Lemma~\ref{lem:track} the recursion becomes
\[
  \tilde d_{k+1}
  \;\leq\;
  \norm{\bar P(\bar P^k\bar y)-\bar P(z_k)}_{*}
  +\norm{(\bar P-\PTh^{(s_k)})(z_k)}_{*}
  \;\leq\;
  \bar\rho\,\tilde d_k + \bar c_2\bar K_1C_B\eta,
\]
which uses only the contraction of $\bar P$ in its adapted norm, with equivalence constant $\bar c_2$, and the uniform $C^0$ bound; the confinement of the crossings $z_k$ in $\Sigma_{r_3}$ follows from $\tilde d_k\leq K_5'\eta$ by the same induction. Third, the initial crossing data is supplied by Lemma~\ref{lem:hitting} for the pair, whose winding estimate for the encoded flow holds with the time-uniform mismatch $\delta$ of Assumption~\ref{asm:learned}, the slope bound being pointwise in time.  Fourth, in Lemma~\ref{lem:clocks} each increment obeys $\abs{\tau^{(s_k)}_{\Th}(z_k)-\tau_{\bar f_{\Th}}(\bar P^k\bar y)} \leq(2/\bar v)C_B\eta+\Lip(\tau_{\bar f_{\Th}})\tilde d_k$, and the sums remain linear in $n$.  All constants of this step depend only on $(f,\ghat,\delta_0,B,\rho_*,\rhoT(f))$ for $\varepsilon+\eta\leq\varepsilon_6$, since the data of $\bar f_{\Th}$ (base point, speed, tube, Lipschitz constants) are $O(\sqrt{\varepsilon}+\eta)$-perturbations of those of $f$: the speed through the orbit mismatch $\delta+\eta$ of Lemma~\ref{lem:delta-eps}, the base point through persistence at rate $O(\varepsilon+\eta)$, and the tube and Lipschitz constants only through the uniform bound $B$.  Only their smallness enters the constants.  The assembly of Theorem~\ref{thm:floquet}(iii) then gives
\[
  \norm{\PhiTh(t;x_0)-\Phi_{\bar f_{\Th}}(t;x_0)}
  \;\leq\;
  C\eta\Bigl(1+\frac{t}{\That}\Bigr),
\]
and with Step 1 and the triangle inequality this proves \eqref{eq:nearaut-main}.  The attraction statement does not follow from the tracking bound, whose linear growth in $t$ would spoil uniformity; it follows from the recursion displayed above.  Iterating $\tilde d_{k+1}\leq\bar\rho\,\tilde d_k+\bar c_2\bar K_1C_B\eta$ gives $\tilde d_k\leq\bar\rho^{\,k}\tilde d_0+\bar c_2\bar K_1C_B\eta/(1-\bar\rho)$, uniformly in $k$: this is the cap recorded above.  Between consecutive crossings the trajectory is flowed for a time at most $\That+1$ by a field of spatial Lipschitz constant at most $B$, so a single-period Gr\"onwall factor $e^{B(\That+1)}$, together with the norm equivalences, converts the crossing bound into $\dist\bigl(\PhiTh^{(0)}(t;x_0),\bar\gamma_\Th\bigr)\leq C\bigl(\bar\rho^{\,t/(\That+1)}\dist(x_0,\bar\gamma_\Th)+\eta\bigr)$ for all $t\geq0$: the trajectory enters and remains in a $C\eta$-neighborhood of $\bar\gamma_\Th$.  The distance conclusion for the encoded flow follows by adding $\dist_H(\bar\gamma_\Th,\Gamma)$; the tracking bound and the triangle inequality enter only in the trajectory estimate \eqref{eq:nearaut-main}, which is where the linear term originates.
\end{proof}

\subsection{\texorpdfstring{Proofs of Proposition~\ref{prop:obstruction-local} and Theorem~\ref{thm:orbital}}{Proofs of the obstruction and the orbital guarantee}}\label{app:pf-orbital}

This part proves the local obstruction and the orbital guarantee.  Both proofs stand outside the convention of Appendix~\ref{app:pf-lemmas}: they use the standing hypotheses of Section~\ref{ssec:orbital}, namely Assumption~\ref{asm:true}, exact $\That$-periodicity of $\fTh$ in $t$, and $\norm{\fTh}_{C^2(\R^d\times\R)}\leq B$, together with the constant $\varepsilon$ of Assumption~\ref{asm:learned} where stated, and their constants depend only on $(d,\That,B,\rho_*)$.

\begin{proof}[Proof of Proposition~\ref{prop:obstruction-local}]
The proof has four steps.  The tangent to $\Gamma$ is an exact eigenvector of the target's stroboscopic map, with eigenvalue one.  An interpolation inequality upgrades the $C^0$ closeness of Assumption~\ref{asm:learned} to $C^1$ closeness, at the cost of a square root.  The tangent is therefore an approximate eigenvector of $M$ with approximate eigenvalue one, and Hadamard's inequality converts that into a lower bound on the spectral radius.

\emph{Step 1: the unit eigenvalue of the target.}  Since $\Gamma$ has period exactly $\That$, $\Phi_f(\That;\ghat(s))=\ghat(s)$ for every $s$.  Differentiating in $s$ gives $DS_f(\ghat(s))\,\dot\ghat(s)=\dot\ghat(s)$, where $S_f:=\Phi_f(\That;\cdot)$.  Fix $x\in\Gamma$ with $\norm{x^\sharp-x}\leq\varrho$ and let $\tau$ be the unit tangent to $\Gamma$ at $x$, which is well defined because $\norm{f}\geq v_{\min}>0$ on $\Gamma$.  Then $DS_f(x)\tau=\tau$.

\emph{Step 2: from $C^0$ to $C^1$ closeness.}  Put $g:=\Sth-S_f$ on $N_{2r_1}$.  Assumption~\ref{asm:learned} at $t=\That$ gives $\norm{g}_{C^0(N_{2r_1})}\leq\varepsilon$.  Both maps are $C^2$: for $\Sth$ this uses only the global bound $\norm{\fTh}_{C^2(\R^d\times\R)}\leq B$ supplied by the cutoff extension of Assumption~\ref{asm:learned}, with no confinement and no condition on $\varepsilon$; for $S_f$, the confinement of the true flow, built from $f$ alone in the proof of Lemma~\ref{lem:confine} and likewise free of any condition on $\varepsilon$, keeps every $f$-trajectory from $N_{2r_1}$ in $N_{\delta_0}$ for $t\in[0,\That]$, where $\norm{f}_{C^2(N_{\delta_0})}\leq B$; the first and second variational equations then bound $\norm{D^2S_f}$ and $\norm{D^2\Sth}$ on $N_{2r_1}$ by a constant depending only on $(d,\That,B)$.  Write $\Lambda_2\geq1$ for the resulting bound on $\norm{D^2g}$ there.  For unit vectors $u,w$ set $\varphi(h):=\langle w,g(x+hu)\rangle$, defined for $\abs{h}\leq r_1$ since $x\in\Gamma$.  Taylor expansion at $0$ in both directions and subtraction give, for $0<a\leq r_1$,
\[
  \abs{\varphi'(0)}\;\leq\;\frac{\varepsilon}{a}+\frac{a\Lambda_2}{2},
\]
and the choice $a=\sqrt{2\varepsilon/\Lambda_2}$, admissible under the stated bound on $\varepsilon$, yields $\abs{\varphi'(0)}\leq\sqrt{2\varepsilon\Lambda_2}$.  Taking the supremum over $u$ and $w$ gives $\norm{Dg(x)}\leq\sqrt{2\varepsilon\Lambda_2}$.

\emph{Step 3: the tangent is an approximate eigenvector of $M$.}  With $\Lambda_1$ a Lipschitz constant for $D\Sth$, again from the variational equations,
\begin{align*}
  \norm{(M-I)\tau}
  &\;\leq\;\norm{D\Sth(x^\sharp)-D\Sth(x)}+\norm{D\Sth(x)-DS_f(x)}+\norm{DS_f(x)\tau-\tau}\\
  &\;\leq\;\Lambda_1\varrho+\sqrt{2\varepsilon\Lambda_2}\;=:\;\Delta ,
\end{align*}
the last term vanishing by Step~1.

\emph{Step 4: Hadamard.}  Complete $\tau$ to an orthonormal basis of $\R^d$.  Hadamard's inequality bounds the determinant by the product of the column norms in that basis, so
\[
  \abs{\det(M-I)}\;\leq\;\norm{(M-I)\tau}\;\norm{M-I}^{\,d-1}\;\leq\;\Delta\,(1+\norm{M})^{d-1}.
\]
Since $\det(M-I)=\prod_i(\lambda_i-1)$ over the eigenvalues of $M$, some eigenvalue satisfies $\abs{\lambda-1}\leq[\Delta(1+\norm{M})^{d-1}]^{1/d}$, and $\rho(M)\geq\abs{\lambda}\geq1-\abs{\lambda-1}$ gives \eqref{eq:obstruction-local}.
\end{proof}

\begin{proof}[Proof of Theorem~\ref{thm:orbital}]
The proof has three steps.  Step~1 turns the spectral hypothesis \textup{(B1)} into a genuine contraction of $\Sth$ on a ball around its fixed point, in an adapted norm.  Step~2 uses the periodicity identity \eqref{eq:strobo-semigroup} to pass from the iterates of $\Sth$ to the flow at intermediate times.  Step~3 is a triangle inequality against $\Gamma$ through hypothesis \textup{(B2)}.  No section, return time or clock comparison appears, and Assumption~\ref{asm:learned} is never used.

\emph{Step 1: contraction of the stroboscopic map.}  Since $\norm{\fTh}_{C^2}\leq B$ on all of $\R^d\times\R$, the variational equation along any trajectory gives, for every $x$,
\begin{equation}\label{eq:strobo-c1}
  \norm{D\Sth(x)}\;\leq\;e^{B\That},
  \qquad
  \Lip\bigl(D\Sth\bigr)\;\leq\;\Lambda_B:=\That B\,e^{2B\That},
\end{equation}
the second by differentiating the variational equation once more and applying Gr\"onwall on $[0,\That]$; both constants depend only on $(\That,B)$.  Put $\theta:=(\rho_*+\bar\rho)/2\in(\rho_*,\bar\rho)$.  The matrix $M$ lies in the compact family $\{A:\rho(A)\leq\rho_*,\ \norm{A}\leq e^{B\That}\}$, so the adapted-norm construction in the proof of Lemma~\ref{lem:persist} applies.  That construction is pure linear algebra, and works for any compact family of square matrices of spectral radius at most $\rho_*$ and norm at most a fixed bound.  There it is applied to the $(d-1)\times(d-1)$ section matrices, here to the full $d\times d$ monodromy.  In particular it is applied for $d=2$, where the one-dimensional shortcut of that proof does not.  It supplies a norm $\norm{\cdot}_\diamond$ with $\norm{M}_\diamond\leq\theta$ and $\norm{\cdot}\leq\norm{\cdot}_\diamond\leq\cd\norm{\cdot}$, the lower constant being $1$ because the sum $\sum_i\theta^{-i}\norm{M^ix}$ starts at $i=0$.  We write $\cd$ rather than $c_2$, which is bound elsewhere in this appendix to the section norm; $\cd$ is uniform over the family and depends only on $(d,\That,B,\rho_*)$.  For a matrix $A$ this gives $\norm{A}_\diamond\leq\cd\norm{A}$ in the induced norms.  Set
\[
  R\;:=\;\frac{\bar\rho-\theta}{\cd\,\Lambda_B},
  \qquad
  B_\diamond:=\{y:\norm{y-x_\Th^\dagger}_\diamond\leq R\} .
\]
For $y\in B_\diamond$ we have $\norm{y-x_\Th^\dagger}\leq\norm{y-x_\Th^\dagger}_\diamond\leq R$, so by \eqref{eq:strobo-c1} and the choice of $R$,
\[
  \norm{D\Sth(y)}_\diamond
  \;\leq\;\norm{M}_\diamond+\cd\norm{D\Sth(y)-M}
  \;\leq\;\theta+\cd\Lambda_BR\;\leq\;\bar\rho .
\]
The set $B_\diamond$ is convex and $\Sth x_\Th^\dagger=x_\Th^\dagger$, so the mean value inequality along the segment from $x_\Th^\dagger$ to $y$, which lies in $B_\diamond$, gives $\norm{\Sth y-x_\Th^\dagger}_\diamond\leq\bar\rho\norm{y-x_\Th^\dagger}_\diamond$.  Hence $B_\diamond$ is invariant under $\Sth$ and, by induction,
\begin{equation}\label{eq:strobo-geom}
  \norm{\Sth^{\,n}y-x_\Th^\dagger}_\diamond\;\leq\;\bar\rho^{\,n}\norm{y-x_\Th^\dagger}_\diamond,
  \qquad y\in B_\diamond,\ n\in\mathbb{N} .
\end{equation}
Define $r_0:=R/\cd$.  If $\norm{x_0-x_\Th^\dagger}\leq r_0$ then $\norm{x_0-x_\Th^\dagger}_\diamond\leq\cd r_0=R$, so $x_0\in B_\diamond$; this is the direction in which the equivalence is used, and it costs $\cd$.

\emph{Step 2: from the iterates to the flow.}  Fix $t\geq0$ and write $t=n\That+s$ with $n=\lfloor t/\That\rfloor$ and $s\in[0,\That)$.  By \eqref{eq:strobo-semigroup}, $\PhiTh(t;x_0)=\PhiTh(s;\Sth^{\,n}x_0)$ and $\PhiTh(s;x_\Th^\dagger)\in\gamma_\Th^\dagger$.  Both points are flowed for the same time $s\leq\That$ by the same field, whose spatial Lipschitz constant is at most $B$, so Gr\"onwall gives
\[
  \norm{\PhiTh(t;x_0)-\PhiTh(s;x_\Th^\dagger)}
  \;\leq\;e^{B\That}\norm{\Sth^{\,n}x_0-x_\Th^\dagger}
  \;\leq\;e^{B\That}\,\bar\rho^{\,n}\,\norm{x_0-x_\Th^\dagger}_\diamond
  \;\leq\;\cd\,e^{B\That}\,\bar\rho^{\,n}\,\norm{x_0-x_\Th^\dagger},
\]
using \eqref{eq:strobo-geom} and then the two equivalences in the directions in which they are needed.

\emph{Step 3: comparison with $\Gamma$.}  By \textup{(B2)} every point of $\gamma_\Th^\dagger$ lies within $\varepsilon_\Gamma$ of $\Gamma$, in particular $\dist(\PhiTh(s;x_\Th^\dagger),\Gamma)\leq\varepsilon_\Gamma$.  The triangle inequality for the distance to the compact set $\Gamma$ then gives
\[
  \dist\bigl(\PhiTh(t;x_0),\Gamma\bigr)
  \;\leq\;\norm{\PhiTh(t;x_0)-\PhiTh(s;x_\Th^\dagger)}+\dist\bigl(\PhiTh(s;x_\Th^\dagger),\Gamma\bigr) ,
\]
which with Step~2 is \eqref{eq:orbital-bound} for $C:=\cd e^{B\That}\geq1$.  All constants depend only on $(d,\That,B,\rho_*)$.
\end{proof}

{\section*{Acknowledgments}
The authors sincerely thank Enrique Zuazua for suggesting this line of research and for his detailed feedback on earlier versions of the manuscript, which shaped the revision. Ziqian Li is supported by the Alexander von Humboldt Professorship program.}

\bibliographystyle{siam}
\bibliography{refs}

\end{document}